\documentclass[11pt,oneside ]{book}
\usepackage[utf8]{inputenc}
\usepackage{setspace}
\newcommand{\bigspace}{\mbox{\hspace{10mm}}}

\usepackage{times}
\usepackage[dvipsnames,usenames]{color}
\usepackage[colorlinks=true,urlcolor=Blue,citecolor=Green,linkcolor=BrickRed]{hyperref}
\usepackage[usenames,dvipsnames]{xcolor}
\usepackage{amsthm,amsmath,amssymb, mathabx}
\usepackage{fullpage}
\usepackage[ruled,noend,linesnumbered]{algorithm2e}
\usepackage{url}
\usepackage[capitalise]{cleveref}
\usepackage[noadjust]{cite}
\usepackage{mathtools}
\usepackage{graphicx}
\usepackage{graphicx}
\usepackage{amsfonts}
\usepackage{authblk}
\newcounter{mycounter}

\usepackage{hyperref}
\usepackage{xcolor}
\usepackage[all]{hypcap}
\usepackage{sidecap}
\usepackage{caption}
\usepackage{layout}
\usepackage{verbatim}
\usepackage{color}
\usepackage{latexsym}
\usepackage{algorithmic}
\usepackage{caption}
\usepackage{subcaption}
\usepackage[english]{babel}
\usepackage{tikz}
\usetikzlibrary{shapes,arrows}
\usepackage{makecell}

\newcommand{\ranges}{\mathrm{ranges}}

\newcommand{\coralg}{\impalg}
\newcommand{\size}{\mathrm{size}}
\newcommand{\stralg}{\textsc{Streaming}_\textsc{Coreset}}
\newcommand{\stream}{stream}

\newcommand{\ta}{\xi}
\newcommand{\impalg}{\textsc{Coreset}}
\newcommand{\cost}{\mathrm{cost}}
\newcommand{\dist}{\mathrm{dist}}
\newcommand{\br}[1]{\left\{#1\right\}}
\newcommand{\REAL}{\ensuremath{\mathbb{R}}}
        \newcommand{\red}{\textsc{Sensitivity}}
        \newcommand{\redd}{$\textsc{WSensitivity}$_{\ell_{\inf}}\textsc{-coreset}}

        \newcommand{\linfcore}{\textsc{$\ell_\infty$-Coreset}}
        \newcommand{\projcore}{\textsc{$\ell_\infty$-Proj-Clustering-Coreset}}
        \newcommand{\linfrec}{\textsc{Recursion}}

\newcommand{\argmin}{\operatornamewithlimits{arg\, min}}
\newcommand{\argmax}{\operatornamewithlimits{arg\, max}}

\newcommand{\R }{\mathbb{R}}

\newcommand{\eps}{\varepsilon}
\newcommand{\loss}{\mathrm{loss}}

\newcommand{\CC}{Y}
\renewcommand{\c}{y}

\newcommand{\range}{\mathrm{range}}
\newcommand{\pr}{\mathrm{Pr}}

\newcommand{\coralgfinal}{\textsc{$k$-GMM-Coreset}}
\newcommand{\ff}{f}
\renewcommand{\paragraph}[1]{\medskip\noindent\textbf{{#1}}}
\renewcommand{\epsilon}{\varepsilon}
\renewcommand{\baselinestretch}{1.67}

\newtheorem{theorem}{Theorem}[chapter]
\newtheorem{corollary}[theorem]{Corollary}
\newtheorem{lemma}[theorem]{Lemma}

\newtheorem{definition}[theorem]{Definition}

\newtheorem{observation}[theorem]{Observation}

\theoremstyle{plain}
\providecommand{\norm}[1]{\left\lVert#1\right\rVert}

\newcommand{\subcore}{\textsc{Subspace-Coreset}}

\begin{document}

\title{\Huge{Coresets for Gaussian Mixture Models\\ of Any Shape}
                \huge
             \\[10mm] Zahi Kfir
             \\[25mm] \Large THESIS SUBMITTED IN PARTIAL FULFILLMENT OF THE
             \\       REQUIREMENTS FOR THE MASTER'S DEGREE
             \\[15mm] University of Haifa
             \\       Faculty of Social Sciences
             \\       Department of Computer Sciences
             \\[10mm] November, 2017
}
\author{}
\date{}
\maketitle{}

\pagestyle{plain}
\pagenumbering{Roman}

\begin{center}
\Huge
Coresets for Big Data Learning\\ of Gaussian Mixture Models of Any Shape
\huge
\\[5mm] By: Zahi Kfir
\\[3mm] Supervised By: Dr. Dan Feldman
\Large
\\ [10mm]THESIS SUBMITTED IN PARTIAL FULFILLMENT OF THE
\\ REQUIREMENTS FOR THE MASTER'S DEGREE
\\ [5mm]University of Haifa
\\ [1mm]Faculty of Social Sciences
\\ [1mm]Department of Computer Sciences
\\ [3mm]November, 2017
\\ [5mm] Approved by:
$\underline{\bigspace\bigspace\bigspace\bigspace\bigspace\bigspace\bigspace\bigspace}$
   \bigspace    Date:$\underline{\bigspace\bigspace}$
\\ (supervisor)\bigspace
\\ [3mm]Approved by:
$\underline{\bigspace\bigspace\bigspace\bigspace\bigspace\bigspace\bigspace\bigspace}$
   \bigspace    Date:$\underline{\bigspace\bigspace}$
\\ (Chairman of  M.Sc Committee) \bigspace

\end{center}

\chapter*{}
\begin{center}
\LARGE{Acknowledgment}
\end{center}
I would like to thank Dr. Dan Feldman, for introducing me to the world of scientific research.
Dan's door is always open, and he is always willing to assist with any kind of problem. His help was invaluable.

I thank Mr. Elad Tolochinsky for his contributions to the thinking process, his friendship, and for the advices over the past years.


Last but not least, I would like to express my profound gratitude to my parents for all their support and encouragement,
and to my wife Yonit, who has been there for me throughout this process.
I could not have done this without you.

\tableofcontents

\chapter*{}
\addcontentsline{toc}{chapter}{Abstract}
\begin{center}
\huge{Coresets for Big Data Learning of\\ Gaussian Mixture Models of Any Shape}
\\ [3mm] Zahi Kfir
\\ [3mm] \huge{Abstract}
\end{center}
    An $\eps$-coreset for a given set $D$ of $n$ points, is usually a small weighted set,
    such that querying the coreset \emph{provably} yields a $(1+\varepsilon)$-factor approximation to the original (full) dataset,
    for a given family of queries. Using existing techniques, coresets can be maintained for streaming, dynamic (insertion/deletions),
    and distributed data in parallel, e.g. on a network, GPU or cloud.

    We suggest the first coresets that approximate the negative log-likelihood for $k$-Gaussians Mixture Models (GMM) of arbitrary shapes (ratio between eigenvalues of their covariance matrices). For example, for any input set $D$ whose coordinates are integers in $[-n^{100},n^{100}]$ and any fixed $k,d\geq 1$, the coreset size is $(\log n)^{O(1)}/\eps^2$,
    and can be computed in time near-linear in $n$, with high probability. The optimal GMM may then be approximated quickly by learning the small coreset.

    Previous results [NIPS'11, JMLR'18] 
    suggested such small coresets for the case of semi-speherical unit Gaussians,
    i.e., where their corresponding eigenvalues are constants between $\frac{1}{2\pi}$ to $2\pi$.

    Our main technique is a reduction between coresets for $k$-GMMs and projective clustering problems. We implemented our algorithms, and provide open code, and experimental results. Since our coresets are generic, with no special dependency on GMMs,
    we hope that they will be useful for many other functions.




\listoffigures
\addcontentsline{toc}{chapter}{List of Figures}

\chapter{Introduction}\label{sec:Introduction}
\thispagestyle{empty}
\pagenumbering{arabic}
    In this section we give background for the problems and techniques that are relevant to the rest of the thesis.
    Section ~\ref{sec_Related_Work} then gives related work followed by our contribution in Section~\ref{sec_Our Contribution}.

    \section{Background}\label{sec_Background}
    The theoretical analyses in this thesis focus on Gaussian Mixture Models as explained below. We then introduce coresets, and the motivation for learning and querying very large and dynamic distributed databases.
        However, we expect that our main algorithm,
        results and techniques would be relevant and inspired many other problems in machine learning and neural networks.

        \paragraph{Gaussian mixture models. }
            A \emph{$k$-Gaussian mixture model} ($k$-GMM for short) in $\REAL^d$ is an ordered set
            \[
                \theta=\big((\omega_i,\Sigma_i,\mu_i)\big)_{i=1}^k=\big((\omega_1,\mu_1,\Sigma_1),\cdots,(\omega_k,\mu_k,\Sigma_k)\big)
             \]
             of $k$ tuples, where $\mu_i\in\REAL^d$, $\Sigma\in\REAL^{d\times d}$ is a positive definite matrix,
             and $\omega=(\omega_1,\cdots,\omega_k)\in[0,1]^k$ is a \emph{distribution vector}, i.e.,
             whose sum is $\norm{\omega}_1=\sum_{j=1}^k\omega_j=1$.

            We consider the $GMM$ fitting problem of computing a Mixture of $k$-Gaussians Models ($k$-GMM for short)
            that maximizes the likelihood of generating a given set $P$ of $n$ points in the $d$-dimensional Euclidean space.
            That is, to minimize the negative log-likelihood
            \[
                L(P,\theta)=-\sum_{p\in P}\ln\sum_{i=1}^k \omega_i \cdot \frac{\exp(-\frac{1}{2}(p-\mu_i)^T\Sigma_i^{-1}(p-\mu_i))}{\sqrt{\det(2\pi \Sigma_i)}}.
            \]
            As is common in computational geometry, we assume worst case input. That is, unlike in PAC-learning and other many machine learning communities, we do not assume e.g. that the input points were sampled i.i.d. from Gaussian or from any other specific distribution. See Section~\ref{subsec Likelihood of Gaussians} for details.

            More generally, we would like to quickly answer database queries where  (not necessarily optimal) $k$-GMMs are given and we wish to know how good each one of them fits the database records,
            in time that is sub-linear in the number of records,
            or solve variant versions of the optimization problem that minimizes $L(P,\theta)+f(\theta,k)$
            for a function $f$ that depends only on $k$ and the $k$-GMM $\theta$, e.g. its sparsity,
            the number of $k$-GMMs or a regularization term~\cite{scholkopf2002learning}.

            In Section~\ref{subsec Likelihood of Gaussians} we give  formal definitions of $k$-GMMs and their fitting cost.

        \paragraph{Turning VLDB to VSDB using coresets. }
            A possible approach for solving the $k$-GMM fitting problem and its variants above, maybe also for big data,
            is to develop new algorithms from scratch. Instead, we suggest \emph{coresets} for this and related problems.
            In this paper, a coreset for a given finite input set $P$ of points is a (possibly weighted) subset $C\subseteq D$ such that,
            for any given $k$ Gaussian Mixture Model $\theta$, the negative log-likelihood $(P\mid \theta)$ is
            \[
            L(P,\theta):=\sum_{p\in P}L(\br{p}, \theta)
             \]
            that $\theta$ generated the input set is provably approximately the same as the negative log-likelihood $L(\theta\mid C)$
            that $\theta$ generated $C$. More precisely, the approximation is up to a given multiplicative factor of $1\pm\eps$.
            The goal is to have a small coreset, i.e.,
            a good trade off between $\eps$ and the size of the coreset $C$
            that is a small database (SMDB) representation of the original (possibly very large) database.
            Note that coreset is problem dependent, and its definition also changes from paper to paper.
            For example, the coresets in this paper are always subsets of the input set $P$ (and not arbitrary points in $\REAL^d$), and they are either un-weighted or positively weighted.

        \paragraph{Why coresets? }
            The first coresets, two decades ago,
            suggested first efficient near-linear time algorithms to optimization problems in Computational
            Geometry~\cite{agarwal2004approximating,agarwal2005geometric,agarwal2001maintaining},
            and then to more general problems in theoretical computer
            science~\cite{feldman2006coresets,assadi2017randomized,feldman2006coresets}.
            However, over the recent decade, coresets suggested significant breakthroughs in many other fields, such as machine learning,
            computer vision, and cryptography~\cite{akavia2017secure,cryptoeprint},
            as well as real-world applications by main players in the
            industry~\cite{clarkson2010coresets,boutsidis2012rich}.

            The natural motivation for having a coreset is simply to run an existing optimization algorithm on the coreset.
            If the coreset is small and its construction time is fast, then the overall running time may be smaller by order of magnitudes.
            For example, while the $k$-means clustering problem for $n$ points in $\REAL^d$ is NP-hard when either $d$ or $k$ are not constants (part of the input),
            a coreset for this problem of size that depends polynomially on $k$ and $1/\eps$,
            and independent of the input cardinality $n$ or dimension $d$
            can be computed in $O(ndk)$ time~\cite{feldman2013turning}.
            The running time is then reduced from $n^{dk}$ to $O(ndk)+2^{O(dk)}$ by running naive exhaustive search algorithms on the coreset.

            However, this is not what is done in practice and also this paper.
            Instead, popular off the shelf heuristic is applied on the coreset to avoid terms such as $2^{O(dk)}$ above.
            In this paper, For learning GMMs, the EM-algorithm is a common candidate for such a heuristic as explained e.g. in~\cite{feldman2011scalable}.
            While the global optimal guarantee is no longer preserved, the coreset property still holds:
            any solution obtained by the heuristics on the original data would be approximated by the coreset.
            In fact, usually running a heuristic on the coreset yields \emph{better} results than running it on the original data; see e.g.~\cite{feldman2015idiary}. Intuitively, the coreset smooth the solution space and removes noise that causes the heuristic to get trapped in local minima.
            Even if the heuristic is already fast,
            we may run it thousands of times on the coreset instead of a single run on the the original data, or use more iterations/seeds etc. in the same running time to improve the state of the art.

        \paragraph{Handling Streaming Distributed Dynamic Data. }
            Even if we already have an efficient and good solution to our problem,
            one of the main advantage of coresets is that an inefficient (say, $n^{10}$ time) off-line non-parallel coreset construction can be maintained for "Big Data":
            a possibly infinite stream of points that may be distributed on a cloud or networks of hundreds of machines using small memory,
            communication (``embarrassingly in parallel"~\cite{regin2013embarrassingly}) and update time per point.
            This holds if the coresets are mergable in the sense that if $C_1$ is a coreset for $D_1$, and $C_2$ is a coreset for $D_2$,
            then $C_1\cup C_2$ is a coreset for $D_1\cup D_2$.
            See survey and details of this well known technique
            in~\cite{har2004coresets,indyk2014composable,bentley1980decomposable}.
            This property allows us to compute coresets independently over time and different machines for only small subsets of the input (say, $O(\log n)$ points),
            and then merge and re-reduce them; see descriptions in Algorithm~\ref{algsup} and figures~\ref{fig:streaming},~\ref{fig:streaming_9}.

            The off-line optimization algorithm can then be applied on the maintained coreset (from scratch) every now and then when needed.
            This simple but generic and provable reduction can be applied for any mergeable coreset and is explained in details in many papers;
            see~\cite{har2004coresets,bentley1980decomposable}. However, the results are usually for a specific problem and not formalized. E.g. issues of handling ``coreset for coreset" via weighted input are usually ignored. We give a generic framework with provable bound regarding time, space and probability of success in Section~\ref{sec Coresets for Streaming}.  Similarly, such coresets support streaming and distribution data simultaneously~\cite{feldman2015more} delete an input point and update the model,
            usually in near-logarithmic time per point, as explained e.g. in~\cite{AMRSLS12,FGSSS13,feldman2010coresets}.

        \paragraph{Constrained and sparse optimization. }
            A coreset for a family of models is significantly different than sparse optimal solution to the problem such as e.g. the output of Frank-Wolfe algorithm~\cite{swoboda2018map,bauer2016understanding}.
            In particular, unlike mergable coresets, it is not clear how to maintain sparse solution when a new point is inserted to the input set,
            or for streaming/distributed data in general.
            Moreover, since a coreset approximates \emph{every} model in a given family of models,
            it can be used to compute not only the optimal model in the family of solutions, but also optimal under any given constraints that depend on on the model. For example, a Gaussian whose covariance matrix is sprase, or has few non-zero eigenvalues.

            A coreset for a family of models also approximates, by its definition, additional regularization terms that depend only on the model.
            While a  different optimization algorithm should be applied on the coreset,
            if the coreset is small -- then this algorithm may be relatively inefficient,
            but still efficient when applied on the small coreset.
            Alternatively, heuristics such as the EM-algorithm~\cite{dempster1977maximum} may be used to handle such constraints on the coreset.

            An important property of all the suggested coresets in this paper is that they are (weighted) subsets of the input set.
            In particular, sparse input points imply sparse points in the coreset.
            The fact that the coreset is a subset of the input, and not, say, linear combinations of points,
            as in~\cite{dasgupta2003elementary}, PCA~\cite{munteanu2014smallest} or random  projections~\cite{kerber2014approximation}
            is also useful in practice for many other applications that need to interpret the coreset as a set of representatives,
            or apply the same coreset for other algorithms that expects data in a specific format.
            It also reduces numerical issues that arise when the coreset consists of linear combinations or projections of the input points.

        \paragraph{Projective Clustering Problem.} In this paper we forge a link between the family of $k$-GMM problems and the family of projective clustering problems. The input for a  projective clustering problem is a set $P$ of $n$ points in $\REAL^d$ and an integer $k\geq 1$.
            The \emph{fitting cost} $\dist(P,S)$ of a given set $S$
            of $k$ hyperplanes ($d-1$ dimensional affine subspaces)
            to $P$ is the maximum over the $n$ distances between every point to its closest hyperplane.
            The optimal projective clustering $S^*$ is the set of $k$ hyperplane that minimize this fitting cost, i.e., smallest width set that covers all the input points.
            Formally, by letting $H_{d,k}$ denote the union over every set of $k$-hyperplanes in $\REAL^d$,
            \[
            \dist(P,S^*) = \min_{S\in H_{d,k}}\dist(P,S) =
            \min_{S\in H_{d,k}}\max_{p\in P}\dist(p,S) =
            \min_{S\in H_{d,k}}\max_{p\in P}\min_{s\in S}\norm{p-s}_2.
            \]

            For a given $\eps>0$, usually in $(0,1)$, an $\eps$-coreset for $(P,k,\dist)$ approximates $\dist(P,S)$
            for every set of $k$ hyperplanes $S\in H_{d,k}$
            up to $(1+\eps)$ multiplicative error.
            More generally, the input may also include a variable
            $j\in\br{0,..,d-1}$ that restricts the dimension of each subspace to be $j$.
            In particular, for $j=0$ the problem is known as $k$-center
            where we wish to cover the input points by the smallest $k$ balls of the same radius.

            The projective clustering problem may also be defined for say,
            sum or sum of squared distances instead of maximum distance between each point to its closest subspace.
            In this case $j=0$ yields the classic $k$-means problem, and $k=1$ is
            related to the PCA problem or low-rank approximation
            (if the subspace should passes through the origin).
            By combining techniques from~\cite{VX12-soda,feldman2013turning} and~\cite{FL11,braverman2016new} coreset for projective clustering for the maximum distance can be used to compute a coreset for sum or sum of squared distances; see also Chapter~\ref{sec From Sensitivity to ell infty}.

    \section{Related Work}\label{sec_Related_Work}
            The GMM fitting problem is one of the fundamental problems in machine learning which generalizes the notion of $k$-means clustering,
            where the covariance matrix that corresponds to each Gaussian is simply the identity matrix,
            i.e., its eigenvalues are all $1$ and the Gaussian has the shape of a ball
            around some point $\mu\in\REAL^d$.
            It is also strongly related to Radial Basis Networks~\cite{aljarah2018training,alexandridis2017fast} and Radial Basis function~\cite{rbf1,rbf2,rbf3,rbf5}.
            The problem is NP-hard when $k$ is part of the input~\cite{raghunathan2017learning}
            and many heuristics and approximation algorithms under different assumptions were suggested over the years;
            e.g.~\cite{vlassis2002greedy,zhang2003algorithms,gauvain1994maximum}.
            The EM-algorithm (Expected Maximization) is one of the popular in practice and used in common software
            libraries~\cite{townsend2016pymanopt,kapoor2015mpi,hicks2017pydigree,gopi2014digital}.
            However, there are very little results that provably handle scalable (big) data,
            or that handle constraints such as sparse covariance matrices that represent the Gaussians,
            or are able to compute the fitting of a given GMM to the data in sub-linear time.

\paragraph{Projective Clustering. }            It was proved in~\cite{EV05} that a coreset of size sub-linear in $n$ for approximating $k$ hyperplanes as defined in the previous section,
            does not exists for some example input sets.
            However, a coreset of size $(\log M)^{g(d,k)}/\eps^d$ was suggested in~\cite{EV05}
            for the case that the input is contained in a polynomial grid, i.e., $P\subset \br{-M,\cdots,M}^d$ and $g(d,k)$ is a function that depends only on $d$ and $k$; see Theorem~\ref{evmaintheorem}. The exponential dependency on $d$ is unavoidable even for the case of $k$-center ($j=0$); see~\cite{agarwal2005geometric} and more references therein.
            For $j\geq 1$ the coreset depends exponentially also in $k$,
            and logarithmic in $n$, which are both unavoidable due to the lower bounds in~\cite{agarwal2005geometric,EV05}. These claims hold for both maximum, sum or sum of squared distances from the points to the subspaces.

        \paragraph{Theoretical results on mixtures of Gaussians as summarized in~\cite{feldman2011scalable}.}
            There has been a significant amount of work on learning and applying GMMs (and more general distributions).
            Perhaps the most commonly used technique in practice is the EM algorithm~\cite{dempster1977maximum},
            which is however only guaranteed to converge to a local optimum of the likelihood.
            Dasgupta~\cite{dasgupta1999learning} is the first to show that parameters of an unknown GMM $P$
            can be estimated in polynomial time, with arbitrary accuracy $\varepsilon$,given i.i.d. samples from $P$.
            However, his algorithm assumes a common covariance, bounded excentricity,
            a (known) bound on the smallest component weight, as well as a separation (distance of the means),
            that scales as $\Omega(\sqrt{d})$.
            Subsequent works relax the assumption on separation to $d^{\frac{1}{4}}$~\cite{dasgupta2013two}
            and $k^{\frac{1}{4}}$~\cite{vempala2004spectral}.
            \cite{arora2005learning} is the first to learn general GMMs, with separation $d^{\frac{1}{4}}$.
            \cite{feldman2006pac} provides the first result that does not require any separation, but assumes that the Gaussians are axis-aligned.
            Recently, \cite{moitra2010settling} and~\cite{belkin2010polynomial}
            provide algorithms with polynomial running time (except exponential dependence on $k$) and sample
            complexity for arbitrary GMMs.
            However, in contrast to our results, all the results described above crucially
            rely on the fact that the data set $D$ is actually generated by a mixture of Gaussians.
            The problem of fitting a mixture model with near-optimal log-likelihood for arbitrary data
            is studied by~\cite{arora2005learning}, who provides a PTAS for this problem.
            However, their result requires that the Gaussians are identical spheres,
            in which case the maximum likelihood problem is identical to the $k$-means problem.
            \cite{feldman2011scalable} make only mild assumptions about the Gaussian components.
            \cite{lucic2017training} extended Feldman et al. \cite{feldman2011scalable},
            by suggesting a more practical algorithm with linear running time in $n$.

        \paragraph{Coreset as summarized in \cite{feldman2011scalable}.}
            Approximation algorithms in computational geometry often make use of random sampling, feature extraction,
            and $\epsilon$-samples \cite{haussler2018decision}.
            Coresets can be viewed as a general concept that includes all of the above, and more.
            See a comprehensive survey on this topic in \cite{feldman2011unified}. 
            It is not clear that there is any commonly agreed-upon definition of a coreset,
            despite several inconsistent attempts to do so \cite{har2004coresets,feldman2007ptas}.
            Coresets have been the subject of many recent papers and several surveys
            \cite{agarwal2005geometric,czumaj2007sublinear}.
            They have been used to great effect for a host of geometric and graph problems,
            including $k$-median \cite{har2004coresets}, $k$-mean \cite{feldman2007ptas,barger2016k},
            $k$-center \cite{har2004high}, $k$-line median \cite{feldman2006coresets},
            pose-estimation \cite{nasser2015coresets,nasser2015low}, etc.
            Coresets also imply streaming algorithms for many of these problems
            \cite{har2004coresets, agarwal2005geometric, frahling2005coresets, feldman2007ptas}.
            Framework that generalizes and improves several of these results has recently appeared in \cite{dasgupta1999learning}.
            \cite{lucic2017training} proved that one can use \emph{any} bicriteria approximation
            for the $k$-means clustering problem as a basis for the importance sampling scheme,
            thus, construct coresets in less time.

        \paragraph{Scaling issues.}
            In~\cite{feldman2011scalable,lucic2017training}
            it was proved that to obtain a coreset for approximating the negative log-likelihood of $k$-GMMs
            we must assume some lower bound on all the eigenvalues of each of the $k$-GMMs
            in the family of approximated $k$-GMMs.
            Otherwise, achieving an $\eps$-coreset is as hard as achieving $0$-coreset with no error at all,
            which is clearly impossible in general unless that coreset has all the input points.
            This problem is due to scaling issue that do not appear in problem such as projective clustering,
            where scaling the input (multiplying each coordinate by a constant)
            would not make the problem easier or harder with respect to $(1+\eps)$ multiplicative factor approximation.
            This is why a lower bound of $1/(2\pi)$ is assumed for each eigenvalue, when we wish to approximate the negative log-likelihood, as explained in~\cite{feldman2011scalable,lucic2017training}.  In this paper we use the same lower bound for these eigenvalues. It is an open problem whether we can obtain a smaller bound. However, for the following $\phi$-function, unlike these previous results we do not assume any (upper or lower) bounds on these eigenvalues.

        \paragraph{$\phi$ approximation. }            Cost functions and approximation algorithms in general are used to approximate sum of non-negative loss functions or fitting errors.
            It is thus more natural in this paper, as many others, e.g.~\cite{tolochinsky2018coresets,feldman2011scalable,kumar2009discriminative,kannan1994maximum},
            to approximate the negative log-likelihood $L(\theta\mid D)$ of a given $k$-GMM $\theta$, which is a sum over non-negative numbers,
            than the likelihood itself, which is a multiplication of $n=|D|$ numbers between $0$ to $1$.
            This term $L(\theta\mid D)$ was decomposed into a sum of two expressions
            in~\cite{feldman2011scalable,lucic2017training}:
            one that is independent of $D$,
            and thus can be computed \emph{exactly} from the given $\theta$, and one that is denoted by $\phi(\theta\mid D)$ and can be approximated by the coreset.
            Moreover, the value $\phi(\theta\mid D)$ captures all dependencies of $L(\theta\mid D)$ on $D$,
            and via Jensen’s inequality, it can be seen that $\phi(\theta\mid D)$ is always nonnegative, as explained in~\cite{feldman2011scalable,lucic2017training}.

            The main result in~\cite{feldman2011scalable,lucic2017training}
            is a coreset for $\phi$ that is denoted by $\phi_0$
            and is generalized in our paper to $\phi_{\xi}$, where $\xi\geq 0$. Using a generalization of~\cite[Theorem 14]{lucic2017training}, a coreset for $\phi_\xi$ is also a coreset for $e^{\xi}/(2\pi)$, where $e^0/(2\pi)\sim 0.11$; see Observation~\ref{obss}.

            Unfortunately, a $(1+\eps)$-approximation to $\phi(\theta\mid D)$ does not imply a $(1+\eps)$-approximation
            to the desired log-likelihood $L(\theta\mid D)$,
            if the first additive term in~\eqref{LL} (that is independent of $D$) is negative.
            In this case, we have an additional additive error.
            This is unavoidable in general, due to scaling issues as explained above;
            See Chapter~\ref{sec Conclusion and Open Problems} for related open problems.
            However, as shown in~\cite{lucic2017training}, if each eigenvalue of the covariance matrices of $\theta$ is at least $1/(2\pi)=e^0/2\pi$,
            the value $\phi(\theta\mid D)$ is indeed a $(1+\eps)$-approximation to $L(\theta\mid D)$.
            In particular, if the optimal $k$-GMM that is computed on the coreset satisfies this constrain, or its eigenvalues are rounded up,
            then we get a $(1+\eps)$-approximation for the original data $D$. We generalize this observation for $\xi>0$ in Observation~\ref{obss}.

        \subsection{Comparison to most related results in~\cite{lucic2017training}.}
            The suggested coresets in~\cite{lucic2017training} and its earlier version in~\cite{feldman2011scalable} can be considered as a special case of our reduction for the case of semi-spherical Gaussians. Formally, they
            approximates $\phi_0$ for eigenvalues between $\xi$ and $1/\xi$,
            where the coreset size has quadratic dependency on $1/\xi$.
            Formally it is an $\eps$-coreset for $(P,H_{k,d},\dist,\norm{\cdot}_\infty)$ as in Definition~\ref{coresetdef3}.

            For approximating negative-log-likelihood there is a lower bound of
            $\xi=1/(2\pi)$ as explained in the previous section, so the coreset in~\cite{lucic2017training}
            approximates the negative log-likelihood only for Gaussians
            whose corresponding eigenvalues are constants in the range $(1/2\pi,2\pi$).

            In contrast, our main reduction from projective clustering also implies
            a coreset for $\phi_{\xi}$ for \emph{any} $k$-GMM with arbitrary eigenvalues,
            and for arbitrarily small constant $\xi>0$ where the coreset size depends polynomially on $1/\xi$.
            In some sense, we remove the constraint on the eigenvalue to the $\xi$ in the $\phi_{\xi}$ function,
            which has much smaller impact on the approximation error $e^{\xi}/2\pi$, i.e., the constant approximation factor changed from $e^0/(2\pi)$ to $e^{\xi}/(2\pi)$ where $\xi$ is arbitrarily small constant.

            For the case of negative log-likelihood we get a coreset that approximates
            any $k$-GMM that satisfies the lower bound in~\cite{lucic2017training},
            but there is no required upper bound.
            In particular, there is no bound on the ratio between eigenvalues
            so the Gaussians may be of arbitrary shapes
            (not just unit spherical Gaussians as in~\cite{lucic2017training}).

            Our result can also be considered as a generalization of ~\cite{lucic2017training} when the queries are $k$ points instead of $k$ hyperplanes. Using our main reduction to projective clustering yields similar coresets that are indeed much smaller than the general case as explained in the next paragraph.


            \paragraph {Remaining Gaps.}
            The main limitation of our coreset is its size,
            which exponential in the dimension of the original space $d$
            and that the input must lie on $P\subseteq \br{-M,..M}^d$.
            Unfortunately, existing lower bounds for projective clustering implies
            that these properties are unavoidable to handle Gaussians of any shape; see Related Work.

            However, for the case of semi-spherical bounds,
            these restrictions are not needed,
            and the coreset size has polynomial dependency on $d$ with no restriction on the input.
            There is still a small gap of $O(\log n)$ in our coreset size
            and the corresponding result in~\cite{feldman2011scalable}, due to our usage of the reduction from $\ell_{\infty}$ to $\ell_1$ in Theorem~\ref{sensitivitylemmasquareddistances}.

            A natural open problem is to generalize our results for $j\in (1,d-2)$ dimensional affine subspaces,
            and to use our reduction for $j$-subspaces
            where $j$-eigenvalues of each Gaussians can be arbitrary small or large.

            Another (less significant) gap is that the coresets in this paper
            assumes $\xi>0$ and not $\xi=0$ as in~\cite{feldman2011scalable}.
            A more interesting result would be to generalize our solution for negative value of $\xi$,
            which would also reduce the lower bound of $1/(2\pi)$ for negative log-likelihood approximation.

    \section{Our Contribution}\label{sec_Our Contribution}
        Our main technical result is a generic reduction from coresets to projective clustering to coresets for $k$-GMMs as described in Section~\ref{sec Road Map}.
        The input is a set $D$ of points in $\REAL^d$, an approximation error and probability of failure $\eps,\delta\in(0,1)$, and an integer $k\geq1$.
        We assume that we are also given a coreset construction scheme (``black-box") that computes an $O(1)$-coreset of size $f(n)$ in time $t(n)$
        for the $\norm{\cdot}_{\infty}$-projective clustering problem as defined in the section~\ref{sec:proj}.
        The main implications of this reduction are then as follows:

        \begin{enumerate}
        \renewcommand{\labelenumi}{\theenumi}
        \renewcommand{\theenumi}{(\roman{enumi})}
        \item \label{res1} An algorithm that returns, with probability at least $1-\delta$,
            a mergable coreset $C\subseteq D$ that approximates the fitting cost $\phi(D,\theta)=\phi_\gamma(D,\theta)$ of \emph{any} $k$-GMM (with no restrictions on its eigenvalues),
            up to a factor of $(1\pm\eps)$. See Section~\ref{sec Likelihood to phi approx} for exact details.

            The size of $C$ and its construction time are $f(n)$ and $t(n)$ respectively,
            as the given coreset construction for projective clustering,
            up to factor that are near-linear in $f(n)$ and $t(n)$, and poly-logarithmic in $n$. See exact details in Theorem~\ref{evmaintheorem}.
        \item \label{res2} A proof that $C$ approximates the negative log-likellihood $L(\theta\mid D)$ of $D$ to any $k$-GMM whose smallest eigenvalue is at least, say,
            $0.160754=e^{0.01}/(2\pi)$. Here, $0.01$ can be replaced by $\ta$ above.
            See Theorem~\ref{offline}. In particular, there is no upper bound on the eigenvalues or the ratio between them.
            As an example, we use the coreset construction for projective clustering in~\cite{har2004no} to obtain a coreset $C$ as defined above
            of size $|C|\in\log^{O(1)} (n)/\eps^2$ for any constant (fixed) $k,d\geq 1$,
            under the assumption that $D$ is scaled to be in a polynomial grid,
            where every coordinate can be represented by $O(\log n)$ bits.
            More precisely, $D\subseteq \br{-n^c,\cdots,n^c}^d$ for some constant $c=O(1)$. See Theorem~\ref{thm33} and its proof for exact dependencies.
        \item \label{res3} Similar coresets for approximating the maximum (worst case fitting error) $L(\theta\mid \br{p})$ and $\phi(\theta\mid \br{p})$, respectively, over every input point $p\in D$.
            Here the construction is deterministic and the multiplicative approximation factor is constant.
            See Theorem~\ref{hyp} and Observation~\ref{obss} for details.
        \item Generalization of results~(i)-(iii) for big data is straight-forward by plugging the above coresets in traditional coreset merge-reduce techniques.
            Specifically, the construction of the above coresets can be maintained for a possibly infinite stream of points where $n$ is the number of points seen so far.
            The update time per point and the required memory is poly-logarithmic in $n$, and the overall running time is thus near-linear in $n$.
            If the input data is distributed in parallel to $M$ machines or threads, then the running time reduces by a factor of $M$, with no communication between the machines,
            except for transmitting their current coreset to a main server.
            If linear $(O(n))$ memory is allowed (e.g. using hard-drive) then maintaining the coreset after deletion of an input point is also possible in $\log ^{O(1)}n$ update time.
            In Section~\ref{sec Coresets for Streaming} we define the notion of \emph{mergable coresets} and provide a generic framework of independent interest on how to compute them in the streaming model.
        \item Experimental results, as well as comparison to previous coresets and uniform sampling are demonstrated on public datasets.
            As expected from many previous coresets experiments, e.g.~\cite{feldman2011scalable,lucic2017training},
            the theoretical upper bound for the worst case analysis are significantly pessimistic compared to real-world data.
            In particular, we ignored both the theoretical assumptions on the input
            (bounded and integer coordinate) and the $k$-GMMs (and their eigenvalues).
            Nevertheless, running existing optimization algorithms on the suggested coresets improves the approximation error up to a factor of $40$ compared to the state of the art.
        \item Open code of our coreset construction is provided to the community
            in order to reproduce and extend our preliminary experiments,
            and for the open problems and future research that are suggested in Section~\ref{sec Conclusion and Open Problems}.
        \end{enumerate}

        \section{Novel Technique: Reduction to Projective Clustering.\label{novel}}
            Our main technical result is the proof of Lemma~\ref{lemma cost_inf to dist_inf}.
            It describes general reduction from  $\ell_{\infty}$-coresets for the family of $k$-GMMs to $\ell_{\infty}$-coresets for projective clustering.
            More precisely, we first pad every $d$-dimensional point in the input set $D$ with zeroes and obtain a set $P$ in a $2d+1$ dimensional space.
            Next, we construct a $(1/3)$-coreset $C\subseteq P$ for projective clustering of $P$ using any existing algorithm (there is no assumption on the construction algorithm).
            That is, for every possible set $S$ of $k$ affine subspaces of $\REAL^{2d+1}$, the farthest point from $S$ in $P$ is at most $(1+1/3)$ times farther than the farthest point in the coreset $C$.
            Finally, we remove the zeroes from each point of $C$ to obtain a point in $D$ which we proved to be an $\ell_{\infty}$-coreset for $k$-GMM in Theorem~\ref{hyp2}.

            As explained in Section~\ref{sec Road Map}, this main result is combined with few existing coreset techniques:
            \begin{enumerate}
            \renewcommand{\labelenumi}{\theenumi}
            \renewcommand{\theenumi}{(\roman{enumi})}
            \item From maximum to sum of distances, with $\eps$ instead of constant factor approximation.
            \item From likelihood to $\phi$ error fitting.
            \item From mixture of $k$-Gaussians to $k$ subspaces.
            \item From inefficient $O(n^2)$ time construction to near-linear $n\log^{O(1)}n$ construction that also supports streaming,
            distributed and dynamic computations; see Section~\ref{sec Coresets for Streaming}.
            \end{enumerate}

            Since our coreset construction itself (ignoring the proofs) has nothing to do directly with the $k$-GMM problem or its corresponding Radial Basis Network,
            we believe that it can be used for many other non-convex kernels or networks.

    \section{Road Map}\label{sec Road Map}
            \tikzstyle{decision} = [ellipse, draw, fill=blue!20, text centered, minimum height=4.5em]
            \tikzstyle{block} = [rectangle, draw, fill=blue!20, text centered, minimum height=4.5em, rounded corners]
            \tikzstyle{cloud} = [rectangle, draw, fill=red!20,  text centered, minimum height=4.5em, rounded corners]
            \tikzstyle{line} = [draw, -latex']
            \renewcommand{\baselinestretch}{1.0}
            \begin{figure*}
            \centering
                \begin{tikzpicture}[node distance = 2cm, auto]
                    \node [block] (a) {\makecell[c]{$k$-GMM \\ $\varepsilon$-coreset for log-likelihood $L(\cdot)$\\ Definition~\ref{coresetdef}}};
                    \node [block, below of=a, node distance=3cm] (b) {\makecell[c]{From $k$-GMM to $\phi(\cdot)$\\$\varepsilon$-coreset for $\phi(\cdot)$ \\ Definition~\ref{coresetdef2}}};
                    \node [block, below of=b, node distance=3cm] (c) {\makecell[c]{From $\phi(\cdot)$ to $k$-SMM \\ $\varepsilon$-coreset for $\cost(\cdot)$ \\ Section~\ref{sec:three}}};
                    \node [cloud, below of=c, node distance=3cm] (d) {\makecell[c]{Coreset using Sensitivity and VC-dimension. \\ Definitions~\ref{sensitivity}~\ref{vdim}  }};
                    \node [cloud, below of=d, node distance=3cm] (e) {\makecell[c]{From Sensitivity to $\ell_\infty$-Coreset \\ Section~\ref{sec From Sensitivity to ell infty}}};
                    \node [cloud, left of=e, node distance=5cm] (f) {\makecell[c]{VC-dimension Bound\\ Section~\ref{sec Bound on the VC-Dimension}}};
                    \node [block, below of=e, node distance=3cm] (g) {\makecell[c]{From $k$-SMM to Projective Clustering $\ell_{\infty}$-Coreset\\ $\cost_{\infty}$ to $\dist_{\infty}$ \\ Section~\ref{sec:proj}  }};
                    \node [cloud, below of=g, node distance=3cm] (h) {\makecell[c]{From off-line Coresets to Streaming Coresets \\ Section~\ref{sec Coresets for Streaming}}};
                    \path [line] (a) -- node {Section~\ref{sec Likelihood to phi approx}} (b);
                    \path [line] (b) -- node {Lemma~\ref{lemma2}} (c);
                    \path [line] (c) -- node {Theorem~\ref{supsampe}} (d);
                    \path [line] (d) -- node {Lemma~\ref{sensitivitylemmasquareddistances2},Theorem~\ref{sensitivitylemmasquareddistances}} (e);
                    \path [line] (d) -- node {} (f);
                    \path [line] (e) -- node {Lemma~\ref{lemma cost_inf to dist_inf}} (g);
                    \path [line] (g) -- node {Algorithm~\ref{algsup}} (h);

                \end{tikzpicture}
            \caption[Article Proof Flow Chart]
            {
                Road-map diagram of the techniques and reductions in this thesis.
            } \label{fig road map}
            \end{figure*}
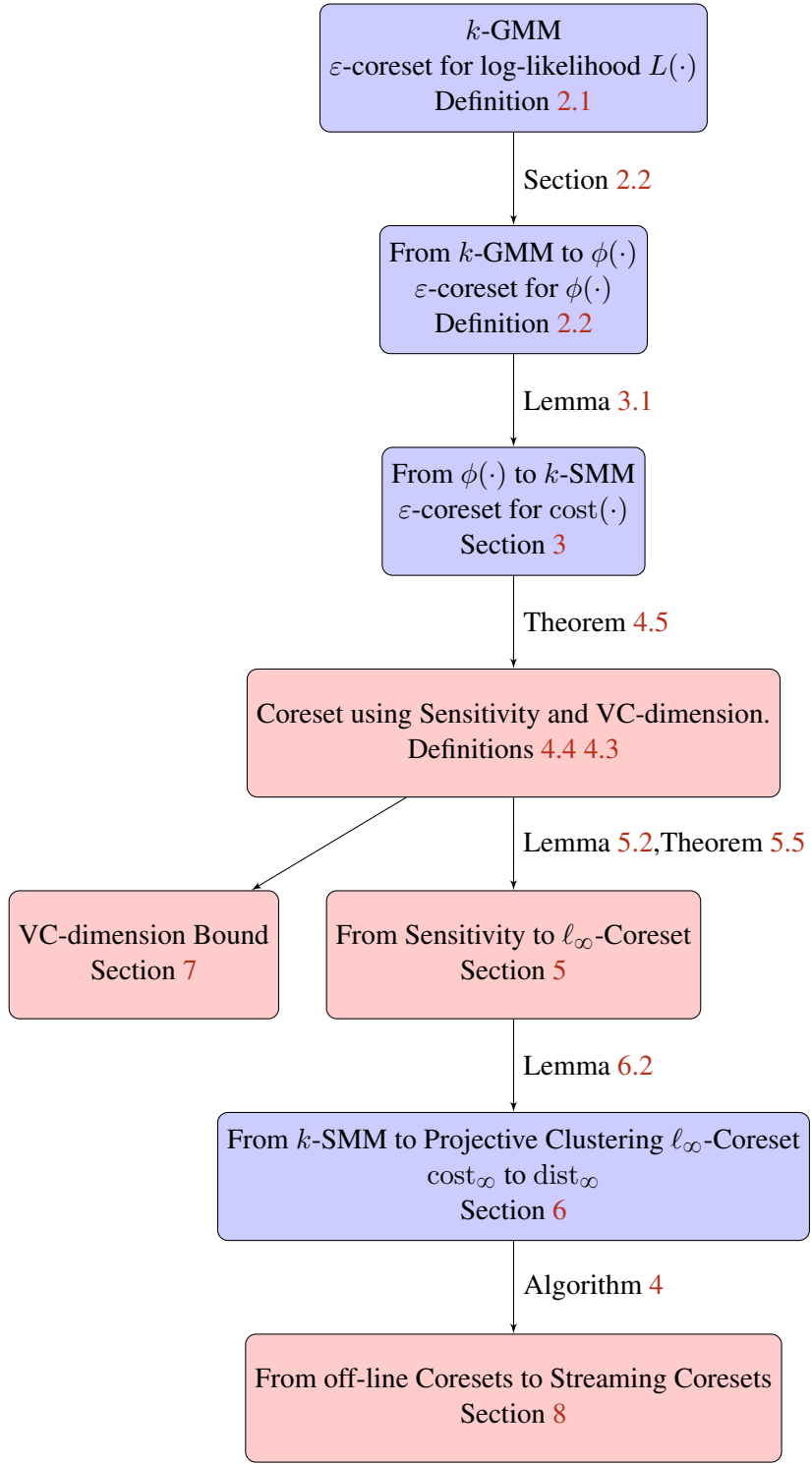
            \renewcommand{\baselinestretch}{1.67}

        To obtain small coresets efficiently,
        we suggest a framework that combines several coreset techniques and the following related reductions that are summarized in Fig.~\ref{fig road map}.

        \paragraph{From $L$ to $\phi_{\xi}$. }
            The reduction from the likelihood $L$ to the cost function $\phi$ was suggested in~\cite{feldman2011scalable} as explained in Section~\ref{sec From Likelihood to phi}.
            However the eigenvalues of the Gaussians are restricted to be in $[\xi,1/\xi]$,
            or in $[1/(2\pi),2\pi]$ for approximating $\phi$ and $L$ respectively.
            In Section~\ref{sec Likelihood to phi approx} we generalize $\phi$ to a related loss function $\phi_{\xi}$ that has similar properties to $\phi$.
            However, by using $\phi_{\xi}$ for an arbitrarily small constant $\xi>0$, instead of $\phi=\phi_0$ as in~\cite{lucic2017training},
            we obtain in the next sections a coreset that approximates \emph{any} $k$-GMM, with no upper/lower bound on its eigenvalues.
            The price is relatively small: the lower bound for the case of the non-negative likelihood approximation $L$ increases
            from $e^0/(2\pi)\sim 0.11$ to $e^{\xi}/(2\pi)\sim 0.11 e^\xi$ which is larger by an arbitrarily small constant that depends on $\ta$.

        \paragraph{Importance sampling. }
            The importance or \emph{sensitivity} of an input point $p\in D$ with respect to a query $\theta$ ($k$-GMM in our case) is its relative contribution to the overall loss;
            See Definition~\ref{sensitivity}.
            In the case of $k$-GMMs, it is the negative log-likelihood $L(\theta\mid \br{p})$ divided by its sum $L(\theta \mid D)$ over $p\in D$.
            The sensitivity~\cite{langberg2010universal,feldman2011unified} $s(p)$
            of a point is the maximum of this ratio over all possible $k$-GMMs.

            It was proven in~\cite{FL11,braverman2016new} that an $\eps$-coreset can be obtained for a given problem
            (query space as in Definition~\ref{def::query space}), with high probability,
            by sampling points i.i.d. from the input with respect to their sensitivity.
            Then we assign for each sample point a weight that is inverse proportional to its sensitivity.
            That is, a point with high sensitivity should be sampled with high probability and get a small weight.
            The number of sampled points should be proportional to $d't/\eps^2$ where $d'$ is related to the VC-dimension of the query space, and $t=\sum_{p\in P}s(p)$
            is the sum of sensitivities, called \emph{total sensitivity}.
            See Section~\ref{sec From Coresets to Sensitivity} for formal definitions and details.

            More generally, an upper bound to the sensitivity $s(p)$ of each point suffices, where a higher bound would yield higher sum $t$ and thus larger coreset.
            It is relatively easy to bound the corresponding VC-dimension as explained e.g. in~\cite{anthony2009neural}.
            We generalize the existing bound to handle weighted input points (that are needed for the streaming merge-reduce in Section~\ref{sec Coresets for Streaming}) and fixed some technical issues as explained in Section~\ref{sec Weighted Input}.

            This technique reduces the problem of computing a coreset for sum of likelihoods ($\norm{\cdot}_1$)
            to the problem of computing a bound on the sensitivity of each input point, so that the sum $t$ would still be small in order to get a small coreset.
            This is the challenge in the next sections.

        \paragraph{From $\norm{\cdot}_\infty$ to sensitivity bound. }
            A general approach for computing the above sensitivity bounds is to use an $\eps'$-coreset for the $\norm{\cdot}_{\infty}$
            version of the problem (max instead of sum of distances), where $\eps'=2$ or any other constant~\cite{feldman2013turning,VX12-soda}. The idea is to compute this coreset, remove it from the input, and continue recursively on the remaining points.
            The sensitivity of a point is then proved to be roughly $1/i$ where $i$ is the iteration that it was removed from the input set.
            The total sensitivity is then approximately $\sum_{i=1}^n 1/i=O(\log n)$ multiplied by the size of the $\norm{\cdot}_{\infty}$ coreset.
            See Section~\ref{sec From Sensitivity to ell infty} for exact and formal details.

            An important observation is that $\eps'$-coreset for $\norm{\cdot}_\infty$ suffices to obtain
            $\eps$-coreset for $\norm{\cdot}_1$ even for $\eps'>1$ and $\eps\ll 1$.
            For example, this may enable us to obtain coresets for $k$-means/median of size polynomial in $d$ and $k$ via this reduction with
            $\eps'\in O(1)$ (that corresponds to semi-spherical Gaussians as explained in~\cite{lucic2017training}, even though the corresponding $\eps'$-coreset for $k$-center is exponential in $d$ and $k$ for every $\eps'<1$~\cite{agarwal2005geometric}.

            The main disadvantage of this approach is the fact that we need to compute different $\norm{\cdot}_\infty$ coresets $O(n)$ times which results in $O(n^2)$ running time.
            See Lemma~\ref{sensitivitylemmasquareddistances2}.
            To obtain linear time, we use the streaming merge-reduce tree in Section~\ref{sec Coresets for Streaming}, even when the input is given off-line.
            We then need to compute the coresets only on small subsets of the input, which gives an algorithm whose running time is near-linear in $n$.
            See Theorem~\ref{thmstream}.

            The remaining challenge is thus to compute an $\norm{\cdot}_{\infty}$ coreset for the $\phi$ loss function of Gaussian.
            To do this, we use the main technical result of this paper, which is the reduction to projective clustering that is described in Section~\ref{novel} and Chapter~\ref{sec:proj}.

\chapter{Problem Statement}
The goal of this thesis is to provide coresets for approximating both the negative log-likelihood of a given mixture of Gaussians, and its related $\phi$-cost. These are defined in this chapter.

    \paragraph{Notation. }For an integer $n\geq1$ we define $[n]=\br{1,\cdots,n}$.
    For an integer $d\geq1$, the set of $n\times d$ real matrices is denoted by $\REAL^{n\times d}$.
    An {affine subspace} $S$ in $\REAL^d$ is a linear subspace of $\REAL^d$ that may be translated from the origin,
    i.e., $S=\br{Ax+v\mid x\in \REAL^d}$ for some matrix $A\in\REAL^{d\times d}$ and a vector $v\in\REAL^d$.
    A vector (point) $v\in \REAL^d$ is a column vector unless stated otherwise.
    A concatenation of two rows $x\in\REAL^d$ and $y\in\REAL^{n}$ is denoted by $(x \mid y)\in\REAL^{d+n}$.
    The exponent of $x\in\REAL^d$ is denoted by $\exp(x)=e^x$.
    For $a,b\in \REAL$ we define the interval $(a\pm b)=[a-b,a+b]$, e.g. $(1\pm\eps)=[1-\eps,1+\eps]$.

    \paragraph{Matrix factorizations. }
        Every matrix $A\in\R ^{n\times d}$ whose rank is $r\geq1$ has a \emph{thin} Singular Value Decomposition (thin SVD) $UDV^T=A$,
        where $U\in\R ^{n\times r}$ and $V\in\R ^{d\times r}$ such that $U^TU=I$ and $V^TV=I$,
        and $D\in\R ^{r\times r}$ is a diagonal matrix whose non-zero entries are $D_{1,1}\geq\cdots\geq D_{r,r}\geq0$ and called its \emph{singular values}.
        A matrix $\Sigma\in\REAL^{d\times d}$ is called a \emph{semi-positive definite covariance matrix}
        if and only if there is a matrix $A\in\REAL^{d\times d}$ such that $\Sigma=A^TA$,
        which is called the \emph{Cholesky Decomposition} of $\Sigma$.
        Hence, $\Sigma=VD^2V^T$ is the SVD of $\Sigma$,
        and $D_{i,i}$ is the $i$th eigenvalue of $\Sigma$ for every $i\in [r]$.
        If $\Sigma$ has a full rank it is called \emph{positive definite covariance matrix}.

    \section{Likelihood of Gaussians} \label{subsec Likelihood of Gaussians}
        \paragraph{Gaussian distribution. }
            Given a covariance positive definite matrix $\Sigma\in\REAL^{d\times d}$ and a vector $\mu\in\REAL^d$,
            the \emph{density function of a multivariate distribution} $\mathcal{N}(\mu,\Sigma)$ is defined for every $p\in\REAL^d$ as
            \[
            \pr(p\mid \mu,\Sigma)=\frac{\exp(-\frac{1}{2}(p-\mu)^T\Sigma^{-1}(p-\mu))}{\sqrt{\det(2\pi \Sigma)}},
            \]
            where $\mu$ represents the mean of a Gaussian, $\Sigma$ is its covariance matrix and $\det(\cdot)$ is the determinant operator.

        \paragraph{$k$-GMM. } A \emph{$k$-Gaussian mixture model} ($k$-GMM for short) in $\REAL^d$ is an ordered set
            \[
            \theta=\big((\omega_i,\Sigma_i,\mu_i)\big)_{i=1}^k=\big((\omega_1,\mu_1,\Sigma_1),\cdots,(\omega_k,\mu_k,\Sigma_k)\big)
             \]
             of $k$ tuples, where $\mu_i\in\REAL^d$, $\mathcal{N}(\mu_i,\Sigma_i)$ is a multivariate distribution as defined above
             for every $i\in\br{1,\cdots,k}$, and $\omega=(\omega_1,\cdots,\omega_k)\in[0,1]^k$ is a \emph{distribution vector},
             i.e., whose sum is $\norm{\omega}_1=\sum_{j=1}^k\omega_j=1$,

        \paragraph{The \emph{likelihood}} of sampling $p\in\REAL^d$ from such a mixture of Gaussians is the distribution
            \begin{align}
            \label{prr}\Pr( p\mid \theta)=\sum_{i=1}^k \omega_i \pr(p\mid \mu_i,\Sigma_i),
            \end{align}
            following e.g.~\cite{feldman2011scalable} and~\cite{lucic2017training}.
            By taking $-\ln(\cdot)$ of~\eqref{prr}, we obtain the \emph{negative log-likelihood} of $x$
            \[
            L(\br{p},\theta):=-\ln\Pr( p\mid \theta)
            =-\ln\sum_{i=1}^k \omega_i \Pr(p\mid \mu_i,\Sigma_i).
            \]
            Similarly, the probability that a set $P=\br{p_1,\cdots,p_n}$ of points in $\REAL^d$ was generated i.i.d.
            from the mixture of Gaussians $\theta$ is
            \[
            \Pr(P\mid \theta)=\prod_{p\in P}\Pr(p\mid \theta),
            \]
            and the corresponding negative log-likelihood is
            \begin{equation}\label{LLDef}
            \begin{split}
            L(P,\theta)&=-\ln(\Pr( P\mid \theta))=-\ln\left(\prod_{p\in P}\Pr( p\mid \theta)\right)=-\sum_{p\in P}\ln\Pr( p\mid \theta)\\
            &=\sum_{p\in P}L(\br{p},\theta)
            =-\sum_{p\in P}\ln\sum_{i=1}^k \omega_i \Pr(p\mid \mu_i,\Sigma_i)\\
            &=-\sum_{p\in P}\ln\sum_{i=1}^k \omega_i \cdot \frac{\exp(-\frac{1}{2}(p-\mu_i)^T\Sigma^{-1}(p-\mu_i))}{\sqrt{\det(2\pi \Sigma)}}.
            \end{split}
            \end{equation}

        \paragraph{A weighted set }
            is a pair $C'=(C,w)$ where $C$ is a non-empty ordered multi-set of points in $\REAL^d$,
            and $w:D\to (0,\infty)$ is a function that maps every $p\in C$ to $w(p)\geq 0$, called the \emph{weight} of $p$.
            A weighted set $(C,\mathbf{1})$ where $\mathbf{1}$ is the weight function $w:C\to\br{1}$
            that assigns $w(p)=1$ for every $p\in C$ may be denoted by $C$ for short.
            We denote the normalized average weight by
            \begin{equation}\label{eq801}
                \overline{w}(C')=\frac{\sum_{p\in D}w(p)}{\min_{q}w(q)},
            \end{equation}
            where the minimum is over every point $q\in D$ with positive weight $w(q)$.

            We generalize the definition of $L(C,\theta)$ from~\eqref{LLDef} for such a weighted set $C'=(C,w)$ by defining the weighted sum of negative log-likelihoods,
            \begin{equation}\label{lw}
            \begin{split}
            L(C',\theta)&=\sum_{p\in C}w(p)L(\br{p},\theta)
            =-\sum_{p\in C}w(p)\ln\sum_{i=1}^k \omega_i \cdot \frac{\exp(-\frac{1}{2}(p-\mu_i)^T\Sigma^{-1}(p-\mu_i))}{\sqrt{\det(2\pi \Sigma)}}.
            \end{split}
            \end{equation}

            Since some of the coresets in this paper have restrictions on the approximated Gaussians,
            we define the set $\vartheta_k(0)$ of all possible $k$-GMMs,
            and $\vartheta_k(\xi)$ as those $k$-GMMs $\theta=\big((\omega_i,\Sigma_i,\mu_i)\big)_{i=1}^k$ in $\REAL^d$ such that all
            the eigenvalues of $\Sigma_i$ are at least $\ta$, for every $i\in[k]$.

            We also request that the sum of weights of the points in the coreset
            would be the same as the original sum $n$ of points.
            The latter property will be used in the streaming version
            in Section~\ref{sec Coresets for Streaming} where we construct
            "coresets for coresets" and wish to keep the total weight constant.

            We are now ready to define an $\eps$-coreset for our main problem of approximating every $k$-GMM in a given set $\vartheta_\xi(k)$ of $k$-GMMs whose eigenvalues are bounded by $\xi$.
            \begin{definition}[$\eps$-coreset for log-likelihood\label{coresetdef}]
            Let $k\geq1$ be an integer, $\ta'>0$, and $\vartheta_k({\ta'})$ denote the union over every
            $k$-GMM $\theta=\big((\omega_i,\Sigma_i,\mu_i)\big)_{i=1}^k$ in $\REAL^d$ such that all the eigenvalues of $\Sigma_i$ are at least $\ta'$, for every $i\in[k]$.

            Let $P$ be a set of $n$ points in $\REAL^d$ and $\eps>0$. A weighted set $C'=(C,u)$
            is an $\eps$-coreset of $(P,\vartheta({\ta'}),L,\norm{\cdot}_1)$  if $C\subseteq P$, $\sum_{q\in C}u(q)=n$,
            and for \emph{every} $k$-GMM $\theta\in \vartheta_k({\ta'})$ we have
            \begin{equation}\label{core}
            (1-\eps)L(C',\theta) \leq L(P,\theta)\leq (1+\eps)L(C',\theta).
            \end{equation}
            \end{definition}
            Although $L$ is defined to be negative log-like-likelihood, the coreset $C'$ above approximated the (non-negative) log-likelihood $(-L)$ since~\eqref{core} implies
            \[
            |-L(P,\theta)-\big(-L(C',\theta)\big)|=|L(P,\theta)-L(C',\theta)|\leq \eps L(C',\theta)=\eps |-L(C',\theta)|.
            \]

    \section{From Likelihood to $\phi$ approximation} \label{sec From Likelihood to phi}\label{sec Likelihood to phi approx}
        Inspired by~\cite{feldman2011scalable}, let $\ta>0$ and define
        \begin{equation}\label{zdef}
        Z(\theta):=\sum_{i=1}^k\frac{ \omega_i e^{\ta}}{\sqrt{\det(2\pi \Sigma_i)}}.
        \end{equation}
        For every $i\in[k]$, let
        \begin{equation}\label{ot}
        \omega'_i:=\frac{\omega_i e^\ta }{Z(\theta)\sqrt{\det(2\pi \Sigma_i)}}.
        \end{equation}
        Hereby $Z$ is a normalizer that asserts that $\omega'$ is a distribution vector,
        and a sufficiently small constant $\ta$ can be considered as scaling of the data by $e^{\ta}\sim (1+\ta)$
        or multiplicative approximation to the likelihood as will be shown in Observation~\ref{obss}. Finally, we define for every $p\in \REAL^d$
        \begin{equation}\label{phi}
        \phi_{\ta}(\br{p},\theta)
        =-\ln\sum_{i=1}^k \omega'_i\exp\left(-\frac{1}{2}(p-\mu_i)^T\Sigma_i^{-1}(p-\mu_i)-\ta\right),
        \end{equation}
        and $\phi(\br{p},\theta)=\phi_{\ta}(\br{p},\theta)$ for short.

        The relation between the negative log-likelihood and $\phi$ as was shown in~\cite{feldman2011scalable} is
        \[
        \begin{split}
        L(\br{p},\theta)&=-\ln\sum_{i=1}^k \omega_i f_{(\nu_i,\Sigma_i)}(p)\\
        &=-\ln\sum_{i=1}^k \omega_i\frac{\exp\left(-\frac{1}{2}(p-\mu_i)^T\Sigma_i^{-1}(p-\mu_i)\right)}{\sqrt{\det(2\pi \Sigma_i)}}\\
        &=-\ln Z(\theta)\sum_{i=1}^k \omega'_ie^{-\ta} \exp\left(-\frac{1}{2}(p-\mu_i)^T\Sigma_i^{-1}(p-\mu_i)\right)\\
        &=-\ln Z(\theta)-\ln\sum_{i=1}^k \omega'_i\exp\left(-\frac{1}{2}(p-\mu_i)^T\Sigma_i^{-1}(p-\mu_i)-\ta\right)\\
        &=-\ln Z(\theta)+\phi(\br{p},\theta).
        \end{split}
        \]
        Letting $\phi(P,\theta):=\sum_{p\in P}\phi(\br{p},\theta)$ yields
        \begin{equation}\label{LL}
        L(P,\theta)=\sum_{p\in P}L(\br{p},\theta)=-n\ln Z(\theta)+\phi(P,\theta).
        \end{equation}
        For a weighted set $(C,w)$ we generalize the definition of $\phi(C,\theta)$ as we did in~\eqref{lw}, by letting
        \[
        \phi((C,w),\theta)=\sum_{x\in C}w(x)\phi(x,\theta)
        =\sum_{p\in C}-w(p)\ln\sum_{i=1}^k \omega'_i\exp\left(-\frac{1}{2}(p-\mu_i)^T\Sigma_i^{-1}(p-\mu_i)-\ta\right).
        \]

        The coreset for the $\phi$ cost function is similar to the log-likelihood as explained above.
        Nevertheless, for the case of $\phi$ we would have coresets for any $k$-GMM. Moreover,
        it would be easier to work with this function in the rest of the paper.
        The justification and reduction for the likelihood is explained in the next section.

        \begin{definition}[$\eps$-coreset for $\phi$\label{coresetdef2}]
        Let $P$ be a set of $n$ points in $\REAL^d$ and $\eps,\ta>0$. A weighted set $C'=(C,u)$ is an $\eps$-coreset of
        $(P,\vartheta_k(0), \phi_{\ta},\norm{\cdot}_1)$ if $\sum_{q\in C}u(q)=n$, $C\subseteq P$, and for \emph{every} $k$-GMM $\theta$ we have
        \[
        \phi_{\ta}(C',\theta)\leq \phi_{\ta}(P,\theta)\leq (1+\eps)\phi_{\ta}(C',\theta).
        \]
        \end{definition}

        As noted in~\cite[Section 2]{feldman2011scalable}, $Z(\theta)$ can be computed \emph{exactly} and independently of the set $P$.
        Furthermore, the function $\phi$ captures all dependencies of $L(P,\theta)$ on $\theta$.
        However, a $(1+\eps)$ multiplicative factor approximation for $\phi(P,\theta)$ is not such an approximation for the likelihood
        (or negative log-likelihood) $L(P,\theta)$: while the left hand side of~\eqref{LL} $L(P,\theta)$ is always non-negative,
        its right hand side $-n\ln Z(\theta)+\phi(P,\theta)$ may be the sum of negative and positive term.
        In fact, such an approximation for the likelihood is impossible in the sense that it can be reduced to no approximation ($\eps=0$)
        via scaling of the input data to be contained in an infinitesimally small ball as explained in~\cite{feldman2011scalable}.

        However, if either the input set $P$ or the $k$-GMM $\theta$ is scaled such that $Z(\theta)\leq 1$,
        then the right hand side of~\eqref{LL} is non-negative and a $(1+\eps)$-approximation to $\phi(P,\theta)$ is indeed a $(1+\eps)$-approximation to $L(P,\theta)$.
        This occurs e.g. if all the eigenvalues $\sigma_{i,1},\cdots,\sigma_{i,d}$ of $\Sigma_i$ are greater than $e^\ta/(2\pi)\leq 0.18$ for every $i\in[k]$.
        More generally, $Z(\theta)\leq1$ if $\sqrt{\det(2\pi \Sigma_i)}\geq e^{\ta}$ which always hold in this case since
        \[
        \begin{split}
        \det(2\pi \Sigma_i)&=\exp(\ln \det(2\pi \Sigma_i))
        =\exp(\ln \prod_{j=1}^d (2\pi \sigma_{i,j}))\\
        &=\exp(\sum_{j=1}^d \ln  (2\pi \sigma_{i,j}))
        =\exp( d\ln(2\pi)+\sum_{j=1}^d  \ln(\sigma_{i,j})
        \geq e^{\ta}.
        \end{split}
        \]
        We obtain a generalization of~\cite[Theorem 14]{lucic2017training} for $\ta\in(0,1/10]$ which we summarize as follows.
        That is, if $Z(\theta)\leq 1$, e.g. all the eigenvalues of $\Sigma_i$ are greater than $0.18$, for every $i\in[k]$, then for every weighted set $C'$,
        \[
        \phi(C',\theta)\leq \phi(P',\theta)\leq (1+\eps)\phi(C',\theta)
        \]
        implies
        \[
        L(C',\theta)\leq L(P',\theta)\leq (1+\eps)L(C',\theta).
        \]
        This gives the following reduction that motivates the construction of coresets of $(P,\vartheta_k(0),\phi_{\ta},\norm{\cdot}_1)$ for the rest of the paper.
        \begin{observation}\label{obss}
        Let $P\subseteq\REAL^d$, $\ta>0$, $\ta'=\frac{e^{\ta}}{2\pi}$ and $\eps>0$.
        Then, for every $z\in [0,\infty]$, an $\eps$-coreset of $(P,\vartheta_k(0),\phi_{\ta},\norm{\cdot}_z)$ is an $\eps$-coreset of $(P,\vartheta_k(\ta'),L,\norm{\cdot}_z)$.
        \end{observation}

\chapter{From $k$-GMM to $k$-SMM\label{sec:three}}
    In this section we reduce the family of $k$-GMM from machine learning to $k$-SMM that are more related to computational geometry,
    and projective clustering in particular. The Euclidean distance between a point $p\in\REAL^d$ and a set $S\subseteq \REAL^d$ is denoted by $$\dist(p,S)=\inf_{s\in S}\norm{p-s}_2.$$
    Its squared is denoted by $\dist^2(p,S)=\big(\dist(p,S)\big)^2$.

    Let $k\geq1$ be an integer.
    A \emph{subspace $k$-mixture model} in $\REAL^d$, or $k$-SMM for short, is a tuple $y=(W,\omega_1,\cdots,\omega_k,S_1,\cdots,S_k)$ where $W\geq 1$,
    $\omega$ is a distribution vector, and each $S_i$ is an affine linear subspace of $\REAL^d$, for every $i\in[k]$.
    We also define $S(y)=\bigcup_{i=1}^k S_i$, so $\dist(p,S(y))=\dist(p,S_1\cup\cdots \cup S_k)=\min_{i\in [k]} \dist(p,S_i)$.

    The squared distance from a point $p\in\REAL^d$ to the $k$-SMM $y$, is defined by
    \[
        \cost(p,y)=-\ln \sum_{i=1}^k \omega_i \exp(-W\dist^2(p,S_i)),
    \]
    which is similar to $\phi(\br{p},\theta)$ as defined in~\eqref{phi}.

    The following corollary follows from replacing $d$ with $2d+1$, $j=d+1$, $r=d+1$, $m=d$, and $k=d-1$ in~\cite{cute}.
    It allows us to compute coresets for subspaces instead of GMMs,
    while assuring a lower bound on the distance between the farthest point from the subspace.

    \begin{lemma}[From $k$-GMM to $k$-SMM]\label{lemma2}
    Let $p\in\REAL^d$, and $p'=(p^T \mid 0,\cdots,0)^T\in\REAL^{2d+1}$.
    Let $\ta\geq 0$, and let $\theta$ be a $k$-GMM in $\REAL^d$.
    Then there is a $k$-SMM $y=(W,\omega_1,\cdots,\omega_k,S_1,\cdots,S_k)$ in $\REAL^{2d+1}$ such that \begin{equation}\label{WW123}
                W\dist^2(p',S(y))\geq \ta,
    \end{equation}
    and
    \[
    \phi_{\ta}(\br{p}, \theta)=\cost(\br{p'},y).
    \]
    \end{lemma}

    \begin{proof}
    Identify $\theta=\big((\omega_i,\Sigma_i,\mu_i)\big)_{i=1}^k$. For every $i\in[k]$, let $\sigma_{i,1}$ denote the largest singular value of $\Sigma^{-1}_i$, and
    \begin{equation}\label{WW}
    W=\frac{\max_{i\in[k]} \sigma^2_{i,1}}{2}.
    \end{equation}

    Put $i\in[k]$. Since $\Sigma_i$ is a positive definite covariance matrix, $\Sigma^{-1}_i/(2W)$ has a Cholesky Decomposition
    \begin{equation}\label{defata}
    A^TA=\frac{\Sigma^{-1}_i}{2W},
    \end{equation}
    for some $A\in\REAL^{d\times d}$. The largest singular value of $A$ is bounded by $1$, since for every unit vector $z\in\REAL^d$,
    \begin{equation}\label{AA}
    \norm{Az}^2=z^TA^TAz=\frac{z^T\Sigma^{-1}_iz}{2W}\leq \frac{\sigma^2_{i,1}}{2W}\leq 1,
    \end{equation}
    where the last inequality is by~\eqref{WW}.

    Let $L\in\REAL^{d\times d}$ be a matrix such that $LL^T=I-AA^T$ is the Cholesky decomposition of $I-AA^T$.
    It exists, since $I-AA^T$ is semi-positive definite matrix. Indeed, let $QDV^T=A$ be the Singular Value Decomposition (SVD) of $A$,
    and $D_{j,j}$ denote the $j$th diagonal entry of $D$ (the $j$th largest singular value of $A$) for every $j\in[d]$.
    By~\eqref{AA} we have that $D_{j,j}\leq  1$, so we can define the diagonal matrix $\sqrt{I-DD^T}\in\REAL^{d\times d}$,
    whose $j$th diagonal entry is $\sqrt{1-D_{j,j}^2}$ for every $j\in[d]$. Hence, the Cholesky decompostion of $I-AA^T$ is $I-AA^T=Q(I-DD^T)Q^T=LL^T$ as claimed,
    for the $d\times d$ matrix $L=Q\sqrt{I-DD^T}$.

    Similarly, let $E=I_{*,1:d}$ denote the first $d$ columns of the identity $2d\times 2d$ matrix, and let $Y\in\REAL^{2d \times d}$
    such that $YY^T=I-EE^T$ is the Cholesky Decomposition of $I-EE^T$. Let $B=[E \mid Y]\in\REAL^{2d\times 2d}$. Since $BB^T=EE^T+YY^T=I$,
    we have that $B^TB=I$ and $E=BI_{*,1:d}$. By defining the $2d\times d$ matrix $U=B[A \mid L]^T$ we obtain
    \[
    U^TU=[A \mid L][A \mid L]^T=AA^T+LL^T=I,
    \] and
    \begin{equation}\label{ute}
    U^TE=U^TB I_{*,1:d}=[A \mid L]B^TB I_{*,1:d} =[A \mid L] I_{*,1:d}=A.
    \end{equation}

    By letting $p\in\REAL^d$, $T_i\subseteq\REAL^{2d}$ denote the $d$-dimensional subspace that is spanned by the columns that are orthogonal to $U$,
    and $T_i+E\mu_i=\br{t+E\mu_i\mid t\in T_i}$ denote its translation by $E\mu_i$, we obtain
    \begin{align}
    \label{ute1}\frac{1}{2}(p-\mu_i)^T\Sigma^{-1}_i (p-\mu_i)
    &=W (p-\mu_i)^TA^TA(p-\mu_i)\\
    \label{ute3}&=W\norm{A(p-\mu_i)}^2=W\norm{U^TE(p-\mu_i)}^2\\
    \label{ute4}&=W\dist^2(E(p-\mu_i),T_i)=W\dist^2(Ep,T_i+E\mu_i),
    \end{align}
    where~\eqref{ute1} is by~\eqref{defata},~\eqref{ute3} is by~\eqref{ute}, and~\eqref{ute4} is since $U$ has $d$ orthogonal columns.

    Finally, we add another entry to every vector $t$ in the affine subspace $T_i+E\mu_i$,
    \[
    S_i:=\br{\left(t^T \mid \sqrt{\frac{\xi}{W}}\right)\mid t\in T_i+E\mu_i}.
    \]
    That is, $S_i$ is a translation of $T_i+E\mu_i$ along a new axis.
    By letting $x$ denote the projection (closest point) of $Ep$ onto
    $T_i+E\mu_i$ and $p'=(Ep  \mid 0)\in\REAL^{2d+1}$, we obtain by the Pythagorean Theorem
    \[
    \dist^2(Ep,T_i+E\mu_i)+\frac{\xi}{W}
    =\norm{Ep-x}^2+\frac{\xi}{W}
    =\norm{p'-({x} \mid 0)}^2+\dist^2((x\mid 0),S_i)
    =\dist^2(p',S_i).
    \]
    Plugging the last equality after~\eqref{ute4} yields
    \begin{equation}\label{tofin}
    \frac{1}{2}(p-\mu_i)^T\Sigma^{-1}_i (p-\mu_i)+\xi
    =W\dist^2(x,S_i).
    \end{equation}
    This proves~\eqref{WW123} as
    \begin{align}
    W\dist^2(p',S(y))
    \nonumber&=\min_{i\in[k]}W\dist^2(x,S_i)\\
    \label{W3}&=\min_{i\in[k]}\frac{1}{2}(p-\mu_i)^T\Sigma^{-1}_i (p-\mu_i)+\xi\\
    \nonumber&\geq \xi,
    \end{align}
    where~\eqref{W3} holds by~\eqref{tofin}.

    The $k$-SMM $y=(W,\omega'_1,\cdots,\omega'_k,S_1,\cdots,S_k)$ satisfies the lemma as
    \[
    \begin{split}
    \phi(\br{p},\theta)
    &=-\ln\sum_{i=1}^k \omega'_i\exp\left(-\frac{1}{2}(p-\mu_i)^T\Sigma_i^{-1}(p-\mu_i)-\ta\right)\\
    &=-\ln \sum_{i=1}^k \omega'_i \exp(-W\dist^2(x,S_i))
    =\cost(x,y),
    \end{split}
    \]
    where the first equality is by~\eqref{phi}, and the second is by~\eqref{tofin}.
    \end{proof}

\chapter{From Coresets to Sensitivity and VC-Dimension} \label{sec From Coresets to Sensitivity}
    Our next goal is to compute coresets for $k$-SMMs, which would be used for $k$-GMMs as explained in the previous chapter. To this end, we introduce an existing generic framework for coreset constructions.

    A coreset is problem dependent, and the problem is defined by four items:
    the input weighted set, the possible set of queries (models) that we want to approximate,
    the cost function per point, and the overall loss calculation.
    In this thesis, the input is usually a set of points in $\R^d$ or $\R^{2d+1}$,
    but for the streaming case in Section~\ref{sec Coresets for Streaming}
    we compute coreset for union of (weighted) coresets and thus weights will be needed.
    The queries are either Gaussians or subspaces (for projective clustering),
    the cost/kernel $f$ is the $phi$ function, negative log-likelihood $L$,
    or Euclidean distance, and the loss would be either maximum or sum over costs in Sections~\ref{sec From Sensitivity to ell infty} and~\ref{sec:proj} respectively.

    \begin{definition}[query space]\label{def::query space}
    Let $\CC$ be a (possibly infinite) set called \emph{query set},  $P'=(P,w)$ be a weighted set called the \emph{input set}, $\ff:P\times \CC\to [0,\infty)$ be called a \emph{kernel or cost function}, and $\loss$ be a function that assigns a non-negative real number for every real vector.
    The tuple $(P',\CC,\ff,\loss)$ is called a \emph{query space}. For every weighted set $C'=(C,u)$ such that $C=\br{p_1,\cdots,p_m}\subseteq P$,
    and every $y\in \CC$ we define the overall fitting error of $C'$ to $y$ by
    \[
    \ff_{\loss}(C',y):=\loss((w(p)\ff(p,y))_{p\in C})=\loss(w(p_1)\ff(p_1,y),\cdots,w(p_m)\ff(p_m,y)).
    \]
    \end{definition}

    A coreset (or core-set) that approximates a set of models (queries) is defined as follows, where Definitions~\ref{coresetdef} and~\ref{coresetdef2} are special cases. Recall that weighted set were defined in Section~\ref{subsec Likelihood of Gaussians} as a pair that consists of a set $P$ and a (weight) function $w:P\to[0,\infty)$.

    \begin{definition}[$\eps$-coreset\label{coresetdef3}]
    Let $P'=(P,w)$ be a weighted set. For an approximation error $\eps> 0$, the weighted set $C'=(C,u)$ is called an \emph{$\eps$-coreset}
    for a query space $(P',Y,f,\loss)$, $C\subseteq P$, and for every $y\in Y$ we have
    \[
    \ff_{\loss}(P',y)\in (1\pm \eps)\ff_{\loss}(C',y).
    \]
    \end{definition}

    \renewcommand{\time}{\mathrm{time}}

    The dimension of a query space $(P,\CC,\ff)$ is the VC-dimension of the range space that it induced, as defined below.
    The classic VC-dimension was defined for sets and subset and here we generalize it to query spaces, following~\cite{feldman2011unified}.
    \begin{definition}[Dimension for a query space~\cite{Vap71a,braverman2016new,feldman2011unified}\label{vdim}]
    For a set $P$ and a set $\ranges$ of subsets of $P$, the VC-dimension of $(P,\ranges)$ is the size $|C|$ of the largest subset $C\subseteq P$ such that
    \[
    |\br{C\cap \range \mid \range\in\ranges}|= 2^{|C|}.
    \]

    Let $Y$ be a set and $f:P\times Y\to\REAL$.
    For every query $y\in \CC$, and $r\in\REAL$ we define the set
    $$\range_{P,\ff}(\c,r):=\br{p\in P\mid \ff(p,\c)\leq r}.$$
    and
    \[
    \ranges:=\ranges(P,Y,f):=\br{C\cap \range_{P,\ff}(\c,r)\mid C\subseteq P, \c\in \CC, r\in\REAL}.
    \]
    The \emph{dimension} of $(P,\CC,\ff)$ is the VC-dimension of $(P,\ranges)$.
    \end{definition}

    In this thesis we use the general reduction for computing coresets
    for sum over the cost of each point,
    by bounding its sensitivity (importance) as defined below.
    The size of the coreset then depends near linearly on the sum of these bounds
    via~\cite{braverman2016new} following the quadratic bound in~\cite{feldman2011unified}.
    The rest of the thesis will be devoted mainly to compute such a bound.

    \begin{definition}[sensitivity\label{sensitivity}]
    Let $P'=(P,w)$ be a weighted set, and $((P,w),Y,\cost,\norm{\cdot}_1)$ be a query space. 
    The function $s^*:P\to[0,\infty)$ is the \emph{sensitivity} of $(P',Y,\cost)$ if
    \[
    s^*(p)= \sup_{y}\frac{w(p)\cost(p,y)}{\sum_{q\in P}w(p)\cost(q,y)},
    \]
    for every $p\in P$, where the $\sup$ is over every $y\in Y$ such that the denominator is non-zero.

    The function $s:P\to[0,\infty)$ is a \emph{sensitivity bound} for $(P',Y,\cost)$ if $s(p)\geq s^*(p)$ for every $p\in P$. The total sensitivity of $s$ is defined as
    \[
    t=\sum_{p\in P}s(p).
    \]
    \end{definition}

    The following theorem proves
    that a coreset can be computed by sampling according to sensitivity of points.
    The size of the coreset depends on the total sensitivity
    and the complexity (VC-dimension) of the query space,
    as well as the desired error $\eps$ and probability $\delta$ of failure.
    \begin{theorem}[\cite{braverman2016new}\label{supsampe}]
    Let
    \begin{itemize}
    \item $((P,w),\CC,\cost,\norm{\cdot}_1)$ be a query space, and $n=|P|$.
    \item $\ff:P\times Y\to[0,\infty)$ such that for every $p\in P$ and $y\in Y$,
    \[
    \ff(p,y)
    =\begin{cases}
    \frac{w(p)\cost(p,y)}{\sum_{q\in P}w(q)\cost(q,y)} & \sum_{q\in P}w(q)\cost(q,y)>0\\
    0 & \sum_{q\in P}w(q)\cost(p,y)=0,
    \end{cases}
     \]
     \item $d'$ be the dimension of $(P,\CC,\ff)$.
    \item $s:P\to[0,\infty)$ be a sensitivity bound of $((P,w),Y,\cost)$, and $t= \sum_{p\in P}s(p)$ be its total sensitivity.
    \item $\eps,\delta\in(0,1)$,
    \item $c>0$ be a universal constant that can be determined from the proof,
    \item
    \[
    m\geq \frac{c(t+1)}{\eps^2}\left(d'\log (t+1)+\log\left(\frac{1}{\delta}\right)\right), and
    \]
    \item $(C,u)$ be the output weighted set of a call to~$\impalg(P,w,s,m)$; see Algorithm~\ref{Alg_Coreset}.
    \end{itemize}

    Then $(i)$--$(v)$ hold as follows.
    \begin{enumerate}
    \renewcommand{\labelenumi}{\theenumi}
    \renewcommand{\theenumi}{(\roman{enumi})}
    \item With probability at least $1-\delta$, $C$ is an $\eps$-coreset of $((P,w),Y,\cost,\norm{\cdot}_1)$.
    \item $|C|=m$.
    \item $(C,u)$ can be computed in $O(n)$ time, given $(P,w,s,m)$.
    \item $u(p)\in [w(p), \sum_{q\in P}w(q)/m]$ for every $p\in C$.
    \item $\sum_{p\in P}w(p)=\sum_{q\in C}u(q)$.
    \end{enumerate}
    \end{theorem}
    The last two properties would be required to support streaming in
    Section~\ref{sec Coresets for Streaming}, and keep the overall weight of the coreset,
    that depends on sum and maximum over input weight.

    \begin{algorithm}
    \caption{$\impalg(P,w,s,m)$}
    \label{Alg_Coreset}
    {\begin{tabbing}
    \textbf{Input: }\quad\=A finite set $P\subseteq\REAL^d$, where $\sum_{p\in P}w(p)>0$\\\>
       $s:P\to[0,\infty)$, and an integer $m\geq 1$.\\
    \textbf{Output:} \>A weighted set $(C,u)$ that satisfies Theorem~\ref{supsampe}.
    \end{tabbing}\vspace{-0.3cm}}
    \tcp{Add small importance to each point to avoid huge coreset weights}
    $s'(p):=s(p)+\cfrac{w(p)}{\sum_{q\in P}w(q)}$\\
    \tcp{Add very important points to the coreset to avoid huge weights}
    $C:=\br{p\in P\mid \frac{s'(p)}{\sum_{q\in P} s'(q)} \geq \frac{1}{m}}$\\
    \For{ every $p\in C$}{
    $u'(p):=w(p)$
    }
    $Q:=P\setminus C$\\
    \For{ $m$ iterations }{
     Sample a point $q$ from $Q$, where $q=p$ with probability
     $\Pr(p):=\cfrac{s'(p)}{\sum_{p'\in Q}s'(p')}$\\
     $C:=C\cup \br{q}$\\
     $u'(q):=\frac{w(q)}{m\cdot \pr(q)}$\\
     }
     \For{ every $p\in C$}{
     \tcp{sum of weights in the coreset and input set should be the same}
      $u(p):=u'(p)\cdot \cfrac{\sum_{p'\in P}w(p')}{\sum_{q\in C}u'(q)}$
      }
    \Return $(C,u)$
    \end{algorithm}

    \chapter{From Sensitivity to $\ell_\infty$-Coreset} \label{sec From Sensitivity to ell infty}
    In order to bound sensitivities, we generalize the reduction that was suggested in~\cite{VX12-soda}
    to compute sensitivities via $\ell_{\infty}$ coreset,
    where instead of approximating the sum of fitting costs or distances we approximate their maximum.

    \section{Non-weighted input}
        The original reduction from sensitivity to $\ell_{\infty}$ coreset in~\cite{agarwal2005geometric}
        was for a specific problem and for non-weighted data. For simplicity and intuition,
        we first generalize it to any query space and only then reduce weight weights
        to non-weighted weights, following the ideas in~\cite{feldman2013turning}.
        To this end, we use the following definition of a coreset scheme as an algorithm that computes coresets.

        \begin{definition}[coreset scheme\label{coresc}]
        Let $(P,Y,\cost,\loss)$ be a query space such that $P$ is an (unweighted, possibly infinite) set.
        Let $\size:[0,\infty)^4\to[1,\infty)$, $\time:[0,\infty)^4\to[0,\infty)$.
        Let $\coralg$ be an algorithm that gets as input a weighted set $Q'=(Q,w)$ such that $Q\subseteq P$, an approximation error $\eps>0$ and a probability of failure $\delta\in(0,1)$.
        The tuple $(\coralg, \size,\time)$ is called an $(\eps,\delta)$-\emph{coreset scheme} for $(P,Y,\cost,\loss)$ and some $\eps,\delta>0$, if (i)-(iii) hold as follows:
        \begin{enumerate}
        \renewcommand{\labelenumi}{\theenumi}
        \renewcommand{\theenumi}{(\roman{enumi})}
        \item A call to $\coralg(Q',\eps,\delta)$ returns a weighted set $(C,u)$.
        \item With probability at least $1-\delta$, $(C,u)$ is an $\eps$-coreset of $(Q',Y,\cost,\loss)$.
        \item $(C,u)$ can be computed in $\time(|Q|,\overline{w}(Q'),\eps,\delta)$ time and its size is
        \[
        |C|\leq \size(|Q|,\overline{w}(Q'),\eps,\delta).
        \]
        \end{enumerate}
        \end{definition}

        \begin{algorithm}
            \caption{$\red(P,\eps,\delta,\linfcore)$}
            \label{algred}
            {\begin{tabbing}
            \textbf{Input: }\quad\=A finite set $P\subseteq\REAL^d$, an approximation error $\eps>0$, and probability $\delta$ of failure.\\
            \textbf{Required: }\quad\=A coreset scheme $\linfcore$ for $(P,Y,\cost,\norm{\cdot}_\infty)$.\\
            \textbf{Output:} \>A sensitivity bound $s:P\to(0,\infty)$ that satisfies Theorem~\ref{sensitivitylemmasquareddistances}.
            \end{tabbing}\vspace{-0.3cm}}
            $P_1:=P;$\quad $i:=1$\\
            \While{$|P_i|\geq 1$ \tcp{$P_i$ is not an empty set}}{
            $S_i:=\linfcore(P_i,\eps,\delta/|P_1|)$\label{ssi}\\
            \For {every $p\in S_i$}
            {
            $\displaystyle s(p):=\frac{(1+\eps)}{i }$\label{lineline}
            }
            $P_{i+1}:=P_i\setminus S_i$\label{remov}\\
            $i:=i+1$\\
            }
            \Return $s$
        \end{algorithm}

        \renewcommand{\tt}{time}
        \newcommand{\lincore}{\text{$\ell_{\infty}$-\textsc{Coreset}}}
        The following variant is a small simplification, improvement and generalization of ~\cite[Lemma 49]{feldman2013turning}
         which is in turn a variant of~\cite[Lemma 3.1]{VX12-soda}.

        The following result shows how a coreset scheme for $\ell_{\infty}$ coresets
        can be used to bound sensitivity.
        The total sensitivity depends on the size of the $\ell_{\infty}$ coreset,
        which in turn determines the size of the desired $\ell_1$ coreset.
        \begin{lemma}[generalization of \cite{VX12-soda,feldman2013turning}]\label{sensitivitylemmasquareddistances2}
        Let $P$ be a set of size $n=|P|$, and $(\lincore,\size,\time)$ be a coreset scheme for $(P,Y,\cost,\norm{\cdot}_\infty)$.
        Let $\eps,\delta\in(0,1/2)$, and let $s:P\to(0,\infty)$ be the output of a call to $\red(P,\eps,\delta,\lincore)$; See Algorithm~\ref{algred}.
        Then, with probability at least $1-\delta$, $s$ is a sensitivity bound for $(P,Y,\cost)$ whose total sensitivity is
        \begin{equation}\label{sb}
        \sum_{p\in P} s(p)\in  \size(n,n,\eps,\delta/n) (1+\eps) O(\log n).
        \end{equation}
        Moreover, the function $s$ can be computed in $O(n)\cdot \time(n,n,\eps,\delta/n)$ time.
        \end{lemma}
        \begin{proof}
        \paragraph{Probability of failure. }
        For $i\in[n]$, the event that $S_i$ is an $\eps$-coreset for $(P_i,Y,\cost,\norm{\cdot}_{\infty})$ during the execution of Line~\ref{ssi},
        occurs with probability at least $1-\delta/n$. For the rest of the proof,
        suppose that this event indeed occured for every $i\in [n]$, which happens with probability at least $1-\delta$, by the union bound.

        \paragraph{Correctness. }
        We first bound the total sensitivity and then the computation time of $s$. Algorithm~\ref{algred} implements the algorithm that is described in the following proof.\\

        Let $i=1$, and $P_1=P$.
        By its construction, $|S_i|\leq \size(n,n,\eps,\delta)$, and for every $y\in Y$
        \[
        \max_{p\in P_i}\cost(p,y)\leq (1+\eps)\max_{p\in S_i}\cost(p,y).
        \]
        Recursively define $P_{i+1}=P_i\setminus S_i$ for every non-empty set $P_i$. Hence, $|P_{i+1}| < |P_i|$, and thus $P_{\ell}=\emptyset$ for
        \begin{equation}\label{ellell}
        \ell\leq |P_1|=n.
        \end{equation}

        Let $p\in P$, and $v=v(p)\in[\ell]$ such that $p\in S_v$.
        Let $y\in Y$ such that $\cost(p,y)>0$.
        Finally, let $j\in [v]$ and $s_j\in\argmax_{s\in S_j}\cost(s,y)$ denote the "farthest point" in $S_j$ from $y$. Since $p\in P_j$,
        \begin{equation}\label{fpy}
        \cost(p,y)\leq \max_{q\in P_j}\cost(q,y)\leq (1+\eps)\max_{s\in S_j}\cost(s,y)=(1+\eps)\cost(s_j,y),
        \end{equation}
        where the second inequality holds by the definition of $S_j$. We can now bound $\cost(p,y)$ by
        \begin{align}
        \label{fff}\cost(p,y)&\leq  (1+\eps) \min_{j\in[v]}\cost(s_j,y)\\
        &\label{ff5}\leq (1+\eps)\cdot \frac{\sum_{j=1}^v \cost(s_j,y)}{v}\\
        \label{ff2}&\leq \frac{(1+\eps)}{v}\sum_{p\in P}\cost(p,y),
        \end{align}
        where~\eqref{fff} holds by~\eqref{fpy},~\eqref{ff5} holds since the minimum cannot be larger than the average,
        and~\eqref{ff2} holds since $\br{s_1,\cdots,s_v}\subseteq P$.

        For every $p\in P$, let
        \begin{equation}\label{ssp4}
        s(p):=\frac{(1+\eps)}{v(p)}.
        \end{equation}
        Then $s(p)$ is a sensitivity bound by~\eqref{fff}, as desired by the lemma and Definition~\ref{sensitivity}.
        Summing~\eqref{ssp4} over $p\in P$ bounds the total sensitivity
        \[
        \begin{split}
        \sum_{p\in P}s(p)
        &\leq (1+\eps)\sum_{p\in P}\frac{1}{v(p)  }\\
        &= (1+\eps)\sum_{j=1}^\ell \frac{|S_j|}{j}
        \in (1+\eps)\size(n,n,\eps,\delta/n) O(\log n),
        \end{split}
        \]
        where the equality holds since $|S_j|$ points were removed from $P_j$ during the $j$th iteration of the algorithm,
        and each $p\in S_j$ was labeled $v(p)=j$, for every $j\in[\ell]$.
        The last deviation holds by the definition of $\size$~\eqref{ellell}
        and the fact that $\sum_{i=1}^n 1/n=O(\log n)$ is an harmonic sequence for every integer $n\geq 1$;
        see Lemma~\ref{toh}.

        \paragraph{Running time. }
        For every $i\in[\ell]$, computing $S_i$ (whose cardinality is at most $n$) for $P_i$ takes at most $\time(n,n,\eps,\delta/n)$ time.
        By~\eqref{ellell}, computing all the $n$ sets takes
        $n\cdot \time(n,n,\eps,\delta/n)$ time.
        Removing $S_i$ from $B_i$ (e.g. using linked lists),
        as well as computing the values of $s$, takes $O(n)$ time, so the dominated time is $n\cdot \time(n,n,\eps,\delta/n)$.
        \end{proof}

    \section{Weighted Input}\label{sec Weighted Input}
        In this section we generalize the result of the previous section to non-weighted input. This is a generalization of the idea that was suggested
        in~\cite{feldman2013turning} with little better bounds.

        \begin{algorithm}
            \caption{$\redd(Q,\eps,\delta,\lincore)$}
            \label{Alg_wsensitivity}
            {\begin{tabbing}
            \textbf{Input: }\quad\quad\=A weighted set $Q=(P,w)$ of points in $\REAL^d$ \\
            \>  an approximation error $\eps>0$, probability of failure $\delta\in(0,1)$.\\
            \textbf{Required: }\>A coreset scheme $\lincore$ for $(P,Y,\cost,\norm{\cdot}_\infty)$.\\
            \textbf{Output:} \>A sensitivity bound $s:P\to(0,\infty)$ that satisfies Lemma~\ref{sensitivitylemmasquareddistances}.
            \end{tabbing}\vspace{-0.3cm}}
            $w_{\min}=\min_{p\in P}w(p)$\\
            \For {every $p\in P$}{
            $h(p):=\left\lceil \frac{w(p)}{\eps w_{\min}}\right\rceil$\\
            $s(p):=0$\\
            }
            $P_1:=P;$\quad $i:=1$\\
            \While{$|P_i|\geq 1$}
            {
            $S_i:=\lincore(P_i, \eps,\delta/|P_1|)$\label{corconst}\\
            Set $q_i\in \argmin_{p\in S_i} h(p)$\\
            \For {every $p\in S_i$}{
            $h(p):=h(p)-h(q_i)$\\
            $m:=i+h(q_i)-1$\\
            $\displaystyle s(p):=s(p)+(1+\eps)^2\left(\ln \left(\frac{m}{i-1}\right)+\frac{1}{2m}-\frac{1}{2(i-1)}+\frac{1}{(i-1)^2} \right)$\label{upup}\\\tcp{$\sim s(p)+(1+\eps)^2\sum_{j=i}^{m}\frac{1}{j}$; see Corollary~\ref{coreuler}}   
            }
            $P_{m+1}:=P_i\setminus \br{q_i}$\\
            $i:=m+1$
            }
            \Return $s$
        \end{algorithm}

        The following constant for approximating harmonic sequences can be approximated very efficiently, in exponential convergence rate. However, in the next corollary we use it only for the analysis, since we use it to compute the difference between two harmonic sequences.

        \begin{theorem}[Euler–Mascheroni Constant~\cite{mortici2010improved}\label{eulerr}]
Let $n\geq1$ be an integer. Then there is a constant $\gamma$ (independent of $n$) such that
\[
 0< \ln n+\gamma+\frac{1}{2n}-\sum_{i=1}^n \frac{1}{i}\leq \frac{1}{n^2}.
\]
\end{theorem}

    \begin{corollary}\label{coreuler}
    For every pair of integers $m\geq i\geq 2$,
    \begin{equation}\label{toh}
    -\frac{1}{(i-1)^2}<  \ln \left(\frac{m}{i-1}\right)+\frac{1}{2m}-\frac{1}{2(i-1)} -\sum_{j=i}^{m}\frac{1}{j} \leq  \frac{1}{m^2}
    \end{equation}
    \end{corollary}
    \begin{proof}
    For the left hand side of~\eqref{toh}, we have by Theorem~\ref{eulerr}
    \[
\begin{split}
    \sum_{j=i}^{m}\frac{1}{j}
    &=\sum_{j=1}^{m}\frac{1}{j}-\sum_{j=1}^{i-1}\frac{1}{j}
    < \ln m+\gamma+\frac{1}{2m}
    -\left(\ln (i-1)+\gamma+\frac{1}{2(i-1)}-\frac{1}{(i-1)^2} \right)\\
    &=\ln \left(\frac{m}{i-1}\right)+\frac{1}{2m}-\frac{1}{2(i-1)} +\frac{1}{(i-1)^2}.
\end{split}
    \]
    Similarly,
    \[
\begin{split}
    \sum_{j=i}^{m}\frac{1}{j}
    &=\sum_{j=1}^{m}\frac{1}{j}-\sum_{j=1}^{i-1}\frac{1}{j}
    \geq  \ln m+\gamma+\frac{1}{2m} -\frac{1}{m^2}
    -\left(\ln (i-1)+\gamma+\frac{1}{2(i-1)}\right)\\
    &=\ln \left(\frac{m}{i-1}\right)+\frac{1}{2m}-\frac{1}{2(i-1)}-\frac{1}{m^2}.
\end{split}
    \]
    \end{proof}

        Recall that $\overline{w}$ was defined in~\eqref{eq801} by
        \[
        \overline{w}(C')=\frac{\sum_{p\in D}w(p)}{\min_{q}w(q)}.
        \]
        \begin{theorem}[\cite{VX12-soda,feldman2013turning}]\label{sensitivitylemmasquareddistances}
        Let $(P,w)$ be a positively weighed set of size $n=|P|$,
        and $(\lincore,\size,\time)$ be a coreset scheme for $(P,Y,\cost,\norm{\cdot}_\infty)$. Let  $\eps,\delta >0$, and
        $s:P\to(0,\infty)$ be the output of a call to $\redd((P,w),\eps,\delta,\lincore)$;
        See Algorithm~\ref{Alg_wsensitivity}.
        Then, with probability at least $1-\delta$, $s$ is a sensitivity bound for $((P,w),Y,\cost)$, its total sensitivity is
        \[
        \sum_{p\in P} s(p)\in  \size(n,n,\eps,\delta/n)O\left(\log \frac{\overline{w}(P)}{\eps}\right),
        \]
        and the function $s$ can be computed in $O(n)\cdot \time(n,n,\eps,\delta/n)$ time.
        \end{theorem}
        \begin{proof}
        \paragraph{The probability }
        that the construction of the coreset in Line~\ref{corconst} would succeed during all the $n$ iterations of the algorithm is at least $1-\delta$, by the union bound.

        \paragraph{Sensitivity bound: }Let $p\in P$ and consider the values of $h$, $s$ and $w_{\min}$ from Algorithm~\ref{Alg_wsensitivity}.
        We have \[
        x\leq y\cdot \left\lceil \frac{x}{y}\right\rceil\leq y\left(\frac{x}{y}+1\right)=x+y
         \]
         for every $x,y\geq 0$. Substituting $x=w(p)$,  $y=\eps w_{\min}$ and $h(p)=\lceil x/y\rceil$, yields
        \[
        w(p)\leq \eps w_{\min}h(p)= \eps w_{\min}\lceil x/y\rceil
        \leq w(p)+\eps w_{\min}\leq (1+\eps)w(p).
        \]
        Hence, by letting $s^*_h(p)$ denote the sensitivity of $((P,h),Y,\cost)$,we obtain
        \begin{equation}\label{hha}
        \begin{split}
        s^*(p)&:=\sup_{y\in Y:\cost(p,y)>0}\frac{w(p)\cost(p,y)}{\sum_{q\in P}w(q)\cost(q,y)}\\
        &\leq \sup_{y\in Y: \cost(p,y)>0}\frac{\eps w_{\min}h(p)\cost(p,y)}{\sum_{q\in P}\eps w_{\min}h(q)\cost(q,y)/(1+\eps)}
        = (1+\eps)s^*_h(p).
        \end{split}
        \end{equation}
        It is left to bound $s^*_h(p)$.

        Let $P'$ denote the (unweighted) multi-set where each point $p\in P$ is duplicated $h(p)$ times. 
        Let $s':P'\to[0,\infty)$ denote the output of a call to $\red(P',\eps,\delta,\lincore)$.
        By Lemma~\ref{sensitivitylemmasquareddistances2}, for a single copy of a point $p$ in $P'$ we have, with probability at least $1-\delta$,
        \[
        s'(p)\geq \sup_{y\in Y, \cost(p,y)>0}\frac{\cost(p,y)}{\sum_{q\in P}h(q)\cost(q,y)}
        \]
        so for all its $h(p)$ copies we have
        \begin{equation}\label{hpsps}
        h(p)s'(p)\geq \sup_{y\in Y, \cost(p,y)>0}\frac{h(p)\cost(p,y)}{\sum_{q\in P}h(q)\cost(q,y)}=s^*_h(p).
        \end{equation}
        That is, $s^*_h(p)\leq h(p)s'(p)$. Next, we bound $h(p)s'(p)$ by $s(p)$.

        Note that the number of copies of a point $p$ in $P'$ has no effect on the computation of the coreset $S_1$ during the first iteration of the call to
        $\red(P',\eps,\delta,\lincore)$, so $S_1$ may be computed on the \emph{unweighted set} of the $n$ distinct points in $P'$.
        Moreover, this number $n$ of distinct points will remain the same after the first iteration,
        as well as the following coresets $S_2, S_3,\cdots$,
        until all the copies of some point $q_1\in S_1$ will be removed in Line~\ref{remov} of $\red$.
        This point $q_1$ is the point with the smallest number of duplicates (weights) in $S_1$.
        During these $h(q_1)$ iterations, the value $s'(p)$ of (one of the copies of) each $p\in S_1$ is increased in Line~\ref{lineline} by,
        \[
        \frac{1+\eps}{1}+\cdots+\frac{1+\eps}{h(q_1)}=(1+\eps)\sum_{j=1}^{h(q_1)}\frac{1}{j}
        \leq (1+\eps)\left(\ln \left(\frac{h(q_1)}{1}\right)+\frac{1}{2h(q_1)}\right)),
        \]
        where the last inequality is by substituting $i=1$ and $m=h(q_1)$ in the left hand side of Corollary~\ref{coreuler}.
        This is indeed the value that is added to $s(p)$ in Line~\ref{upup} of Algorithm~\ref{Alg_wsensitivity}.

        Similarly, during the $i$th iteration $q_i\in S_i$ is the point with the smallest number of remaining copies in $P'=P_i$. In Algorithm~\ref{Alg_wsensitivity}, $q_i$ is removed in its $j$th iteration for some $j\in[n]$.
       Let $i_j$ and $m_j$ respectively denote the value of $i$ and $m$ during the execution of the $j$th iteration. Hence, $i_1=1$, $i_{j+1}=i_{j}+h(q_{i_{j}})$ and $m_j=i_{j}+h(q_{i_j})-1=i_{j+1}-1$, for every $j\in[n]$. We obtain that $S_i$ is the same for every iteration $i\in [i_{j}, m_j]$ in Algorithm~\ref{algred}. During these $h(q_{i_{j-1}+1})$ iterations, until $q_{i_{j}}$ is removed from $P'=P_{i_j}$, the value $s'(p)$ of every $p\in S_i$ was increased by
        \[
        h(p)s'(p)=\frac{1+\eps}{i_{j}}+\cdots+\frac{1+\eps}{m_j}
        =(1+\eps)\sum_{k=i_{j}}^{m_j}\frac{1}{k}
        < (1+\eps)\left(\ln \left(\frac{m_j}{i_j-1}\right)+\frac{1}{2m_j}-\frac{1}{2(i_j-1)}+\frac{1}{(i_j-1)^2} \right),
        \]
        where the inequality is by substituting $i=i_j$ and $m=m_j$ in the left hand side of~\eqref{toh}.

        The right hand side of the last inequality multiplied by $(1+\eps)$ is the update of $s(p)$ in Line~\ref{upup} of the $j$th iteration in Algorithm~\ref{Alg_wsensitivity} that imitates the updates of $s'(p)$ during iterations $i_{j}$ till $i_{j+1}$ of Algorithm~\ref{algred}. Hence,
        \begin{equation}\label{lastlast}
        (1+\eps)h(p)s'(p)\leq s(p).
        \end{equation}
        This proves the desired sensitivity bound as
        \[
        s(p)\geq (1+\eps)h(p)s'(p)\geq (1+\eps)s^*_h(p)
        \geq s^*(p)
        =\sup_{y\in Y:\cost(p,y)>0}\frac{w(p)\cost(p,y)}{\sum_{q\in P}w(q)\cost(q,y)},
        \]
        where the first inequality is by~\eqref{lastlast}, the second is by~\eqref{hpsps}, and the third is by~\eqref{hha}.

        \paragraph{The total sensitivity } of $s$ is bounded using the fact that $\eps\in O(1)$, $m_j=i_{j+1}-1$ and
        \begin{align}
        \label{aay1}&\sum_{j=1}^{n}\left(
        \ln \left(\frac{i_{j+1}-1}{i_j-1}\right)
        +\frac{1}{2(i_{j+1}-1)}-\frac{1}{2(i_j-1)}
        +\frac{1}{(i_j-1)^2}
        \right)
        \leq \sum_{j=1}^{n}\frac{1}{(i_{j+1}-1)^2}+\sum_{j=1}^{n}\sum_{k=i_j}^{i_{j+1}-1}\frac{1}{k}\\
        &\leq \sum_{j=1}^{\infty}\frac{1}{j^2}+\sum_{k=1}^{i_{n+1}-1}\frac{1}{k}
         \in O(\log (i_{n+1}))=O\left(\sum_{p\in P} h(p)\right)
        \label{ssf}=O\left(\log \frac{\overline{w}(P)}{\eps}\right).
        \end{align}
        where~\eqref{aay1} is by substituting $i=i_j$ and $m=i_{j+1}-1$ in the right hand side of~\eqref{toh}, and~\eqref{ssf} holds since $2=\sum_{j=1}^{\infty}1/j^2$, and $\sum_{k=1}^m\in O(\lg m)$ by~\eqref{toh}.
        Summing the accumulated sensitivities over every $j\in[n]$ iteration, each over $|S_{i_j}|$ points yields
        \[
        \begin{split}
        \sum_{p\in P}s(p)&
        \leq  (1+\eps)^2\sum_{j=1}^n |S_{i_j}|(1+\eps)^2 \left(\ln \left(\frac{m_j}{i_j-1}\right)+\frac{1}{2m_j}-\frac{1}{2(i_j-1)}+\frac{1}{(i_j-1)^2} \right)\\
        &\in  (1+\eps)^2\size(n,n,\eps,\delta/n)O\left(\log \frac{\overline{w}(P)}{\eps}\right),
        \end{split}
        \]
        where the last derivation is by~\eqref{ssf} and the definition of $S_{i_j}$ in Line~\ref{corconst} of Algorithm~\ref{Alg_wsensitivity}.

        \paragraph{The running time }follows from the fact that in the $j$th "for" iteration, the point $q_{i_j}$ is removed from $P$, so there are $n$ iterations.
        The dominated time in each of the $n$ iterations is computing the coreset $S_i$ in $\time(n,n,\eps,\delta/n)$ time.
        \end{proof}

        By combining Theorem~\ref{sensitivitylemmasquareddistances} and Theorem~\ref{supsampe} we obtain the following corollary
        which shows how to compute coresets with $\norm{\cdot}_1$ loss
        based on coresets for $\norm{\cdot}_{\infty}$ loss on weighted data.

        \begin{corollary}\label{55}
        Let
        \begin{itemize}
        \item $(P,w,Y,\cost,\norm{\cdot}_1)$ be a query space.
        \item $f:P\times Y\to[0,\infty)$ such that for every $p\in P$ and $y\in Y$,
        \begin{equation}\label{ffff}
        \ff(p,y)
        =\begin{cases}
        \frac{\cost(p,y)}{\sum_{p\in P}w(p)\cost(p,q)} & \cost(p,q)>0\\
        0 & \cost(p,y)=0,
        \end{cases}
        \end{equation}

        \item $s,\time$ and $\size$ be defined as in Theorem~\ref{sensitivitylemmasquareddistances}.
        \item  $\eps,\delta\in (0,1)$, and $m$ be defined as in Theorem~\ref{supsampe}.
        \item $(C,u)$ be the output of a call to $\impalg(P,w,s,m)$; see Algorithm~\ref{Alg_Coreset}.
        \end{itemize}
        Then $(i)$--$(v)$ hold as follows.
        \begin{enumerate}
        \renewcommand{\labelenumi}{\theenumi}
        \renewcommand{\theenumi}{(\roman{enumi})}
         \item With probability at least $1-\delta$, $(C,u)$ is an $\eps$-coreset for $((P,w),Y,\cost,\norm{\cdot}_1)$.
        \item  $C\subseteq P$ and $|C|\in O(m)$.
         \item  $C$ can be computed in $O(n) \cdot \time(n,n,\eps,\delta/n))$ time where $n=|P|$.
        \item $u(p)\in [w(p), \sum_{q\in P}w(q)/m]$ for every $p\in C$, and
        \item $\sum_{q\in C}u(q)=\sum_{q\in P}w(q)$.  
        \end{enumerate}
        \end{corollary}

\chapter{From SMM-Coreset\\ to Projective Clustering\label{sec:proj}}
    In the previous chapters we proved that in order to compute coreset
    for $L(P,\theta)$ it suffices to compute coreset for $\cost(\br{p'},y)$.
    To compute the latter coreset,
    we reduce the problem to computing coresets for the query space
    $(P',\CC,\ff,\loss)$ of projective clustering as explained in Section~\ref{sec:Introduction}.

    For a set $P\subseteq\REAL^d$ of points, and a union $S=S_1\cup\cdots\cup S_k$ of $k$ subspaces in $\REAL^d$, we define
    \begin{equation}\label{defdist}
            \dist_{\infty}(P,S)=\max_{p\in P} \dist(p,S),
    \end{equation}
    to be the distance of the farthest point in $P$ from $S$, and for every $k$-SMM $y$
    \[\cost_{\infty}(P,y)=\max_{p\in P} \cost(p,y),\]
    be the point in $P$ with the maximum $\cost$ to $y$.

    Recall that $S(y)$ and $\cost$ were defined in Section~\ref{sec:three}.
    The following lemma proves that $cost_{\infty}$ for $k$-GMM
    is an upper bound for $W\dist_{\infty}$ for the corresponding $k$-subspaces.
    It will be used in our main result later.

    \begin{lemma}[$\cost_{\infty}$ upper bounds $\dist_{\infty}$]\label{lemhelp}
        For every $k$-SMM $y$ and a finite set $P$ of points in $\REAL^d$, we have
        \[
            \cost_{\infty}(P,y)\geq W\dist_{\infty}^2(P,S(y)).
        \]
    \end{lemma}
    \begin{proof}
        Let $y=(W,\omega_1,\cdots,\omega_k,S_1,\cdots,S_k)$ be a $k$-SMM. We then have
        \begin{align}
            \cost_{\infty}(P,y)
            \label{xa2}&=\max_{p\in P}\cost(p,y)\\
            \label{xa3}&=\max_{p\in P}-\ln \sum_{i=1}^k \omega_i \exp(-W\dist^2(p,S_i))\\
            \label{xa4}&\geq \max_{p\in P}-\ln \sum_{i=1}^k \omega_i \exp(-W\dist^2(p,S(y)))\\
            \label{xa7}&= W\dist_{\infty}^2(P,S(y)),
        \end{align}
        where~\eqref{xa2} and~\eqref{xa3} hold by definition,~\eqref{xa4} holds since $\dist(p,S_i)\geq \dist(p,S(y))$.
    \end{proof}

    The following lemma is the heart of our main technical result.
    Informally, it states that a coreset for projective clustering,
    i.e., distance to the farthest input point from a set of $k$ subspaces,
    can be used to compute a coreset for the cost function above,
    which is not a distance function at all.
    The fact that a $1/3$-coreset for $\norm{\cdot}_\infty$
    suffices to get $\eps$-coreset is crucial for getting smaller coresets
    in special cases as explained in Section~\ref{sec From Sensitivity to ell infty}.

    \begin{lemma}[$\cost_{\infty}$ to $\dist_{\infty}$ \label{hyp}] \label{lemma cost_inf to dist_inf}
        Let $P$ be a finite set of points in $\REAL^d$, $k\geq1$ be an integer, and $\ta\in(0,\sqrt{2k})$.
        Let $H_{k,d}$ be the union over every set of $k$ subspaces in $\REAL^d$, and $Y_\ta$
        be the union over every $k$-SMM  $y=(W,\omega_1,\cdots,\omega_k,S_1,\cdots,S_k)$ such that
        \begin{equation}\label{emin}
            W\dist^2_{\infty}(P,S(y))\geq \ta.
        \end{equation}
        Then a $(1/3)$-coreset for $(P,H_{k,d},\dist,\norm{\cdot}_\infty)$ is a $O(k/\ta)$-coreset for  $(P,Y_\ta,\cost,\norm{\cdot}_\infty)$.
    \end{lemma}
    \begin{proof}
    Let $$y=(W,\omega_1,\ldots,\omega_k,S_1,\ldots,S_k),$$ be a $k$-SMM such that $y$ satisfies~\eqref{emin}. Let $g= \sqrt{1+2k/\ta}$
    and suppose that $C$ is a $(1/3)$-coreset of $\dist_{\infty}(P,\cdot)$ for every set of $k$ subspaces in $H_{k,d}$ y in $\REAL^d$. In particular,
    \begin{equation}\label{aa2}
        \dist_{\infty}(P,S(y))\leq (1+1/3)\dist_{\infty}(C,S(y))\leq  g\dist_{\infty}(C,S(y)),
    \end{equation}
    where the last inequality holds since $g\geq \sqrt{2}\geq 1+1/3$ by the assumption $\ta\leq \sqrt{2k}$ of the lemma.

    We will prove that for $x\in \argmax_{p\in P}\cost(p,y)$, $r=\cost_{\infty}(C)$ and for appropriate function
    $h:[k]\to[0,\infty)$ such that $h(k)\leq 2eg^2$, we have
    \begin{equation}\label{aa2246}
        \cost(x,y)\leq rh(k).
    \end{equation}
    By the definition of $x$ and $r$, this would prove the lemma as
    \[
    \cost_{\infty}(P,y)=\cost(x,y)\leq rh(k) \leq 2eg^2\cost_{\infty}(C)=O(k/\ta)\cost_{\infty}(C).
    \]
    \renewcommand{\r}{r}

    Indeed, let $\eps=-\ln(1-e^{-k/r})/ r$, $z=1+k/r^2$, and let $h:[k]\to [0,\infty)$
    such that for every $j\in[k]$ we have
    \begin{equation}\label{hh}
        h(j)=(z+g^2)(1+\eps)^{j-1}.
    \end{equation}
    We first prove that $h(k)\in O(g^2)$ as claimed above, and then prove~\eqref{aa2246}.

    \textbf{Upper bound on $h$: }
    Since $(t-1)/t\leq \ln t\leq t-1$ for every $t>0$,
    \[
        -\ln(1-e^{-k/r})\leq \frac{e^{-k/r}}{1-e^{-k/r}}=\frac{1}{e^{k/r}-1}\leq \frac{r}{k},
    \]
    for substituting $t=1-e^{-k/r}$ in the first inequality, and $t=e^{k/r}$ in the last inequality.
    Hence,
    \begin{equation}\label{epss}
        \eps=\frac{-\ln(1-e^{-k/r})}{ r}\leq \frac{1}{k}.
    \end{equation}
    Plugging~\eqref{epss}, $z=1+k/r^2$ and $g=\sqrt{1+2k/\ta}$ in~\eqref{hh} yields
    \begin{equation}\label{hhk}
                h(k)\leq (z+g^2)(1+\eps)^{k}\leq (z+g^2) \left(1+\frac{1}{k}\right)^{k}\leq e(z+g^2).
    \end{equation}

    It is left to prove that $z\leq g^2$, i.e., $1+k/r^2\geq 1+2k/\ta$. Indeed,
            \begin{equation}\label{rr}
        r=\cost_{\infty}(C,y)\geq W\dist_{\infty}^2(C,S(y))
        \geq (1+1/3)\big(W\dist_{\infty}^2(P,S(y))\big)\geq \frac{\ta}{2},
    \end{equation}
    where the first inequality is by substituting $P=C$ in Lemma~\ref{lemhelp},
    the second inequality is by~\eqref{aa2}, the third is by~\eqref{emin}. Plugging~\eqref{rr} in~\eqref{hhk} yields $h(k)\leq 2eg^2$ as desired.

     The proof of~\eqref{aa2246} is by induction on $k$.
     \\
    \medskip
    \noindent\textbf{Proof of~\eqref{aa2246} for the base case $k=1$.}
    Let $a^*\in\arg\max_{a\in P}\cost(a,y)$.
    In this case, $y=(W,1,s)\in Y_1$, and~\eqref{aa2246} follows as
    \begin{equation}\label{kkt}
        \begin{split}
            \cost_{\infty}(x,y)
            &=-\ln \exp(-W\dist^2(x,S(y)))
            =W\dist_{\infty}^2(P,S(y))\\
            &\leq W\cdot (g\dist_{\infty}(C, S(y)))^2
            \leq g^2r\leq (z+g^2)r= h(1)\cost_{\infty}(C),
        \end{split}
    \end{equation}
    where the first inequality follows from~\eqref{aa2}, and the second inequality by the second inequality of~\eqref{rr}.
    This proves \eqref{aa2246}, and in turn the lemmas (as explained above) for the case $k=1$.
    \medskip
    \textbf{Proof for the case $k\geq 2$.}
    Without loss of generality, assume that $S_k$ is a closest subspace to $x$ in $S(y)$, i.e.,
    \begin{equation}\label{wlog3}
        \dist(x,S_k)=\dist(x,S(y)).
    \end{equation}
    Inductively assume that the lemma holds for $k'\in [k-1]$. More precisely,
    if $C'$ is a $(1/3)$-coreset of $(P',H_{k-1,d},\dist,\norm{\cdot}_\infty)$
    then it is an $(h(k')-1)$-coreset of $(P,Y_\ta,\cost,\norm{\cdot}_\infty)$
    for every $k'$-GMM $y'$ that satisfies~\eqref{emin}.
    The base case $k'=1$ follows from~\eqref{kkt}.

    The proof is by case analysis: \textbf{(i)} $\omega_k\in [e^{-z\r},1]$, and \textbf{(ii)} $\omega_k\in (0,e^{-z\r})$.

    \textbf{Case \textbf{(i)}: $\omega_k\in [e^{-z\r},1]$.} We have
    \begin{equation}\label{aa5}
        \begin{split}
            \r
            &=\max_{a\in C}\bigg{(} -\ln \sum_{i=1}^k \omega_i \exp(-W\dist^2(a,S_i)) \bigg{)}\\
            &\geq \max_{a\in C}\bigg{(} -\ln \sum_{i=1}^k \omega_i \exp(-W\dist^2(a,S(y))) \bigg{)}\\
            &= W\max_{a\in C} \dist^2(a,S(y))=W\dist_{\infty}^2(C,S(y)).
        \end{split}
    \end{equation}
    Hence,
    \begin{equation}\label{aa7}
        \begin{split}
                &\dist(x,S_k)=\dist(x,S(y))\leq \dist_{\infty}(P,S(y))\leq g\dist_{\infty}(C,S(y))
                \leq g\sqrt{\r/W},
        \end{split}
    \end{equation}
    where the derivations are, respectively, by~\eqref{wlog3},~\eqref{defdist},~\eqref{aa2}, and~\eqref{aa5}.
    Therefore,
    \begin{align}
        \cost(x,y)     &=-\ln \bigg{(}\sum_{i=1}^k \omega_i \exp(-W\dist^2(x,S_i))\bigg{)}\nonumber\\
        &\leq-\ln\bigg{(} \omega_k \exp(-W\dist^2(x,S_k))\bigg{)}\label{z1}\\
        &\leq-\ln \left(\omega_k \exp(-g^2\r)\right)\label{z2}\\
        &\leq-\ln \left(e^{-z r} \exp(-g^2\r)\right)\label{z3}\\
        &=(z+g^2)\r\leq rh(1)\leq r h(k),\label{z4}
    \end{align}
    where~\eqref{z1} holds since it is a single item from the previous sum, and~\eqref{z2} holds by~\eqref{aa7},~\eqref{z3}
    holds by the assumption of Case (i), and~\eqref{z4} holds by~\eqref{hh} and the fact that $h(\cdot)$ is a monotonic function.
    This proves~\eqref{aa2246} for Case (i).

    \textbf{Case \textbf{(ii)}: $\omega_k\in (0,e^{-z\r})$.}
    We prove that in this case $\omega_k$ can be replaced by $\omega_k=0$ so that $y$ can be replaced by $y'\in Y_{\xi}$. 
    \renewcommand{\c}{c}
    \renewcommand{\t}{t}

    We have $\omega_k< e^{-z\r} \leq 1$, where the first inequality is by the assumption of Case (ii), and the second is since $z,r>0$.
    Hence, $\phi=1/(1-\omega_k)$ is well defined. Since $\sum_{i=1}^k\omega_i=1$,
    \[
        \sum_{i=1}^{k-1} \phi \omega_i=\frac{1}{(1-\omega_k)}\sum_{i=1}^{k-1} \omega_i=1,
    \]
    so the tuple $y'=(W,\phi\omega_1,\ldots,\phi\omega_{k-1},S_1,\ldots,S_{k-1})$ is a $(k-1)$-SMM.

    Let $a'\in \argmax_{a\in C}\cost(a,y')$. Hence,
    \[
        \begin{split}
            e^{- r}&\leq e^{-\cost(a',y)}= \sum_{i=1}^k \omega_i \exp(-W\dist^2(a',S_i))\\
            &\leq \omega_k+\sum_{i=1}^{k-1} \omega_i \exp(-W\dist^2(a',S_i)),
        \end{split}
    \]
    where the first inequality is by the definition of $r=\cost_{\infty}(C,y)$, and the second inequality is since $W\dist^2(a',S_k)\geq 0$.
    Subtracting $\omega_k$ from both sides and multiplying by $\phi$, yields
    \[
        \begin{split}
            \phi(e^{-r}-\omega_k)
            &\leq \sum_{i=1}^{k-1} \phi\omega_i \exp(-W\dist^2(a',S_i))\\
            &=e^{-\cost(a',y')}=e^{-\cost_{\infty}(C,y')}.
        \end{split}
    \]
    Taking the $\ln$ of both sides and multiplying by $(-1)$ yields %
    \begin{equation}\label{lastlast2}
        \begin{split} 
                 \cost_{\infty}(C,y')
                 &\leq -\ln(\phi(e^{-r}-\omega_k))\leq -\ln(e^{-r}-e^{-zr})\\
                 &\leq -\ln(e^{-r}-e^{-r+\ln (1-\exp\br{-\eps r})})
                 = -\ln(e^{-r}- e^{-r} (1-\exp\br{-\eps r}))\\
                 &=-\ln(\exp(-r-\eps r))
                 =r(1+\eps),
        \end{split}
    \end{equation}
    where the second inequality holds by the assumption of Case (ii), and $\phi>1$, and the third one since
    \[
    z=1+\frac{k}{r^2}=1-\frac{\ln (1-e^{-\ln(1/(1-e^{-k/r}))})}{r}
    =1-\frac{\ln (1-e^{-\eps r})}{r},
    \]
    where the last equality is by the definition of $\eps$ in~\eqref{epss}.
    Since $C$ is a $(1/3)$-coreset of $(P,H_{k,d},\dist,\norm{\cdot}_\infty)$,
    it is also such a coreset for $k-1$ subspaces (e.g. by duplicating one of the $k$ subspaces in the query).
    By this and the inductive assumption that the lemma holds for $k'=k-1$, $C$ is also an $(h(k-1)-1)$-coreset of $P$ for every $k-1$ SMM $y''$ that satisfies~\eqref{emin}.
    In particular, for $y''=y'$
    \begin{equation}\label{XY2}
            \cost_{\infty}(P,y')\leq h(k-1)\cost_{\infty}(C,y').
    \end{equation}
    Therefore,
    \begin{equation}\label{wa}
        \begin{split}
            -\ln \sum_{i=1}^{k-1} \phi\omega_i \exp(-W\dist^2(x,S_i))
            &=\cost(x,y')
            \leq \cost_{\infty}(P,y')\\
            &\leq h(k-1) \cost_{\infty}(C,y')\leq h(k-1)(1+\eps)r,
        \end{split}
    \end{equation}
    where the first inequality is since $x\in P$ and second inequality is by~\eqref{XY2}.

    For every distribution vector $(\omega_1,\cdots,\omega_k)$, and every vector $z=(z_1,\cdots,z_k)\in[0,\infty)^k$ whose maximum is $\norm{z}_\infty=z_k$, we have
    \begin{align}
            \sum_{i=1}^k\omega_i z_i
            &= \sum_{i=1}^{k-1}\phi \omega_i z_i(1-\omega_k)+\omega_k z_k
            \nonumber=\sum_{i=1}^{k-1}\phi \omega_i z_i+\omega_k(z_k-\sum_{i=1}^{k-1}\phi \omega_i z_i)\\
            \label{zz4}&\geq \sum_{i=1}^{k-1}\phi \omega_i z_i+\omega_k(z_k-\norm{z}_\infty)\\
            \label{zz5}&= \sum_{i=1}^{k-1}\phi \omega_i z_i,
    \end{align}
    where the first equality holds since $\phi=1/(1-\omega_k)$,~\eqref{zz4} holds since
    $\phi(\omega_1,\cdots,\omega_{k-1})$ is a distribution vector, and~\eqref{zz5} is by the definition of $z_k$. Taking the $-\ln$ of both sides in~\eqref{zz5} yields
    \begin{equation}\label{zz2}
        -\ln \sum_{i=1}^k\omega_i z_i\leq -\ln \sum_{i=1}^{k-1}\phi \omega_i z_i.
    \end{equation}
    By substituting $z_i=\exp(-W\dist^2(x,S_i))$ for every $i\in[k]$ in~\eqref{zz2}, we obtain
    \begin{equation}
        \begin{split}
            \cost(x,y)&=-\ln \sum_{i=1}^{k}\omega_i \exp(-W\dist^2(x,S_i))\\
            &\leq -\ln \sum_{i=1}^{k-1}\phi\omega_i \exp(-W\dist^2(x,S_i))
            \leq h(k-1)(1+\eps)r= rh(k),
        \end{split}
    \end{equation}
    where the last inequality is by~\eqref{wa}. This proves~\eqref{aa2246} for Case (ii).
    \end{proof}

    Lemma~\ref{lemhelp}. yields a coreset for $k$-SMMs in $Y_{\xi}$.
    The following Theorem generalize it to $k$-GMMs,
    which aim to minimize maximum (instead of sum) of likelihoods over the input points.

    \begin{theorem}[$\ell_{\infty}$ coreset for log-likelihood\label{hyp2}]
    Let $D$ be a finite set of points in $\REAL^d$, and $k\geq1$ be an integer.
    Let $H_{2d+1,k}$ denote the union over every set of $k$ subspaces in $\REAL^{2d+1}$.
    Let $P=\br{(p\mid 0,\cdots,0)\in\REAL^{2d+1}\mid p\in D}$, and $C$ be a $(1/3)$-coreset for $(P,H_{k,2d+1},\dist,\norm{\cdot}_{\infty})$.

    Then, for every $\ta>0$, $S=\br{p\in D\mid (p\mid 0,\cdots,0)\in C}$ is an $O(k/\ta^2)$-coreset for $(D,\vartheta_k(0),\phi_{\ta},\norm{\cdot}_{\infty})$, and for $(D,\vartheta_k(e^{\ta}/(2\pi)),L,\norm{\cdot}_{\infty})$.
    \end{theorem}
    \begin{proof}
    Let $\theta$ be a $k$-GMM in $\REAL^d$ and $\ta\geq 0$ be a constant.
    By Lemma~\ref{lemma2}, there is a $k$-SMM $y$ in $\REAL^d$ such that
                \begin{equation}\label{emin3}
                W\dist^2_{\infty}(P,S(y))\geq \ta,
            \end{equation}
            and
    \begin{equation}\label{phii}
    \phi(\br{p}, \theta)=\cost(x,y),
    \end{equation}
    for every $p\in D$ and its corresponding point $x=(p^T \mid 0,\cdots,0)^T\in P$.
    By summing~\eqref{phii} over every $p\in D$ and $p\in S$ respectively, we obtain
    \begin{equation}\label{ppi1}
    \phi_{\ta}(D, \theta)=\cost_{\infty}(P,y) \text{ and } \phi_{\ta}(S, \theta)=\cost_{\infty}(C,y).
    \end{equation}

    By Lemma~\ref{hyp}, $C$ is an $O(k/\ta^2)$-coreset of $\cost_{\infty}(P,\cdot)$ for every $k$-SMM $y$ that satisfies~\eqref{emin3}. Hence,
    \[
    \cost_{\infty}(P,y)\in O(k/\ta^2)\cost_{\infty}(C,y).
    \]
    Combining this with~\eqref{ppi1} yields
    \[
    \phi_{\ta}(D, \theta)=
    \cost_{\infty}(P,y)
    \in O(k/\ta^2)\cost_{\infty}(C,y)
    =
    O(k/\ta^2)\phi_{\ta}(S, \theta).
    \]
    Since the last equality holds for every $k$-GMM $\theta$, we conclude that $S$ is the desired coreset for $D$.

    It is also a coreset for $(D,\vartheta_k(e^\ta/(2\pi)),L,\norm{\cdot}_\infty)$ by Observation~\ref{obss}.
    \end{proof}

\chapter{VC-Dimension Bound} \label{sec Bound on the VC-Dimension}
    To compute a coreset using Theorem~\ref{supsampe}
    we need to bound both the sensitivity $s$ and the corresponding dimension of the query space.
    In this section we bound the dimension. This is based on the following general result that bounds the VC-dimension on a set via the time it takes to answer a query.

    \begin{definition}[operations\label{op}]
    Let $P,Y$ be two sets, and $f:P\times Y\to \REAL$ be a function that can be evaluated by an algorithm that gets
    $(p,y)\in P\times Y$ and returns $f(p,y)$ after no more than $z$ of the following operations:
      \begin{enumerate}
      \item the exponential function $\alpha\mapsto e^\alpha$ on real numbers,
      \item the arithmetic operations $+, -, \times, $ and $/$ on real numbers,
      \item jumps conditioned on $>,\geq, <, \leq, =,$ and $\neq$ comparisons of real numbers.
      \end{enumerate}
      If the $z$ operations include no more than $k$ in which the exponential function is evaluated,
      then we say that the function $f$ can be evaluated using \emph{$z$ operations that include $k$ exponential operations.}
    \end{definition}

    The following is a variant of~\cite[Theorem 8.14]{anthony2009neural} for our version of VC-dimension's definition.
    \begin{theorem}[\label{vcl}Variant of~\cite{anthony2009neural}\label{thm:anthony2009neural}]
      Let $h:\REAL^d\times\REAL^m\to\br{0,1}$ be a binary function that can be evaluated using $O(z)$ operations that include $O(k)$ exponential operations;
      see Definition~\ref{op}.
      Then the dimension of $(\REAL^d,\REAL^m,h)$ is $O(m^2k^2 + mkz)$.
    \end{theorem}

    \begin{corollary}\label{pdi}
    Let $P\subseteq\REAL^d$, $Y'\subseteq \REAL^m$ and $f:P\times Y'\to[0,\infty)$.
    If, for every $p\in P$ and $y'\in Y'$ the value $f(p,y')$ can be computed in $O(z)$ arithmetic operations that include $O(k)$ exponential operations,
    then the dimension of $(P,Y',f)$ is $O(m^2k^2 + mkz)$.
    \end{corollary}
    \begin{proof}
    It suffices to prove the desired bound on the dimension of $(\REAL^d,\REAL^m,f)$,
    i.e. assume that $Y'=\REAL^m$, since the VC-dimension,
    as the dimension of $(P,Y,f)$, is monotonic in the cardinality of the set $Y'$ of queries and set $P$.
    Suppose that $S$ is the largest subset $S\subseteq \REAL^d$ such that
    \begin{equation}\label{ee}
    |\br{\range_{S,f}(\c,r) \mid r\in\REAL, \c\in \REAL^m}|= 2^{|S|};
    \end{equation}
    see Definition~\ref{vdim}.
    We need to upper bound the size of $S$.

    Define $h:\R^{m+1}\times \R^{d} \to \br{0,1}$ such that $h(x,p) = 0$ if and only if there is
    $y'\in \REAL^m$ and $r\in\REAL$ such that $x^T=((y')^T \mid r^T)$ and $f(p,y')\leq r$. We then have for every $x=((y')^T \mid r^T)$ in $\REAL^{m+1}$,
    \[
    \begin{split}
    \range_{S,h}(x,0)
    &=\range_{S,h}(((y')^T \mid r^T),0)
    =\br{p\in S \mid h(p,((y')^T \mid r^T))\leq 0 }\\
    &=\br{p\in S \mid h(p,((y')^T \mid r^T))=0 }
    =\br{p\in S \mid f(p,y')\leq r }
    =\range_{S,f}(y',r).
    \end{split}
    \]

    Hence,
    \begin{equation}\label{ss123}
    \begin{split}
    |\br{\range_{S,h}(x,r') \mid r'\in\REAL, x\in \REAL^{m+1}}|
    &\leq |\br{\range_{S,h}(x,0)\mid x\in\REAL^{m+1}}|\\
    &=|\br{\range_{S,f}(y',r) \mid r\in\REAL, y'\in \REAL^m}|
    =2^{|S|},
    \end{split}
    \end{equation}
    where the last equality is by~\eqref{ee}.
    Since $\range_{S,h}(x,r')$ is a subset of $S$, and there are at most $2^{|S|}$ such subsets,~\eqref{ss123} implies
    \begin{equation}\label{before}
    |\br{\range_{S,h}(x,r') \mid r'\in\REAL, x\in \REAL^{m+1}}|=2^{|S|}.
    \end{equation}

    By Theorem~\ref{vcl}, dimension of $(\REAL^d,\REAL^{m+1},h)$ is $d'\in O(m^2k^2 + mkz)$.
    Hence, the size of the largest subset $S\subseteq \REAL^d$ that satisfies~\eqref{before} is $|S|\leq d'$.
    \end{proof}

    The following corollary generalizes Theorem~12 in \cite{lucic2017training} from $\cost$ to $f$, and from semi-spherical Gaussians instead of any Gaussian, using similar approach.
    \begin{corollary}[Generalization of~[Theorem 12]\cite{lucic2017training}\label{thm:pdimgmm}]
    Let $k\geq1$ be an integer and $(P,w)$ be a weighted set such that $P\subseteq\REAL^d$ is finite.
    Let $\ta\geq 0$ and $f:\REAL^d\times \vartheta_k(0)\to[0,\infty)$ such that
    \begin{equation}\label{ffdef}
    f(p,\theta)=\frac{\phi_{\ta}(\bar{p},\theta)}{\phi_{\ta}((P,w),\theta)},
    \end{equation}
    if the denominator is positive, and $f(p,\theta)=0$ otherwise.
    Then the dimension of $(P,\vartheta_k(0),f)$ is $O(d^4k^4)$.
    \end{corollary}
    \begin{proof}
    Let $P=\br{(p\mid 0,\cdots,0)\in\REAL^{2d+1}\mid p^T\in P}$.
    For every $k$-SMM $y=(W,\omega_1,\cdots,\omega_k,S_1,\cdots,S_k)$ in $Y$, and $p\in P$ define
            \[
                \cost(p,y)=-\ln \sum_{i=1}^k \omega_i \exp(-W\dist^2(p,S_i)),
            \]
    and
    \[
    g(p,y)=\frac{\cost(p,y)}{\sum_{q\in P}w(q)\cost(q,y)},
    \]
    if the denominator is positive, and $g(p,y)=0$ otherwise. We first prove that
    \begin{equation}\label{dimm}
    \dim(D,\vartheta_k(0),f)\leq \dim(P,Y,g).
    \end{equation}

    Indeed, let $S$ be the largest subset of $P$, such that
    \begin{equation}\label{ssvc}
    |\br{S\cap \range \mid \range \in \ranges(P,\vartheta_k(0),f)}|=2^{|S|}.
    \end{equation} 
    Hence,
    \begin{equation}\label{2ss}
    \begin{split}
    2^{|S|}
    &=|\br{S\cap \range \mid \range \in \ranges(D,\vartheta_k(0),f)}|\\
    &=|\br{\range_{S,f}(\theta,r)\mid r\geq 0, \theta\in\vartheta_k(0) }|.
    \end{split}
    \end{equation}

    Let $S'=\br{(p^T \mid \mathbf{0})\mid p\in D}$, where $\mathbf{0}:=(0,\cdots,0)\in\REAL^{d+1}$.
    Let $\theta\in \vartheta_k(0)$ be a $k$-GMM. By Lemma~\ref{lemma2}, there is a $k$-SMM $y\in Y$ such that for every
    $p\in \REAL^d$ we have $\phi_{\ta}(\br{p},\theta)=\cost((p^T\mid \mathbf{0}),y)$.
    Hence, for every $p\in S$, there is a corresponding point $(p^T\mid \mathbf{0})\in S'$ such that
    \[
    f(p,\theta)
    =\frac{\phi_{\ta}(\br{p},\theta)}{\sum_{q\in P}w(q)\phi(q,\theta)}
    =\frac{\cost((p^T\mid \mathbf{0}),y)}{\sum_{q\in P}w(q)\cost((q^T\mid 0),y)}
    =g((p^T\mid \mathbf{0}),y).
    \]
    Here, we assumed $\phi_{\ta}(p,\theta)>0$, otherwise $f(p,\theta)=0=g((p^T\mid \mathbf{0}),y)$.

    In particular, for every $r\geq 0$ the set
    \[
    \range_{S,f}(\theta,r)=\br{p\in S\mid f(p,\theta)\leq r }
    \]
     has a corresponding distinct set
    \[
    \range_{S',g}(y,r)=\br{(p^T \mid \mathbf{0})\in S'\mid g((p^T \mid \mathbf{0}),y)\leq r }.\]

    Therefore,
    \begin{equation}\label{2sst}
     2^{|S|}
     = |\br{\range_{S,f}(\theta,r)\mid r\geq 0, \theta\in\vartheta_k(0) }|
    \leq |\br{ \range_{S',g}(y,r)\mid r\geq 0, y\in Y }|,
     \end{equation}
     where the first equality is by~\eqref{2ss}. 
    Since the last expression is a set of subsets from $S'$, its size is upper bounded by $2^{|S'|}=2^{|S|}$. Together with~\eqref{2sst} we obtain,
    \[
    2^{|S'|}= |\br{ \range_{S',g}(y,r)\mid r\geq 0, y\in Y }|,
    \]
    so $|S'|\leq \dim(P,Y,f)$ by the definition of $\dim(P,Y,f)$. The last inequality proves~\eqref{dimm} as
    \[
    2^{\dim(D,\vartheta_k(0),g)}=2^{|S|}= 2^{|S'|}\leq  2^{\dim(P,Y,f)},
    \]
    where the first equality is by the definition of $S$ and $\dim(P,\vartheta_k(0),g)$.

    Next, we bound $\dim(P,Y,g)$. Indeed, the range $\br{p\in S\mid g(p,y)\leq r}$ of $(P,Y,g)$ is the same as the range
    $\br{p\in P\mid \cost(p,y)\leq r'}$ for $r'=r\sum_{q\in P}w(q)\cost(q,y)$, so it suffices to bound the dimension of $(P,Y,\cost)$
    which has the same dimension if the function $\cost(\cdot,\cdot)$ is replaced by $\cost'=e^{\cost(\cdot,\cdot)}$, i.e.
    \begin{equation}\label{bbb}
    \dim(P,Y,g)=\dim(P,Y,\cost)=\dim(P,Y,\cost').
    \end{equation}

    For every $y\in Y$ let $y'\in\REAL^m$ denote the concatenation of the parameters $W,\omega$ and the orthogonal bases of $S_1,\cdots,S_k$ into a single vector of length $m=O(d^2k)$.
    The value $e^{\cost(p,y)}$ can be evaluated for every $p\in P$ and $y'\in Y'$ using $t = O(m)$ operations.
    Let $P,Y$ be two sets, and $f:P\times Y\to \REAL$ be a function that can be evaluated by an algorithm that gets $(p,y)\in P\times Y$ and returns
    \[
    \cost'(p,y)=e^{\cost(p,y)}=-\sum_{i=1}^k \omega_i \exp(-W\dist^2(p,S_i))
    \]
    after $O(m)$ operations that include $k+1$ exponential operations; see Definition~\ref{op}.

    Applying Corollary~\ref{pdi} yields that the dimension of $(P,Y',\cost')$ is
    \[
    \dim(P,Y,\cost')\in O(d^4k^4).
    \]
    Combining this with~\eqref{bbb} and~\eqref{dimm} proves the corollary as
    \[
    \dim(D,\vartheta_k(0),f)\leq \dim(P,Y,g)=\dim(P,Y,\cost')\in O(d^4k^4).
    \]
    \end{proof}
    As stated in~\cite{feldman2011scalable,lucic2017training},
    the lower-bound of $\Omega(kd^2)$ was established by \cite{akama2011vc} for the dimension of $(P,Y,f)$ above.
    It is an open problem whether this gap can be closed further in the general setting.

\chapter{Coresets for Streaming Data} \label{sec Coresets for Streaming}
    In the previous sections we bound the sensitivity and dimension of the desired query space $(P',\CC,\ff,\loss)$.
    However, as stated later in Lemma~\ref{offline}, the construction time of the coreset is quadratic in $n$.
    This is due to the computation time of the sensitivity $s$ in Corollary~\ref{sensitivitylemmasquareddistances}.
    To obtain a near-linear time algorithm, we use the well-known streaming approach that is described in this section.
    It enables us to compute the coreset only on small subsets of the input $n$ times.
    Hence, we use it even if all the input points are given (off-line).

    The idea behind the merge-and-reduce tree that is shown in Algorithm~\ref{algsup}
    is to merge every pair of subsets and then reduce them by half.
    The relevant question is what is the smallest size of input that the given coreset can reduce by half.
    The log-Lipschitz property below is needed for approximating the cumulative error during the construction of the tree.

    In the following definition "sequence" is an ordered multi-set.
    \begin{definition}[input stream]
    Let $P$ be a (possibly infinite, unweighted) set.
    A\emph{ stream of points from $P$} is a procedure whose $i$th call returns the $i$th points $p_i$
    in a sequence $(p_1,p_2,\cdots)$ of points that are contained in $P$, for every $i\geq1$.
    \end{definition}

    \begin{definition}[halving function\label{halv}]
    Let $\eps,\delta, r>0$.
    A non-decreasing function $s:[0,\infty)\to[0,\infty)$ is an \emph{$(\eps,\delta,r)$-halving} function of a function
    $\size:[0,\infty)^4\to[0,\infty)$ if for every integer $h\geq1$, $n=s(h)$, and $w'= 2^hn$ we have
    \[
    \size(2n,w',\eps/h,\delta/4^h)\leq n,
    \]
    and $s$ is \emph{$r$-log-Lipschitz} over $[c,\infty)$ for some $c=O(1)$, i.e., for every $\Delta,h\geq c$ we have $s(\Delta h)\leq \Delta^r s(h)$.
    \end{definition}

    \begin{definition}[mergable coreset scheme\label{merge}]
    Let $(\coralg,\size,\time)$ be an $(\eps,\delta)$-coreset scheme for the query space $(P,Y,\cost,\loss)$,
    such that the total weight of the coreset and the input is the same, i.e. a call to $\coralg((Q,w),\eps,\delta)$ returns a weighted set $(C,u)$
    whose overall weight is $\sum_{p\in C}u(p)=\sum_{p\in Q}w(p)$. Let $s$ be an $(\eps,\delta,r)$-halving function for $\size$.
    Then the tuple ($\coralg, s,\time)$ is an $(\eps,\delta,r)$-\emph{mergable coreset scheme} for $(P,Y,\cost,\loss)$.
    \end{definition}

    \begin{algorithm}
        \caption{$\stralg(\stream,\eps,\delta,\coralg,s)$}
        \label{algsup}
        {\begin{tabbing}
        \textbf{Input:\quad }\quad\=An input $\stream$ of points from a set $P$, \\\>an error parameter $\eps\in(0, 1/2)$,
        probability of success $\delta \in(0, 1/2)$, and \\
        \textbf{Required:} \> An algorithm $\coralg$ and $s:[0,\infty)\to [0,\infty)$ such that $(\coralg,s,\time)$ is a \\
        \>mergable coreset scheme for $(P,Y,\cost,\loss)$.\\
        \textbf{Output:} \>A sequence $C'_1,C'_2,\cdots$ of coresets that satisfies Theorem~\ref{thmstream}.
        \end{tabbing}\vspace{-0.3cm}}
        \For {every integer $h$ from $1$ to $\infty$\label{forh1}}
        {Set $S_i:= \emptyset$ for every integer $i\geq 0$\\
        $T_{h-1}\gets S_{h-1}$ \label{forh15}\\
        \For{$2^{h-1}\cdot  s(h)$ iterations\label{forh}}
        {
        Read the next point $p$ in $\stream$ and add it to $S_0$\\
        \If {$|S_0|=s(h)$ \label{forh5}}
        {$i:=0$; $S:= \emptyset$\\
        \While {$S_i\neq  \emptyset$\label{forh7}}
        {$S:=\coralg\left(S\cup S_i,\frac{\eps}{h},\frac{\delta}{4^h}\right)$\label{forh8}\\
        $S_i:=\emptyset$\label{forh9}\\
        $i:=i+1$
        }
        $S_i:= S$\\
        }
        $C'_n:=\coralg\left(\left(\bigcup_{i=0}^{h-1}T_i\right)\cup \left(\bigcup_{i=0}^{h}S_i\right),\eps,\delta\right)$\label{forh12}\\
        \textbf{Output} $C'_n$\label{forh13}\\
        }
        }
    \end{algorithm}

    \begin{figure*}
        \centering
        \begin{subfigure}[t]{0.3\textwidth}
            \includegraphics[width=\textwidth]{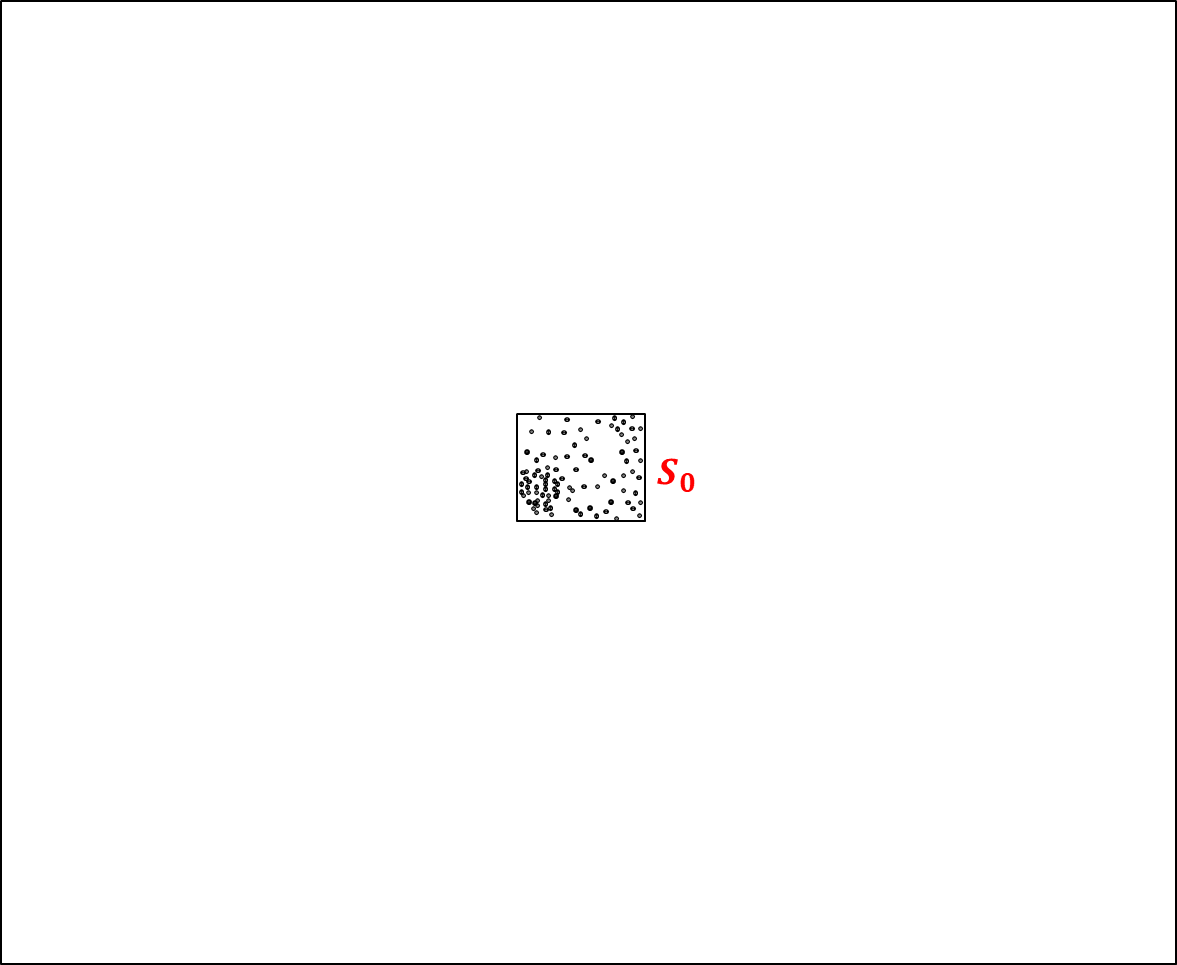}
            \caption{Construct a coreset $S_0$ of size $|S_0|=m/2$ from the first $m$ points in the stream.}
            \label{fig:streaming_1}
        \end{subfigure}
        ~ \quad
        \begin{subfigure}[t]{0.3\textwidth}
            \includegraphics[width=\textwidth]{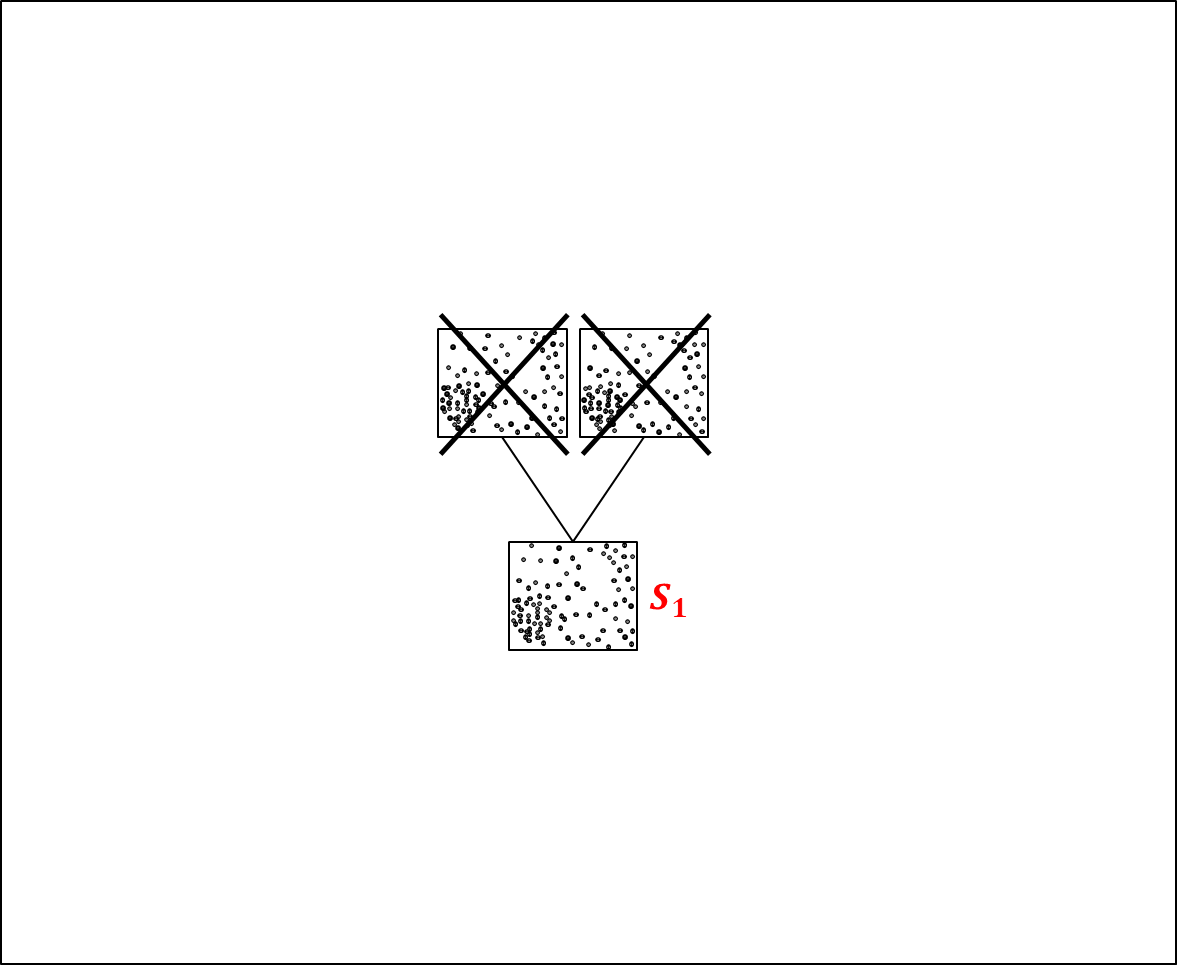}
            \caption{Read the next $m$ points, merge their coreset with $S_0$ to obtain $S_1$.}
            \label{fig:streaming_2}
        \end{subfigure}
        ~ \quad 
        \begin{subfigure}[t]{0.3\textwidth}
            \includegraphics[width=\textwidth]{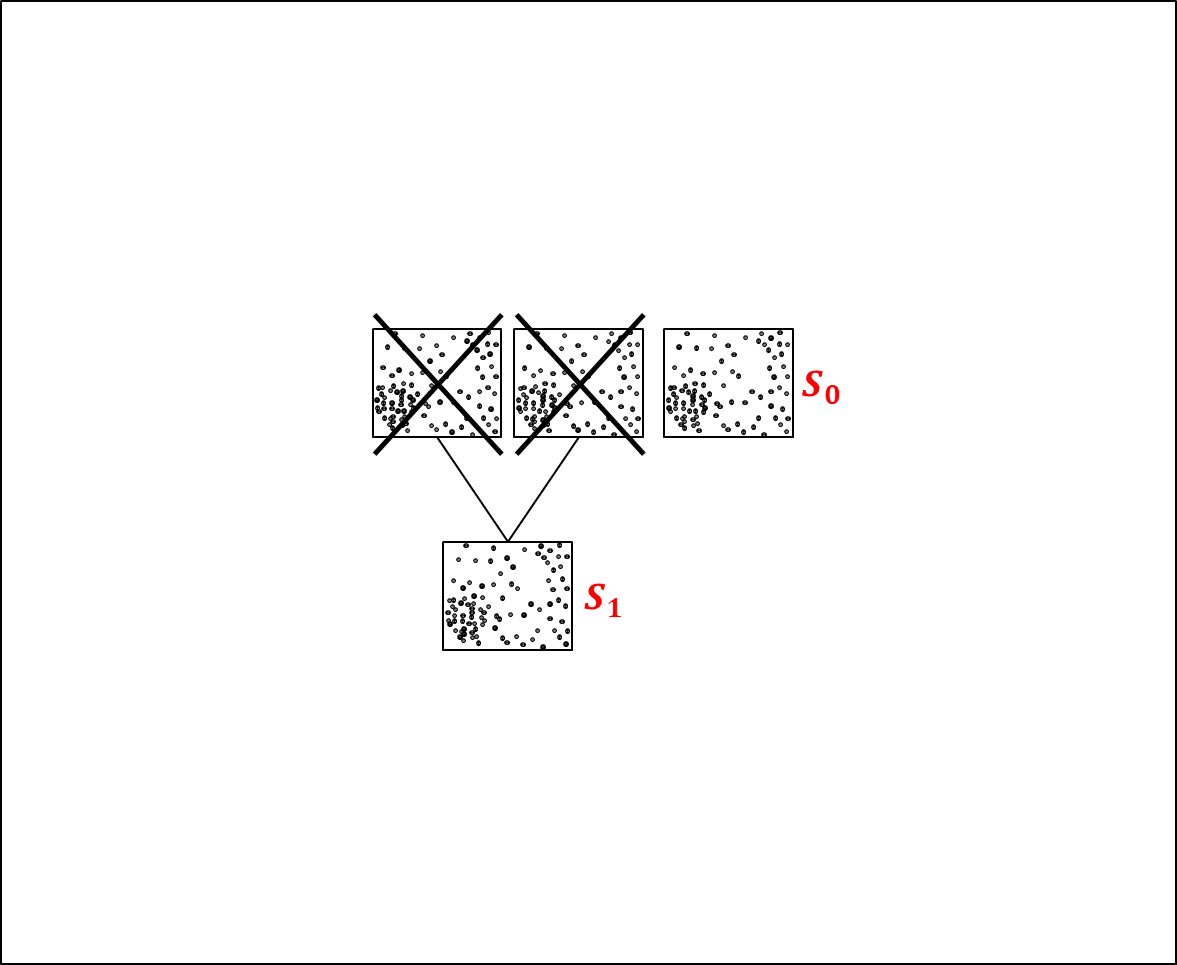}
            \caption{Read the next $m$ points, construct a coreset $S_0$ of size $|S_0|=m/2$.}
            \label{fig:streaming_3}
        \end{subfigure}
        ~

        ~

        \begin{subfigure}[t]{0.30\textwidth}
            \includegraphics[width=\textwidth]{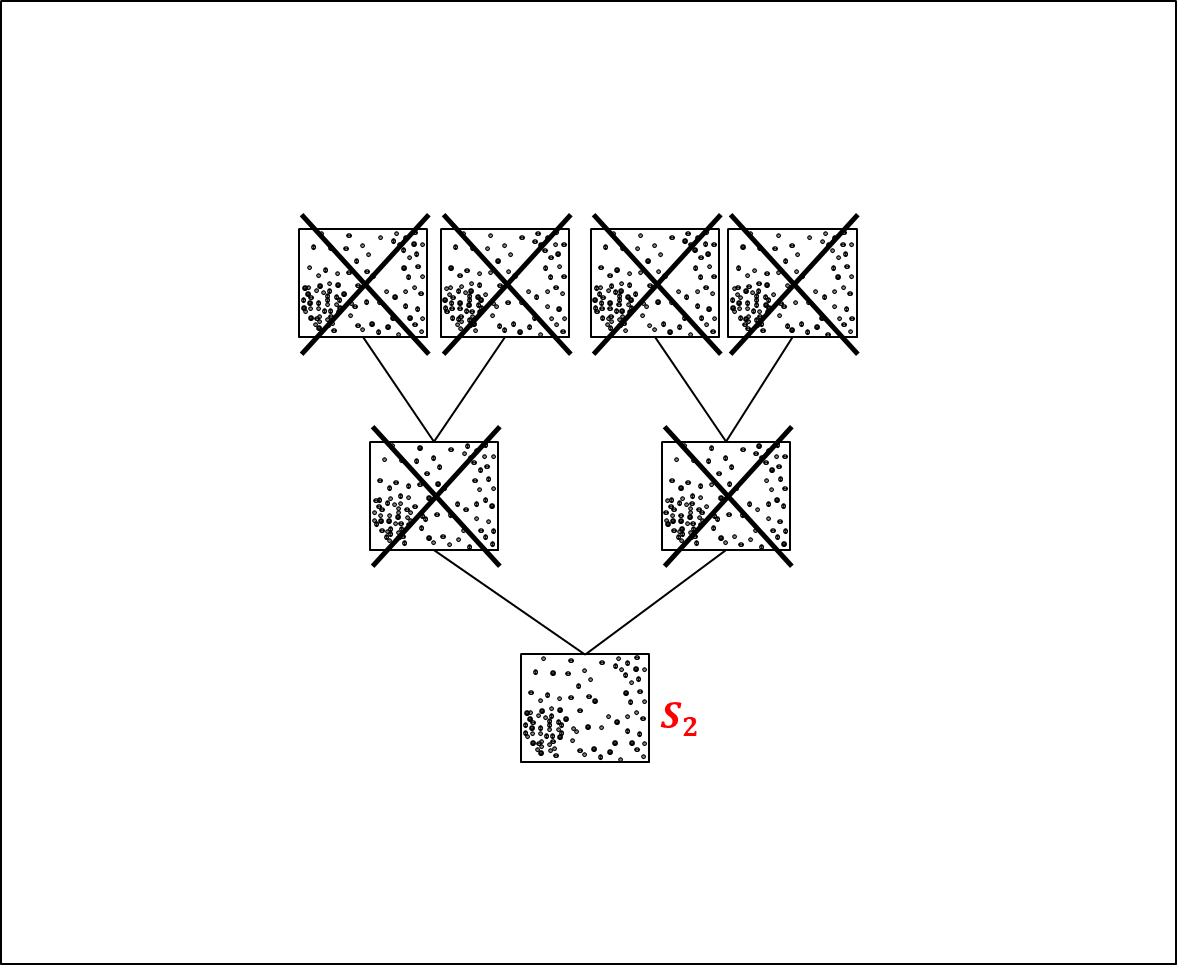}
            \caption{Read the next $m$ points, merge their coreset with $S_0$ then with $S_1$ to obtain $S_2$.}
            \label{fig:streaming_4}
        \end{subfigure}
        ~ \quad
        \begin{subfigure}[t]{0.30\textwidth}
            \includegraphics[width=\textwidth]{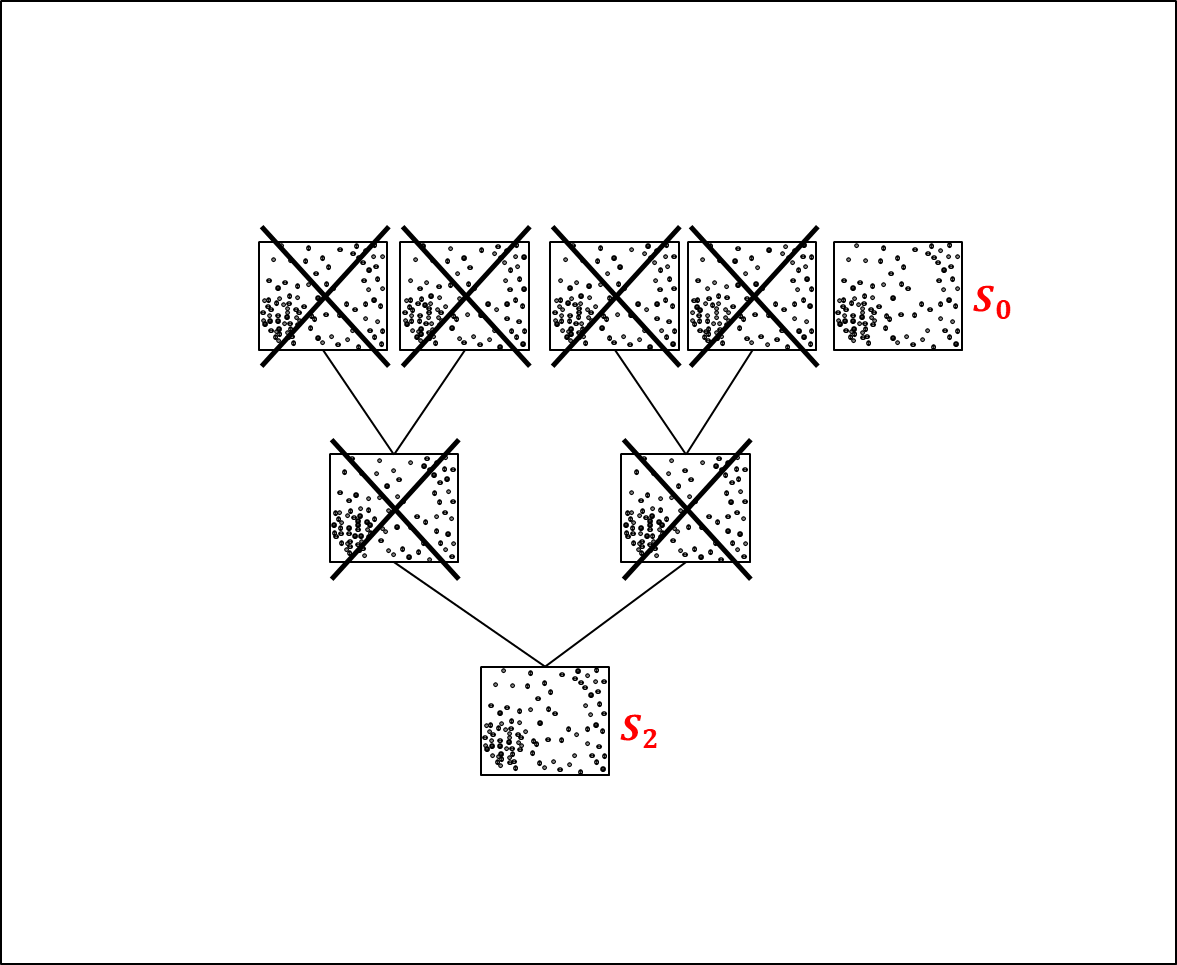}
            \caption{Read the next $m$ points, construct a coreset $S_0$ of size $|S_0|=m/2$.}
            \label{fig:streaming_5}
        \end{subfigure}
        ~ \quad 
        \begin{subfigure}[t]{0.30\textwidth}
            \includegraphics[width=\textwidth]{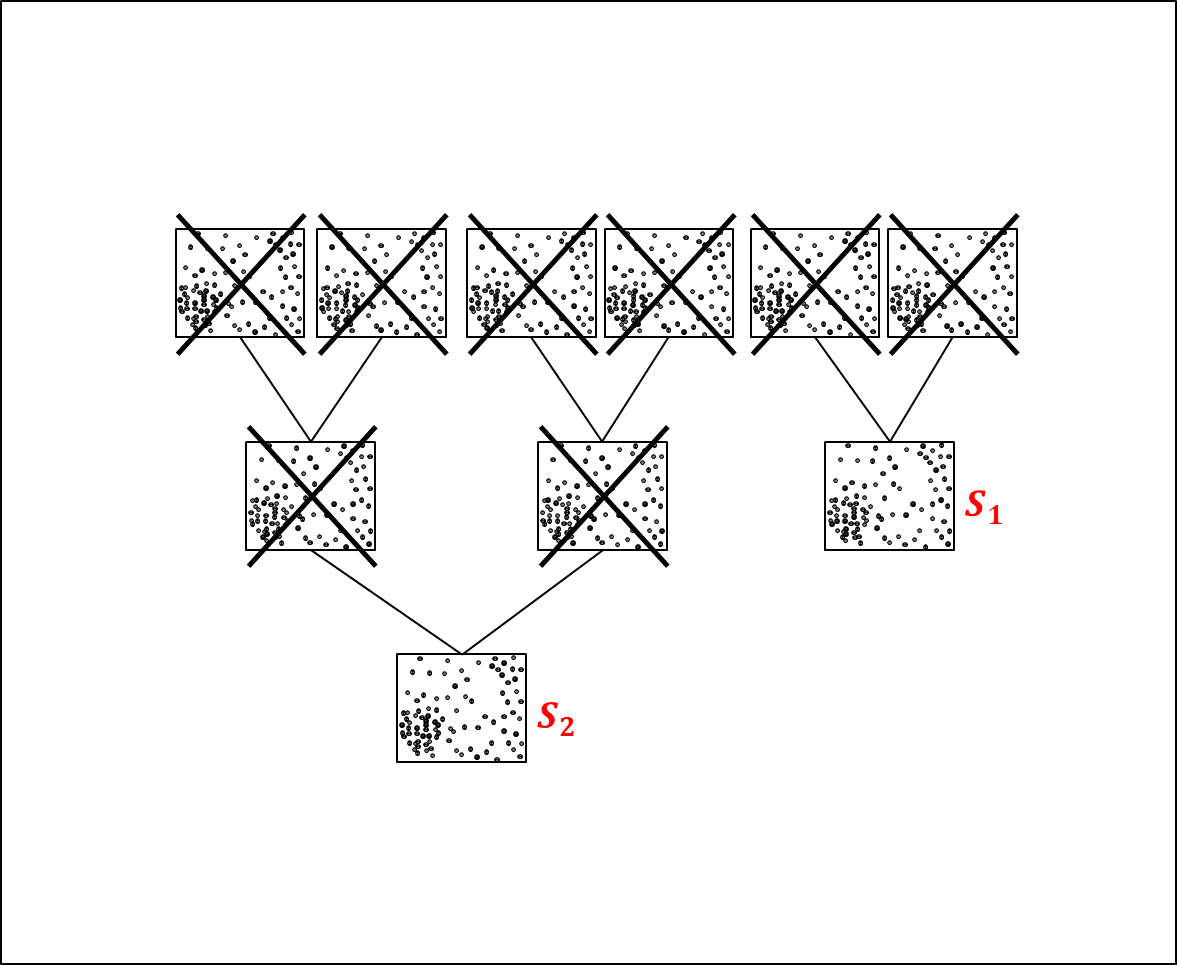}
            \caption{Read the next $m$ points, merge their coreset with $S_0$ to obtain $S_1$.}
            \label{fig:streaming_6}
        \end{subfigure}
        ~

        ~

        \begin{subfigure}[t]{0.30\textwidth}
            \includegraphics[width=\textwidth]{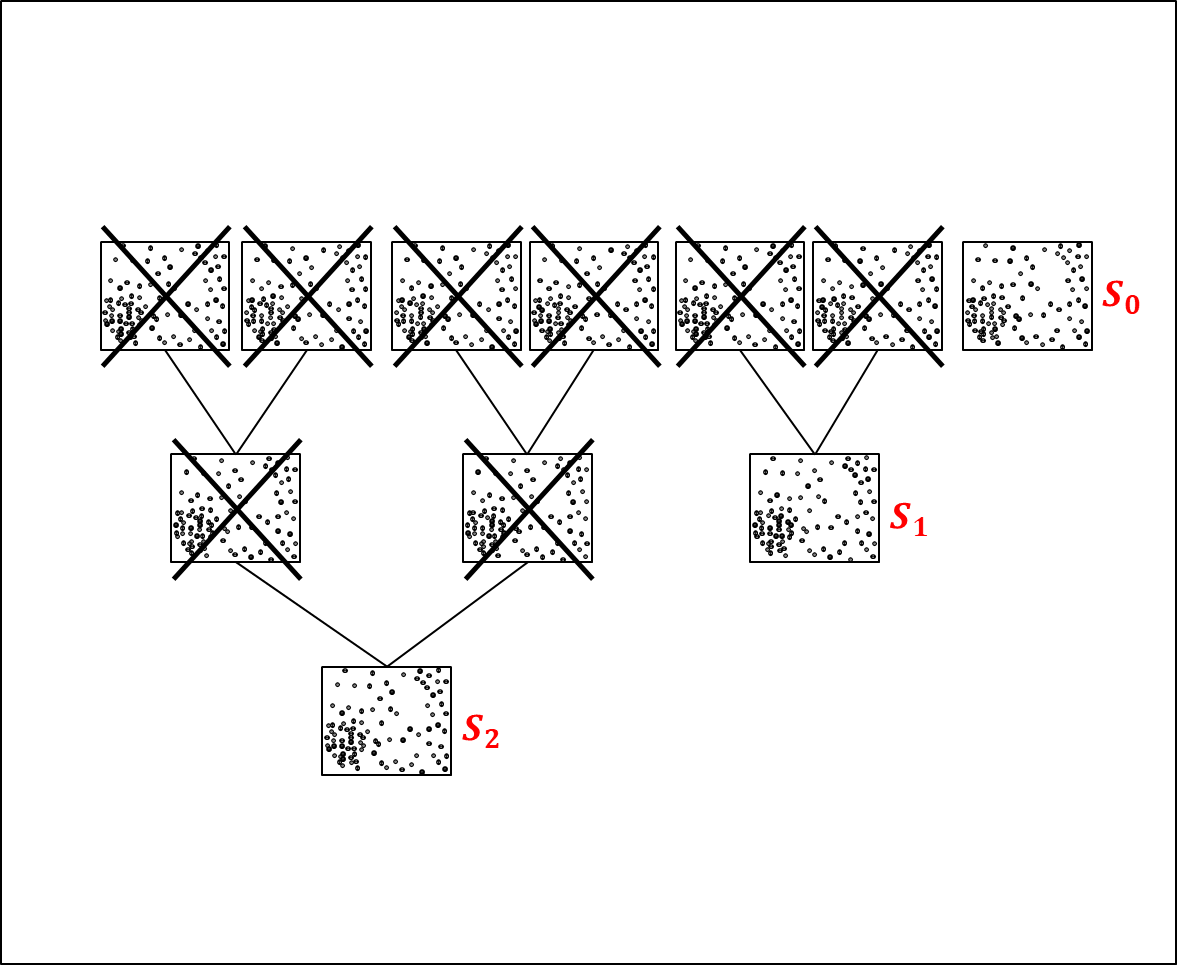}
            \caption{Read the next $m$ points, construct a coreset $S_0$ of size $|S_0|=m/2$.}
            \label{fig:streaming_7}
        \end{subfigure}
        ~ \quad
        \begin{subfigure}[t]{0.30\textwidth}
            \includegraphics[width=\textwidth]{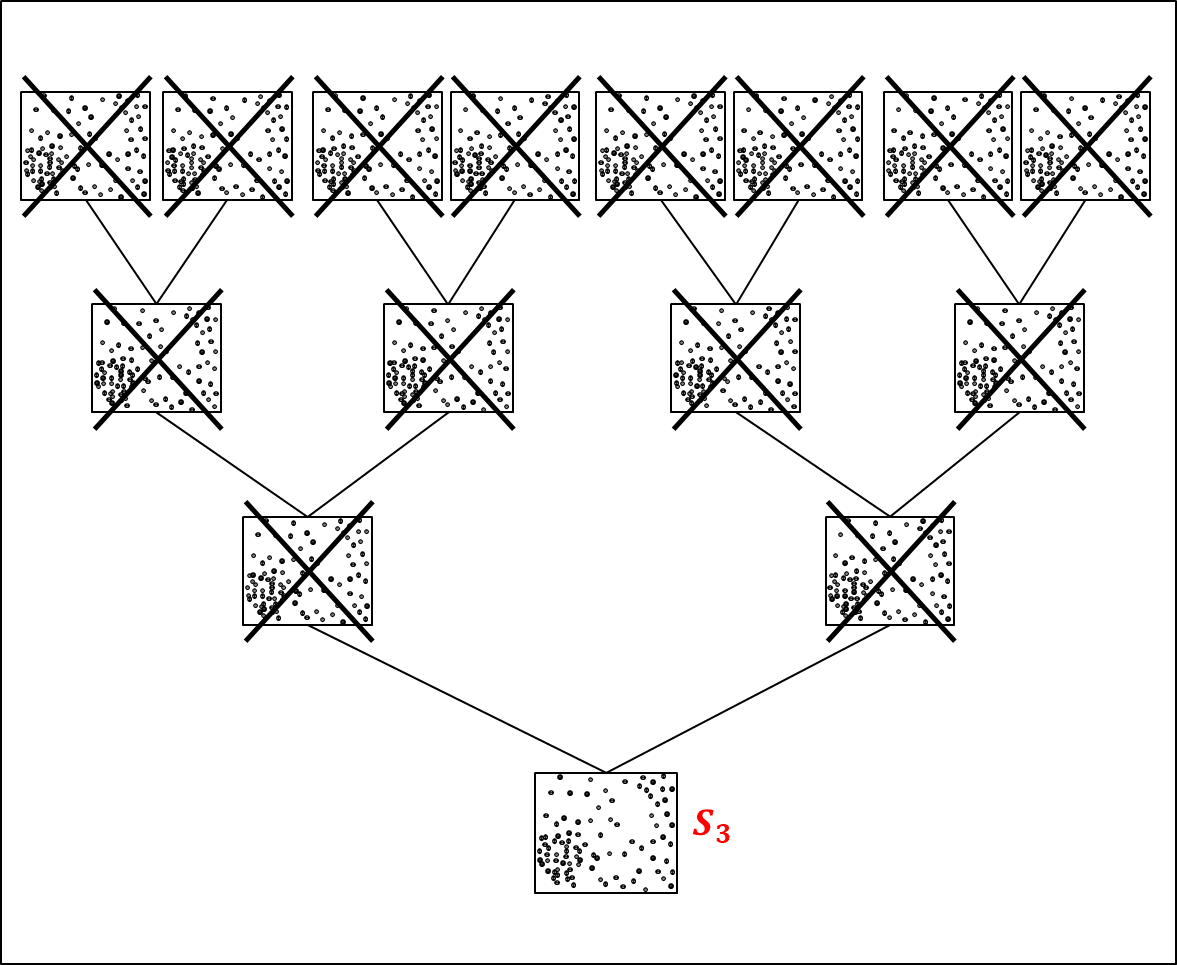}
            \caption{Read the next $m$ points, merged their coreset with $S_0$ then with $S_1$ then with $S_2$ to obtain $S_3$.}
            \label{fig:streaming_8}
        \end{subfigure}
        \caption[$\stralg$, known size]{Illustration of Algorithm~\ref{algsup}.
        Let $P$ be a set of $m$ points, error parameter $\varepsilon \in (0,\frac{1}{2})$ and probability of success $\delta \in (0,\frac{1}{2})$.
        Assume a coreset function $f(P,\varepsilon,\delta)$ returns a coreset of size $\frac{1}{2}m$ with $\varepsilon$ error parameter and $\delta$ probability of success.
        This figure shows an algorithm for $n=8m$ streaming points.
        The algorithm maintains a binary tree, where each new $n$ points are added to the tree as a leaf.
        Every two nodes with the same level are merged using the coreset function $f$ to a node in next level.
        Hence, each level has maximum of one node, and a total of $\mathcal{O}(\log \frac{n}{m})$ nodes.}
        \label{fig:streaming}
    \end{figure*}

	\begin{figure*}
    \centering
    \includegraphics[width=0.5\textwidth]{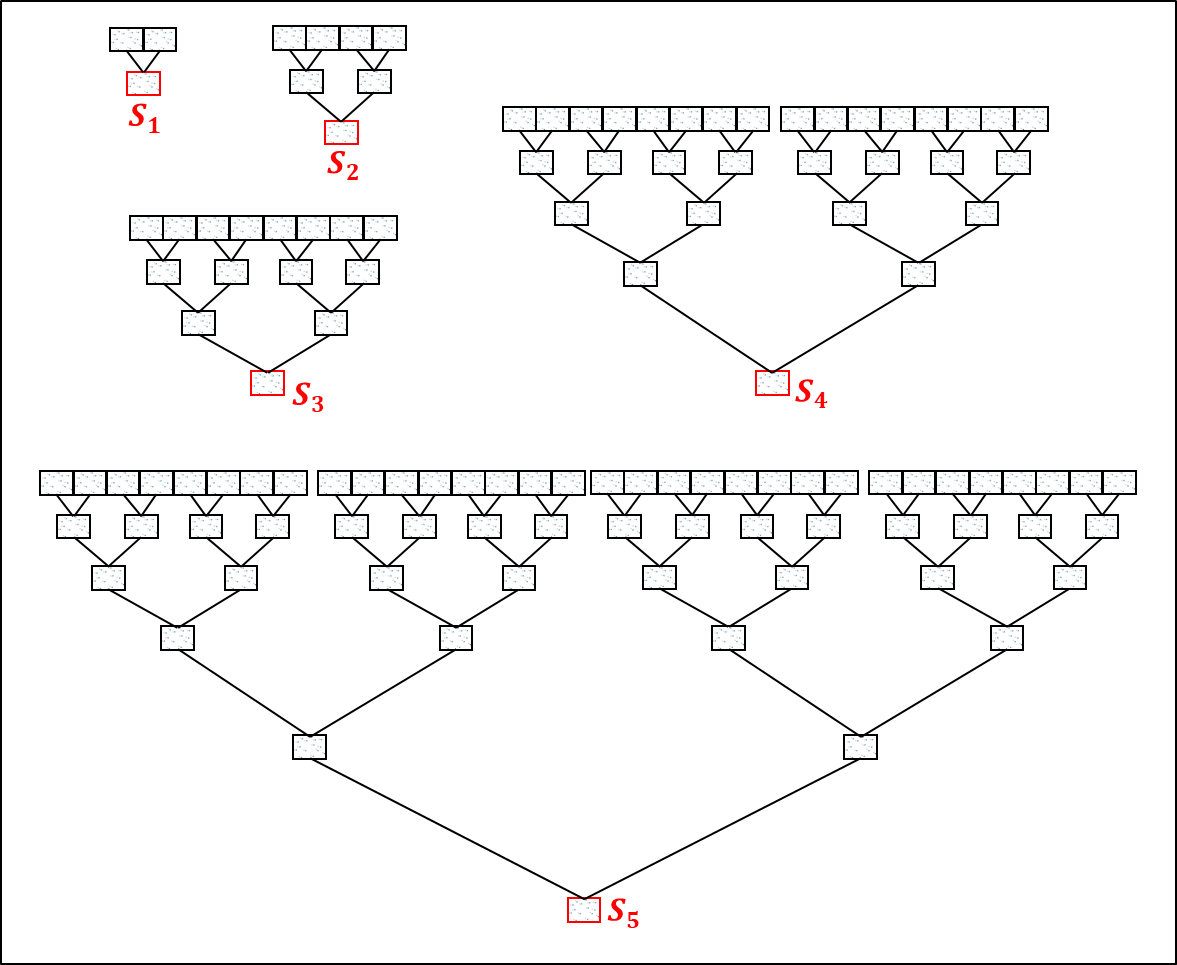}
    \caption[$\stralg$, un-known size]
    {
    Using coreset function on a corest increase the error.
    Therefore when building the tree, instead of using input $\varepsilon$, we use a scaled down  $\varepsilon$ with respect to the tree hight.
    However, when stream size is unknown (or $\infty$) the tree height is also unknown.
    So, we use several trees with height between $1$ and $\infty$.
    Here we can see five tress, where for each tree we store only the coreset at the head.
    }
    \label{fig:streaming_9}
    \end{figure*}

    The following theorem states a reduction from off-line coreset construction to a coreset that is maintained during streaming.
    The required memory and update time depends only logarithmically in the number $n$ of points seen so far in the stream.
	It also depends on the halving function that corresponds to the coreset via $s(\cdot)$.
	
	The theorem below holds for a specific $n$ with probability at least $1-\delta$.
	However, by the union bound we can replace $\delta$ by, say, $\delta/n^2$ and obtain,
    with high probability, a coreset $C'_n$ for each of the $n$ point insertions, simultaneously.
	
    \begin{theorem}[generalization of~\cite{feldman2013turning}]
    \label{thmstream}
    Let $(\coralg,s,\time)$ be an $(\eps,\delta)$-mergable coreset scheme for $(P,Y,\cost,\loss)$
    and $s$ be its \emph{$(\eps,\delta,r)$-halving} function
    of size $r\geq 1$ is constant, and $\eps,\delta\in(0,1/2)$.
    Let $\stream$ be a stream of points from $P$.
    \sloppy Let $C'_n$ be the $n$th output weighted set of a call to $\stralg(\stream,\eps/6,\delta/6,\coralg,s)$; see Algorithm~\ref{algsup}.
    Then, with probability at least $1-\delta$,
    \begin{enumerate}
    \renewcommand{\labelenumi}{\theenumi}
    \renewcommand{\theenumi}{(\roman{enumi})}
    \item \label{resss1}(Correctness) $C'_n$ is an $\eps$-coreset of $(P_n,Y,\cost,\loss)$, where $P_n$ is the first $n$ points in $\stream$.
    \item \label{resss2}(Size) $\displaystyle |C_n|\in \size(s(c), n, \eps,\delta)$ for some constant $c\geq 1$. 
    \item \label{resss3}(Memory) there are at most $b=s(c)\cdot O(\log^{r+1}n)$ points in memory during the computation of $C'_n$.  
    \item \label{resss4}(Update time) $C'_n$ is outputted in additional $t\in O(\log n)\cdot \time(b, n,\frac{\eps}{O(\log n)},\frac{\delta}{n^{O(1)}})$ time after $C'_{n-1}$.
    \item \label{resss5}(Overall time) $C'_n$ is computed in $nt$ time.
    \end{enumerate}
    \end{theorem}
    \begin{proof}
    \sloppy We prove that for a call to $\stralg(\stream,\eps,\delta,\coralg)$, the theorem holds if we replace $\eps$ with $6\eps$ in properties $(i)$ to $(v)$, and $\delta$ with $6\delta$.
    This would prove the theorem for a call to $\stralg(\stream,\eps/6,\delta/6,\coralg)$.

    Let $h\geq1$, and let $P_h$ denote the set of points that are read from $\stream$ during the $h$th "for" iteration in Line~\ref{forh1}.
    We consider the values of $h$, $S_i$, and $T_h$ during the time that $C'_n$ was outputted, after reading the first $n$ points from the stream.
    We define $s(h)$ as in Definition~\ref{halv}.

    The set $P_h$ can be partitioned into equal consecutive $2^{h-1}=|P_h|/s(h)$ subsets $P_{h,1},P_{h,2},\cdots,P_{h,m}$, each of size $s(h)$ by Line~\ref{forh}.
    We now recursively define a binary complete and full tree that corresponds to $P_h$ whose height is $h$ levels,
    where each of its nodes corresponds to an $(\eps/h)$-coreset $S$ that is computed, with probability at least $1-\delta/4^h$, in Line~\ref{forh8};
	 see Fig.~\ref{fig:streaming}.
    The $j$th leftmost leaf of the tree for $P_h$, for every $j\in[2^{h-1}]$, is the $(\eps/h)$-coreset of $P_{h,j}$.
    An inner node in the $i$th level corresponds to the $(\eps/h)$-coreset of the union $S_i$ of coresets that correspond to its pair of children and their corresponding input points.
    Hence, the root of this tree corresponds to the coreset $S_h = T_h$ of $P_h$; see Line~\ref{forh15}.

    The input to each coreset construction call is therefore the union $C_1\cup C_2$ of a pair of weighted sets.
    In the leaves, each coreset has size of at most $s(h)/2$ points, due to the definition of the halving function,
	so the input to the second level is of size $|C_1\cup C_2|\leq s(h)$ points.
    The sum of weights in $C_1\cup C_2$ equals to the number of input points they represent,
	by Definition~\ref{merge} of a mergeable coreset, and it is at most $W=|P_h|= 2^{h-1}s(h)$.
    The output coreset has therefore size at most $s(h)-1$ or $s(h)/2$, if $|C_1\cup C_2|\leq s(h)-1$ or $|C_1\cup C_2|= s(h)$, respectively.
    Similarly, in the higher levels, the input is a union of coresets, each of size at most $s(h)-1$, which is also an upper bound on the size of the output coreset.

    \paragraph{The probability}
    that a coreset call fails during the construction of a coreset for the tree of $T_h$ in Line~\ref{forh8} is $\delta/4^{h}$,
    and the number of such calls is the number $2^{h}-1$ of nodes in this tree.
    Using the union bound, one of these constructions will fail with probability at most
    \[
    \frac{\delta}{4^h}\cdot (2^{h}-1) < \frac{\delta}{2^{h}}\leq \frac{2\delta}{h^2}.
     \]
    The probability that one of the coreset during the construction of all the trees in the stream would fail is thus by the union bound,
    \begin{equation}\label{deldel}
    \sum_{h=1}^{\infty} \frac{2\delta}{h^2} =2\delta \sum_{h=1}^{\infty} \frac{1}{h^2}\leq 4\delta.
    \end{equation}

    Suppose that all the coreset constructions in Line~\ref{forh8} indeed succeed
    (which happens with probability at least $1-4\delta$). In particular, the input to $\coralg$ in Line~\ref{forh12} is a union of coresets.
    The construction of $C'_n$ in Line~ref{forh12} would fail with probability at most $\delta$, and thus $C'_n$ is an $\eps$-coreset with probability at least $1-5\delta\geq 1-6\delta$.

    \paragraph{The required memory }for storing the $O(h)$ coresets $S_0,\cdots,S_h$ during the construction of $T_h$ and the previous trees $T_1,\cdots, T_{h-1}$, each of size at most $s(h)$ is
    \begin{equation}\label{hbbound}
    O(h)\cdot s(h)\in O(h^{r+1})s(c)
    \end{equation}
    since $s(h)=s((h/c)c)\leq (h/c)^r s(c)\leq h^r s(c)$ is $r$-log-Lipschitz for a sufficiently large constant $c\geq1$ and $h\geq c^2$, by Definition~\ref{halv}. We now prove $h=O(\log n)$.

    We have
    \begin{equation}\label{ss}
    s(h)\leq 2^{r-1}(s(|h-1|)+s(|1-0|))\leq 2^{r}s(h-1),
    \end{equation}
    where the first inequality follows from the fact that $s$ is $r$-log-Lipschitz (see~\cite[Lemma 6.3]{braverman2016new}),
    and the second inequality holds since such a function is non-decreasing by definition, so $s(1)\leq s(h-1)$. Hence,
    \begin{equation}\label{PP}
    |P_h|= 2^{h-1}s(h)
    \leq 2^{h-1}\cdot 2^{r}\cdot s(h-1)
    =2^{r+1}\cdot 2^{h-2}s(h-1)
    =2^{r+1}|P_{h-1}|,
    \end{equation}
    where the first equality is by Line~\ref{forh}, and the inequality is by~\eqref{ss}. The value of $h$ is then bounded by
    \begin{equation}\label{hbound}
    \begin{split}
    h&=\log_2\left(\frac{|P_h|}{s(h)}\right)+1
    \leq \log_2\left(\frac{2^{r+1}|P_{h-1}|}{s(h)}\right)+1\\
    &\leq \log_2\left(\frac{2^{r+1}n}{s(h)}\right)+1
    \leq \log_2\left(2^{r+1}n\right)+1
    \leq  (r+1)+\log_2n+1\in O(\log n),
    \end{split}
    \end{equation}
    where the first equality is since $|P_h|=2^{h-1}s(h)$, and the first inequality is by~\eqref{PP}. 

    Plugging~\eqref{hbound} in~\eqref{hbbound} yields an overall memory as in Claim~\ref{resss3}
    \[
    O(h)\cdot s(h)\in O(h^{r+1})s(c)\subseteq O(\log n)^{r+1}s(c)=O(\log n)s(c).
    \]

    \paragraph{The multiplicative approximation error }in the coreset for the nodes of the tree $T_h$ increases by a multiplicative factor of
    $(1+\eps/h)$ in each level of the tree $T_h$, by Line~\ref{forh8}.
    We have
    \begin{equation}\label{epsln}
     \eps \leq \frac{2\eps}{1+2\eps}  \leq \ln(1+2\eps)
    \end{equation}
    where the first inequality holds since $\frac{x-1}{x} \leq \ln x$ for $x>0$, and the last inequality holds by the assumption $\eps<1/2$. Hence,
    \[
    \left(1+\frac{\eps}{h}\right)^{h}
    =\left(\left(1+\frac{\eps}{h}\right)^{h/\eps}\right)^{\eps}
    \leq e^{\eps}\leq 1+2\eps,
    \]
    where the last inequality is by~\eqref{epsln}.
    so the coreset $T_h=S_h$ that corresponds to the root is a $(2\eps)$-coreset for $P_h$.
    Hence, $\bigcup_{i=0}^{h-1}T_i$ is a $(2\eps)$-coreset for $\bigcup_{i=0}^{h-1}P_i$.
    Similarly, $\bigcup_{i=0}^{h}S_{i}$ is a $(2\eps)$-coreset for the points that were read from $P_h$.
    Hence, in Line~\ref{forh12}, $C'_n$ is an $\eps$-coreset of a union of $(2\eps)$-coresets,
    which implies that $C'_n$ is a $(4\eps)$-coreset, as $(1+2\eps)(1+\eps)\leq (1+4\eps)$ where in the last inequality we use the assumption $\eps\leq 1/2$. This proves Claim~(i).

    \paragraph{Update time.  }In the worst case, the "while" loop in Line~\ref{forh7} is executed for all the $h$ levels of $T_h$.
    In this case, we construct $O(h)$ coresets, each for input of at most $s(h)$ points whose overall weight is at most $n$,
    in overall $O(h)\cdot \time(s(h),n,\frac{\eps}{h},\frac{\delta}{4^h})$ time.
    In Line~\ref{forh12} we compute a coreset for the union of $O(h)$ coresets, each of size at most $s(h)$,
    which represents $n$ input points, so their overall size is $O(h)s(h)$ and their construction time is $\time\left(O(h)s(h), n,\eps,\delta\right)$.
    Substituting in the last expression, $h\in O(\log n)$ and $O(h)=O(h^{r+1})s(c)$ by~\eqref{hbound} and~\eqref{hbbound}, respectively,
    yields the update time in Clam~\ref{resss4}, as
    \[
    \begin{split}
    &O(h)\time\left(s(h), n,\frac{\eps}{h},\frac{\delta}{4^h}\right)+\left(O(h)s(h), \eps,\delta\right)
    \subseteq O(h)\time\left(O(h)s(h), n,\frac{\eps}{h},\frac{\delta/4^h}{4^hs(h)}\right)\\
    &\subseteq O(\log n)\time\left(s(c)\log^{r+1}n, n,\frac{\eps}{O(\log(n))},\frac{\delta}{n^{O(1)}}\right).
    \end{split}
    \]

    \paragraph{The overall running time }
    for computing $C'_n$ is obtained by multiplying the update time per point in the previous paragraph by $n$ updates, which proves Claim~\ref{resss5}.
    \end{proof}

    \section{Halving Calculus}
        If we have a coreset of size $m(\eps,\delta)$ that is independent of the total weight of the or number of input points,
        for every input set $P$, then $\size(n,\cdot,\eps/h,\delta/h)\leq n$ for $n=s(h)$ if $s(h)=m(\eps/h,\delta/h)$. Otherwise the analysis is a bit more involved.
        In this case we need to solve equations such as $\log(n)= n/2$
        for computing the halving function. There is no close solution for such equations
        but the solution can be represented by the following function
        which can be computed in a very high precision
        (a bit for every iteration) using e.g. Newton-Raphson method~\cite{chapeau2002numerical}.

        \begin{lemma}[Lambert $W$ function~\cite{barry2000analytical}\label{lam}]
        There is a unique monotonic decreasing function $W_{-1}:[-1/e,0]\to[-1,-\infty)$ that satisfies
        \[
        x=W_{-1}(x)e^{W_{-1}(x)},
        \]
        for every $x\in [-1/e,0]$. It is called the \emph{lower branch of the Lambert $W$ function}.
        \end{lemma}

        The following lemma would be useful to compute the halving function
        of coresets whose size depend on $n$, as in this thesis.
        \begin{lemma}\label{firstlem}
        For every $c\geq1$ and $\eps\in(0,1/e^{c}]$, if
        \begin{equation}\label{nn}
        n\geq \left(\frac{4}{\eps}\ln\frac{4}{\eps}\right)^c,
        \end{equation}
        then
        \begin{equation}\label{2n}
        n\geq \left(\frac{\ln n}{c\eps}\right)^c.
        \end{equation}
        Equality holds for $n=e^{-cW_{-1}(-\eps)}$.
        \end{lemma}
        \begin{proof}
        We denote the Lambert W function by $w$ instead of $W_{-1}$ for simplicity.
        Multiplying~\eqref{2n} by $(c\eps)^c/n$ yields that we need to prove $(c\eps)^c \geq (\ln^cn)/n$.
        The right hand side is a non-increasing monotonic function in the range
        $n\geq 2e^c\ln(2e^c)\geq e^c$, as the enumerator of its derivation is
        \[
        (\ln^c n)' n-\ln^c(n)= (c\ln^{c-1} n)\cdot (\ln n)' n    -\ln^c n =\ln^{c-1} n (c-\ln n)\leq0.
        \]
        By letting $x=e^{-cw(-\eps)}$, it thus suffices to prove that (i) $(c\eps)^c=\ln^c(x)/x$, and (ii) $n\geq x$.

        \textbf{Proof of (i): } Taking $-(1/c)\ln$ of both sides in $x=e^{-cw(-\eps)}$ yields $-\ln(x)/c=w(-\eps)$. Hence,
        \[
        \begin{split}
        c\eps=-c(-\eps)=-c\cdot w(-\eps)e^{w(-\eps)}
        =\ln(x)e^{-\ln (x)/c}=\frac{\ln(x)}{x^{1/c}},
        \end{split}
        \]
        where the first equality is by the definition of $w(-\eps)$. Taking the power of $c$
        proves \textbf{(i)}.

        \textbf{Proof of (ii): }
        By~\eqref{nn}, to prove $n\geq x$ it suffices to prove
        \begin{equation}\label{xx}
        \left( \frac{4}{\eps}\ln\frac{4}{\eps}\right)^c\geq x.
        \end{equation}
        Let $a=\eps/2$ and
        \[
        b=\ln\left(\frac{1}{a}\ln\frac{1}{a^2}\right).
        \]
        By the assumption $\eps\leq 1/e^c\leq 1/e$, we have
        \begin{equation}\label{bb3}
        b=\ln\left(\frac{2}{\eps}\ln\frac{4}{\eps^2}\right) \geq \ln(2\ln(4))\geq\ln(2e)\geq 1.
        \end{equation}
        Hence,
        \begin{equation}\label{www}
        -a
        \leq -a\cdot \frac{\ln\left(\frac{1}{a}\ln\frac{1}{a^2}\right)}{\ln\frac{1}{a^2}}
        = -\ln\left(\frac{1}{a}\ln\frac{1}{a^2}\right)\cdot \frac{a}{\ln\frac{1}{a^2}}
        =(-b)\cdot e^{-b}=w^{-1}(-b),
        \end{equation}
        where the inequality holds since $1/a\geq \ln(1/a^2)$ for $1/a=2/\eps\geq 1$,
        and the last equality holds by letting $y=w^{-1}(-b)$
        so that $w(y)=-b$ and $y=w(y)e^{w(y)}$ by the definition of $w$.
        Since $w$ is monotonic decreasing we have by~\eqref{www} that
        \begin{equation}\label{wwa}
        w(-a)\geq w(w^{-1}(-b))=-b.
        \end{equation}
        This proves~\eqref{xx} as
        \[
        x=e^{-cw(-a)}\leq e^{cb}
        =\left(\frac{2}{a}\ln\frac{1}{a}\right)^c
        \leq\left(\frac{4}{\eps}\ln\frac{4}{\eps}\right)^c,
        \]
        where the first inequality is by~\eqref{wwa}.
        \end{proof}

        \begin{corollary}\label{corhal}
        Let $\eps,\delta>0$, and $u:[0,\infty)\to(1,\infty)$ such that $u(\cdot)$ is $r$-log Lipschitz function for some $r\geq 1$.
        Let $c\geq1$ and $\size:[0,\infty)^4\to[0,\infty)$ be a function such that
        \begin{equation}\label{sizea}
        \size(2n,\overline{w},\eps/h,\delta/4^h) \leq \left(\frac{u(h)\ln(\overline{w})}{hc}\right)^c
        \end{equation}
        for every $h,n,\overline{w}\geq 1$.
        Let overloading of $s:[0,\infty)\to[0,\infty)$ be a function such that
         \begin{equation}\label{ssize}
            s(h)\geq \big(4u(h)\ln (4u(h))\big)^c
         \end{equation}
        for every $h\geq1$. Then $s$ is an $(\eps,\delta,2cr)$-halving function of the function $\size$; see Definition~\ref{halv}.
        \end{corollary}
        \begin{proof}
        Let $h\geq1$, $n=s(h)$, and $\overline{w}= 2^{h}n$.
        Then $s$ is a halving function of $\size$ as
        \begin{align}
        \size(2n,\overline{w},\eps/h,\delta/4^h)
        \label{sizeb}&\leq \left(\frac{u(h) \ln(\overline{w})}{hc}\right)^c\\
        \label{sizebb}&= \left(\frac{u(h) \ln(2^{h}n)}{hc}\right)^c\\
        \label{sized}&\leq \left(\frac{u(h)\ln(n)}{c}\right)^c
        \leq n,
        \end{align}
        where~\eqref{sizeb} is by~\eqref{sizea},~\eqref{sizebb} is by the definition of $\overline{w}$,~\eqref{sized} is since $2\leq 4u\leq s(h)=n$
        by the definition of $s$, and the last inequality holds by replacing $\eps$ with $1/u(h)$ in Lemma~\ref{firstlem}.

        If $g(x)=(x\ln x)^c$, then for every $x,b\geq e>2$
        \begin{align}
       \nonumber g(bx)
        &=(b x)^c(\ln(b)+\ln x)^c
        \leq (b x)^c(\ln(b)\ln x)^c
        = (x\ln x)^c(b\ln b)^c\\
\label{label2}&\leq b^{2c}(x\ln x)^c =b^{2c}g(x),
        \end{align}
        where the first inequality holds since $(a+y)\leq 2y\leq ay$ for every $y\geq a\geq 2$, and~\eqref{label2} holds since $\ln b\leq b$. 
        Since $s(h)=g(u(h))$, for every $\Delta\geq e$ we have
        \[
        s(\Delta h)=g(u(\Delta h))\leq g(\Delta^{r}u( h))\leq \Delta^{2cr}g(u(h))= \Delta^{2cr}s(h),
        \]
        where the first inequality holds since $u$ is $r$-log-Lipschitz and the second holds by substituting $b=\Delta^r$ in~\eqref{label2}. Jointly with~\eqref{sized} we obtain that $s$ is $(\eps,\delta,2cr)$-halving function of $\size$.
        \end{proof}

        \begin{algorithm}
            \caption{$\coralgfinal(P',k,\varepsilon,\delta)$}
            \label{Alg_kgmmcoreset}
            {\begin{tabbing}
            \textbf{Input:\quad }\quad\=
                A weighted set $P'=(P,w)$ of points in $\REAL^d$,
                $k\in \mathbb{N} \cap [1,\inf]$ number clusters,
                an approximation error $\varepsilon>0$
                and a probability $\delta$ of failure
            \\ \textbf{Output:} \> $\eps$-coreset $(C,u)$ for $k$-GMMs of $P$,
                with probability at least $1-\delta$; see Theorem~\ref{thm33}.
            \end{tabbing}\vspace{-0.3cm}}
                $\ell_{\infty}:= \linfcore(P,k,\varepsilon)$ \tcp{See Algorithm~\ref{Alg_linfcore}}
                $s:= \redd(P',\varepsilon,\delta,\ell_{\infty})$  \tcp{See Algorithm~\ref{Alg_wsensitivity}}
                $t=:=\sum_{p\in P}s(p)$ \\
                $d'=k^4 d^3$ \\
                $m:=f(\varepsilon,\delta,d',t)$ \\
                $(C,u):= \impalg(P,w,s,m)$ \tcp{See Algorithm~\ref{Alg_Coreset}}
                \textbf{Return} $(C,u)$ \\
        \end{algorithm}

\chapter{Wrapping All Together}
In this section we use the previous chapter to give an example coreset for any $k$-GMM. First we suggest an inefficient construction (quadratic running time in $n$). Then we use it in the streaming setting to obtain time that is near-linear in $n$.
    \newcommand{\evcore}{\textsc{$\ell_{\infty}$-Projective-Clustering}}
            \begin{algorithm}
            \caption{$\projcore_{\subcore}(P,k,\varepsilon)$}
            \label{Alg_linfcore}
            {\begin{tabbing}
            \textbf{Input:\quad }\quad\=
                A finite set $P\subseteq \br{-M,-M+1,\cdots,M}^d$ for some integer $M\geq1$,\\
                \>an integer $k\geq1$, and an approximation error $\eps\in(0,1)$.
            \\ \textbf{Required:} \>
                An algorithm $\subcore(P,\eps)$ that returns an $\eps$-coreset for $(P,H_{d,k},\dist,\norm{\cdot}_{\infty})$.
            \\ \textbf{Output:} \> An $\eps$-coreset $C\subseteq P$ for $(P,H_{d,k},\dist,\norm{\cdot}_{\infty})$; see Theorem~\ref{evmaintheorem}.
            \end{tabbing}\vspace{-0.3cm}}
                \If {$k=1$}
                {
                  \Return $\textsc{Subspace-Coreset}(P,\varepsilon)$ \\
                }
                $C:= \projcore(P,k-1,\varepsilon)$ \\
                \For {every $v_0\in C$}
                {
                  $P[v_0]:= P$ \\
                  $C[v_0]:=\textsc{Recursive}(P[v_0],k,\varepsilon,\br{v_0})$\\
                  $C:= C \cup C[v_0]$
                }
                \Return $C$
        \end{algorithm}

{\setstretch{1.0}
        \begin{algorithm}
            \caption{$\linfrec(P,k,\varepsilon,V)$}
            \label{Alg_linfrec}
            {\begin{tabbing}
            \textbf{Input:\quad }
            \quad\=A finite set $P\subseteq \br{-M,-M+1,\cdots,M}^d$ for some integer $M\geq1$,\\
            \> an integer $k\geq1$, an approximation error $\eps\in(0,1]$, and a set $V=\br{v_0,\cdots,v_t}\subseteq P$.\\
            \textbf{Output: }A set $C\subseteq P$ that is returned to the caller, Algorithm~\ref{Alg_linfcore}.
            \end{tabbing}\vspace{-0.3cm}}
                $C:=\emptyset$; $A[\br{v_0}]=\br{v_0}$\\
                \If {$t\geq 1$}{
                \For {$i:=1$ to $t$ }
                {
                  \newcommand{\aff}{\mathrm{aff}}
                  $A[\br{v_0,\cdots,v_{i}}]:=\br{\sum_{h=0}^{i}\alpha_h v_h\mid \sum_{b=0}^{i}\alpha_b=1}$\\
                  \tcp{The affine $i$-subspace that passes through $v_0,\cdots, v_i$.}

                  Set $\pi(v_{i},A[\br{v_0,\cdots,v_{i-1}}])\in \argmin_{x\in A[\br{v_0,\cdots,v_{i-1}}]}\norm{v_i-x}_2$\\
                  \tcp{The projection (closest point) of $v_i$ onto $A[\br{v_0,\cdots,v_{i-1}}]$}

                  $u_{i}:= v_{i} - \pi(v_{i},A[\br{v_0,\cdots,v_{i-1}}])$\\
                  \tcp{The vector from $v_{i}$ to its projection on $A[\br{v_0,\cdots,v_{i-1}}]$}
                }
                $R[V]:= \br{v_0 + a_1 u_1 + \cdots + a_t u_t \mid a_i \in[-1,1], i\in [t]} $
                \tcp{A $t$-dimensional rectangle centered at $v_0$ whose $i$th side length is $2\norm{u_i}$.}   %
                $\displaystyle r[V]:= \br{a_1 u_1 + \cdots + a_t u_t\mid a_i\in [-\eps,\eps], i\in[t]}$\\
                \tcp{A $t$-dimensional rectangle whose $i$th side length is $2\eps\norm{u_i}$}
                $\displaystyle \mathcal{R}[V]\gets$ A partition of $R[V]$ into $\frac{1}{\eps^t}$ translated copies of $r[V]$\\
                \For{each rectangle $R\in \mathcal{R}[V]$}
                {
                  $C_R[V]:=\projcore(P \cap R,k-1,\varepsilon)$\\\tcp{see Algorithm~\ref{Alg_linfcore}.}
                  $C:= C \cup C_R[V]$
                }
                }
                \If{$t \leq d-1$}
                {
                      $B_0[V]:=P\cap A[V]$ \\
                      $c:=1/d^{3(d+1)/2}$\\
                  \For {$j:=1$ to $8d\log_2M+\log_2(1/c)$}
                  {
                    $\displaystyle B_j[V]:=\br{p\in P \mid 2^{j-1}c/M^{d+1} \leq  \dist(p,A[V]) <  2^jc/M^{d+1}}$ \\ \tcp{$\dist(p,A[V])$ is the distance from $p$ to $A[V]$.}
                    $K_j[V]:= \projcore(B_j[V],k-1,\varepsilon)$\\
                    $C := C \cup K_j[V]$ \\
                    \For {every $v_{t+1} \in K_j[V]$}
                    {
                      $V':=V\cup \br{v_{t+1}}\quad$\tcp{$V'=\br{v_0,\cdots,v_{t+1}}$}
                    $P[V']:= \bigcup_{i=0}^j B_i[V]$ \\
                    $C[V']:=\linfrec(P[V'],k-1,\varepsilon,V')$\\
                      $C := C \cup C[V']$
                    }
                  }
                }
                \Return $C$
        \end{algorithm}
}
    \section{Inefficient off-line construction}

        \begin{figure*}
            \centering

            \begin{subfigure}[t]{0.3\textwidth}
                \includegraphics[width=\textwidth]{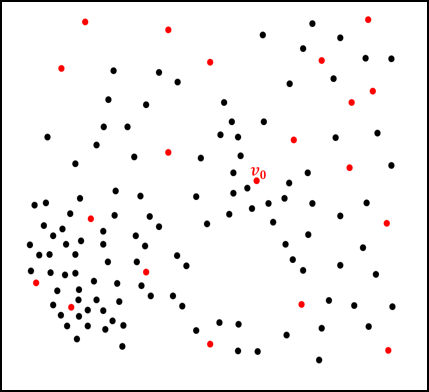}
                \caption{Compte a coreset (in red) for $(k-1)$ hyperplane queries recursively.}
                \label{fig:Alg_1}
            \end{subfigure}
            ~ \quad
            \begin{subfigure}[t]{0.3\textwidth}
                \includegraphics[width=\textwidth]{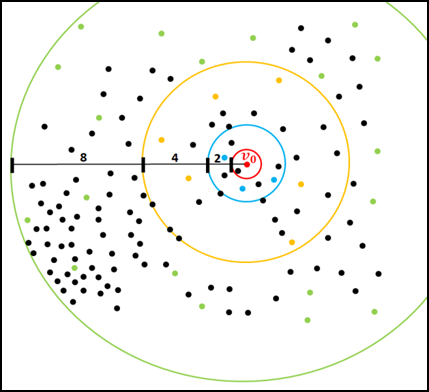}
                \caption{For each coreset point $v_0$, partition the input into exponential increasing balls around $v_0$. Compute a coreset for $k-1$ hyperplanes in each ring.}
                \label{fig:Alg_2}
            \end{subfigure}
            ~ \quad 
            \begin{subfigure}[t]{0.3\textwidth}
                \includegraphics[width=\textwidth]{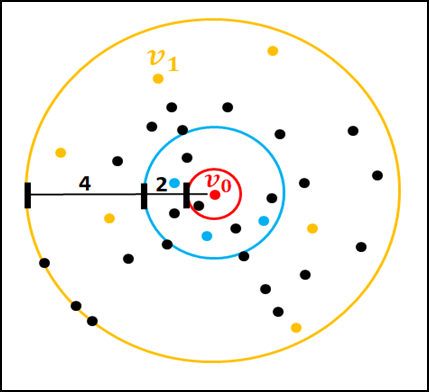}
                \caption{For each point $v_1$ in the last coreset, run the remaining steps while ignoring points in balls that enclose $v_1$.}
                \label{fig:Alg_3}
            \end{subfigure}
            ~

            ~

            \begin{subfigure}[t]{0.30\textwidth}
                \includegraphics[width=\textwidth]{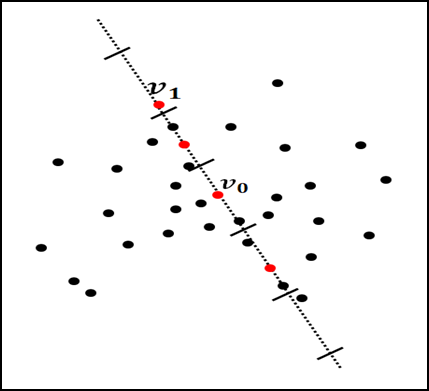}
                \caption{Partition the segment between $v_0$ and $v_1$ into $O(1/\eps)$ intervals. Compute a $k-1$ coreset for each interval.}
                \label{fig:Alg_4}
            \end{subfigure}
            ~ \quad
            \begin{subfigure}[t]{0.30\textwidth}
                \includegraphics[width=\textwidth]{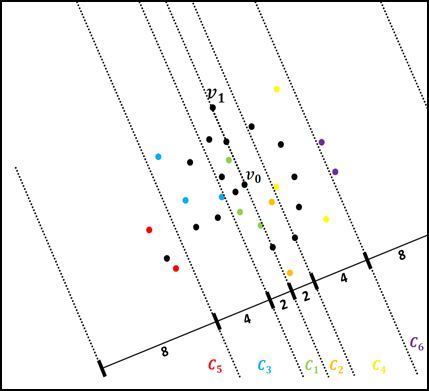}
                \caption{Split the data into strips using exponential growing parallel strips around the line between $v_0$ and $v_1$. Compute $k-1$ coreset for each such strip.}
                \label{fig:Alg_5}
            \end{subfigure}
            ~ \quad 
            \begin{subfigure}[t]{0.30\textwidth}
                \includegraphics[width=\textwidth]{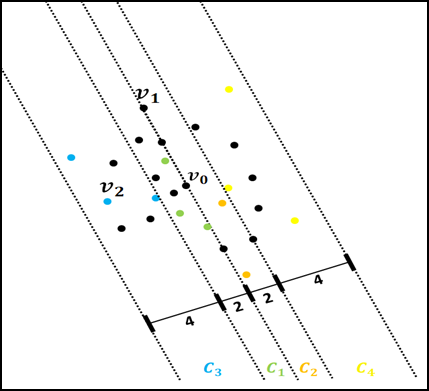}
                \caption{For each point $v_2$ in the third coreset, consider only points in the strips up to the one that contains $v_2$.}
                \label{fig:Alg_6}
            \end{subfigure}
            ~

            ~

            \begin{subfigure}[t]{0.30\textwidth}
                \includegraphics[width=\textwidth]{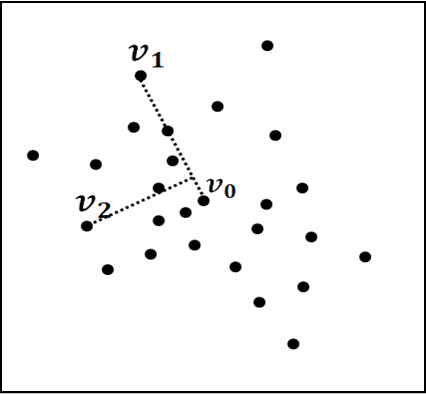}
                \caption{Compute the projection of $v_2$ onto the segment between $v_0$ and $v_1$.}
                \label{fig:Alg_7}
            \end{subfigure}
            ~ \quad
            \begin{subfigure}[t]{0.30\textwidth}
                \includegraphics[width=\textwidth]{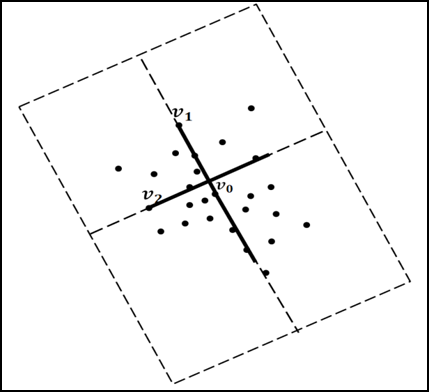}
                \caption{The segment and the projection of $v_2$ on it form a rectangle that contains all the points (after scaling by a factor of $2$).}
                \label{fig:Alg_8}
            \end{subfigure}
            ~ \quad 
            \begin{subfigure}[t]{0.30\textwidth}
                \includegraphics[width=\textwidth]{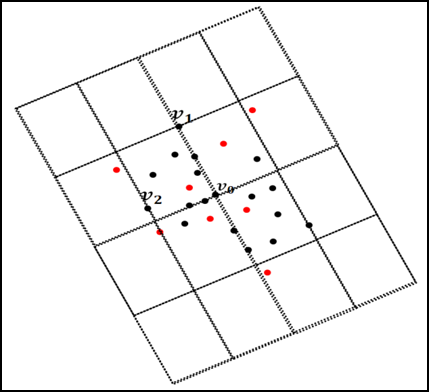}
                \caption{Split the rectangle into small scaled $O(1/\eps^d)$ rectangles of the same shape. Compute a $(k-1)$-coreset for each rectangle.}
                \label{fig:Alg_9}
            \end{subfigure}

            \caption[$\linfcore$]{Step by step visualization of Algorithm~\ref{Alg_linfcore}. We inductively assume that coreset for approximating $k-1$ hyperplanes is given. The resulting coreset for $k$ hyperplanes is the union of $k-1$-coresets that are constructed here.} \label{fig:linfcore}
        \end{figure*}

        Michael Edwards and Kasturi R. Varadaraja~\cite{EV05} suggested a coreset
        for the projective clustering problem  where the fitting cost is the maximum distance over every input point
        to its closest subspace in the query. It was also proven in~\cite{EV05,har2004no} that no such coreset of size sub-linear in $n$ exists, unless we assume that the input can be scaled to be on a grid of integers.
        The suggested coreset size then depends poly-logarithmically on the size of this grid.
        This assumption is usually reasonable in practice, as, unlike the theoretical RAM model,
        every coordinate is stored in memory using a small number of bits (e.g. 16 or 32 bits).
        The original result is for any $\eps'\in(0,1)$ but due to its usage in Theorem~\ref{sensitivitylemmasquareddistances},
        $\eps'=1/3$ or any other constant -- suffices. The impact on the total sensitivity and thus the overall coreset size would be a factor of $(1+\eps')$, but the multiplicative approximation error would still be $\eps\in(0,1)$.

        \begin{theorem}[Projective Clustering\cite{EV05} \label{evmaintheorem}]
        Let $M\geq2$ and $k\geq1$ be a pair of integers.
        Let $g(d,k)$ be a number that depends only on $d$ and $k$, and can be computed from the proof.
        Let $P\subseteq \REAL$, and $C\subseteq P$ be the output of a call to $\projcore$;
        see Algorithm~\ref{Alg_linfcore}.
        Then $C$ is a $(1/3)$-coreset for $(P,H_{k,d},\dist,\norm{\cdot}_\infty)$
        of size $(\log M)^{g(d,k)}$.
        Moreover $C$ can be computed in $n\cdot (\log M)^{O(1)\cdot g(d,k)}$ time.
        \end{theorem}

        The following lemma states our main application
        which is an $\eps$-coreset that approximates
        the sum $\norm{\cdot}_1$ of fitting error $\phi_{\xi}$
        for \emph{any} $k$-GMM and an arbitrary small constant $\gamma>0$.
        Using Observation~\ref{obss}, it is also a corest for the likelihood
        $L(P,\theta)$ of any $k$-GMM whose smallest eigenvalues is at least $0.160754+\xi$, say, $0.161$.
        The running time is quadratic in $n$ but would be reduced to linear in Section~\ref{sec Coresets for Streaming}
        by applying it only on small weighted subsets of $D$.
        This is also the reason why the lemma is stated for weighted input.

        \begin{lemma}\label{offline}
        Let $M\geq2$ and $D'=(D,w)$ be a weighted set such that $D\subseteq \br{-M,-M+1,\cdots,M}^d$.
        Let $g(d,k)$ be a number that depends only on $d$ and $k$, as defined in Theorem~\ref{evmaintheorem}.
        Let $\eps,\delta\in(0,1/10)$, $\ta>0$ be an arbitrarily small constant,
        and $(C,u)$ be the output of a call to $\textsc{$K$-GMM-Coreset}(D',k,\eps,\delta)$; see Algorithm~\ref{Alg_kgmmcoreset}.

        Then, with probability at least $1-\delta$, $(C,u)$ is an $\eps$-coreset for
        $\left(D', \vartheta_k\left(\frac{e^{\ta}}{2\pi}\right), L,\norm{\cdot}_1\right)$ and for $(D', \vartheta_k(0), \phi_{\ta},\norm{\cdot}_1)$,
        where
        \[
        |C|\leq (\log M)^{g(d,k)}\cdot \frac{\log^2 \overline{w}(D')}{\eps^2}\log\left(\frac{1}{\delta}\right),\]
        its computation time is
        \[
        O(n^2)\cdot (\log M)^{g(d,k)},
        \]
        and
        $\sum_{p\in C}u(p)=\sum_{p\in D}w(p)$.
        \end{lemma}
        \begin{proof}
        Let $(D,w)$ be a positively weighted set of $n$ points in $\br{-M,\ldots,M}^d$, and
        \[
        P=\br{(p^T\mid 0,\cdots,0)^T\in\REAL^{2d+1}\mid p\in D}.
        \]

        By substituting $C:=S$ in Theorem~\ref{evmaintheorem}, a $(1/3)$-coreset $S'$ for $(P,H_{2d+1,k},\dist,\norm{\cdot}_{\infty})$
        of size $|S'|\in (\log M)^{g(d,k)}$ can be computed in time $n\cdot|S'|^{O(1)}$.
        By Theorem~\ref{hyp2}, $S=\br{p\in D\mid (p\mid \mathbf{0})\in S' }$ is an $O(k/\ta^2)$-coreset for  $(D,\vartheta_k(0),\phi_\ta,\norm{\cdot}_\infty)$.
        Hence, $(\lincore,\size,\time)$ is a coreset scheme for
        $(D,\vartheta_k(0),\phi_\ta,\norm{\cdot}_\infty)$ where $\eps'=O(k/\ta^2)$, $\size(\cdot,\cdot,\eps',\cdot)=|S|=|S'|$ and $\time(n,\cdot,\eps',\cdot)\in n\cdot|S'|^{O(1)}$.

        Substituting $P:=D$, $\eps=\eps'$, $\delta=0$, $\cost:=\phi_\ta$, and $Y=\vartheta_k(0)$ in Theorem~\ref{sensitivitylemmasquareddistances}
        yields that we can compute a sensitivity bound $s:D\to [0,\infty)$ for $(D',\vartheta_k(0),\phi_{\ta})$, whose total sensitivity is
        \[
        t=\sum_{p\in D}s(p)\in \size(n,n,\eps',0)\cdot O\left(\log \frac{\overline{w}(D')}{\eps'}\right)
        \subseteq |S'|\cdot O\left(\log \overline{w}(D')\right)=(\log M)^{O(1)}\cdot O\left(\log \overline{w}(D')\right),
        \]
        in time $n\cdot |S|^{O(1)}$.

        By Corollary~\ref{thm:pdimgmm}, the dimension of $(D,\vartheta_k(0),f)$ is $d'\in O(d^4k^4)$
        where $f$ is defined there in~\eqref{ffdef}.
        Plugging $s$ in Corollary~\ref{55} and choosing $\eps,\delta\in(0,1)$ yields that an $\eps$-coreset $(C,u)$ for $(D',\vartheta_k(0),\phi_{\ta},\norm{\cdot}_1)$ of size
        \[
        \begin{split}
        |C|&\in O(1)\cdot\frac{(t+1)}{\eps^2}\left(d'\log (t+1)+\log\left(\frac{1}{\delta}\right)\right)\\
        &=O(1)\cdot\frac{(t+1)}{\eps^2}\left(d^4k^4\log (t+1)+\log\left(\frac{1}{\delta}\right)\right)\\
        &\subseteq O(1)\cdot\left(\frac{t}{\eps}\right)^2\log\left(\frac{1}{\delta}\right)
        \end{split}
        \]
        can be computed in $O(n)\cdot n\cdot |S|^{O(1)}=O(n^2)\cdot (\log M)^{g(d,k)}$ time,
        with probability at least $1-\delta$.

        By Observation~\ref{obss}, $(C,u)$ is also an $\eps$-coreset of $(D',\vartheta_k(e^{\frac{\xi}{2\pi}}),L,\norm{\cdot}_1)$.
        \end{proof}

        \begin{theorem}\label{thm33}
        Let $M\geq2$ be an integer, and $\stream$ be a stream of points from $\br{-M,-M+1,\cdots,M}^d$.
        Let $\eps,\delta\in(0,1)$, and for every $h\geq1$ let
        \begin{equation}\label{ssh5}
        s(h)=\frac{h}{\eps^2}\log\frac{h}{\eps}\cdot  \log^2\left(\frac{1}{\delta}\right)\log^{2g(d,k)} M,
        \end{equation}
        \sloppy where $g(d,k)$ is a function that depends only on $d$ and $k$ as defined in Theorem~\ref{evmaintheorem}.
        Let $\coralgfinal$ be defined as in Algorithm~\ref{Alg_kgmmcoreset}, and $C'_1,C'_2,\cdots$ be the output of a call to
        $\stralg(\stream,\frac{\eps}{6},\frac{\delta}{6},\coralgfinal,s)$; see Algorithm~\ref{Alg_Coreset}.
        Then, with probability at least $1-\delta$, the following hold.

        For every integer $n\geq1$:
        \begin{enumerate}
        \renewcommand{\labelenumi}{\theenumi}
        \renewcommand{\theenumi}{(\roman{enumi})}
        \item (Correctness) $C'_n$ is an $\eps$-coreset of
        $\left(D_n, \vartheta_k\left(\frac{e^{\ta}}{2\pi}\right), L,\norm{\cdot}_1\right)$ and for $(D_n, \vartheta_k(0), \phi_{\ta},\norm{\cdot}_1)$,
        where $D_n$ is the first $n$ points in $\stream$, and $\ta>0$ is an arbitrarily small constant.
        \item (Size) \[
        \displaystyle |C_n|\in (\log M)^{g(d,k)}\cdot\frac{\log^2 n}{\eps^2}\log\left(\frac{1}{\delta}\right).\]
        \item (Memory) there are
        \[
        b\in 
        \frac{1}{\eps^2}\log\frac{1}{\eps}\cdot  \log^2\left(\frac{1}{\delta}\right)\log^{O(1)}(n)\log^{2g(d,k)} M
        \] points in memory during the streaming.
        \item (Update time) $C'_n$ is outputted in additional $t\in O(b^2)\cdot (\log M)^{g(d,k)}$ time after $C'_{n-1}$.
        \item (Overall time) $C'_n$ is computed in $nt$ time.
        \end{enumerate}
        \end{theorem}
        \begin{proof}
        Let $h,w'\geq 1$ and $n=s(h)$. Substituting $D'=D_n$ in Lemma~\ref{offline} yields,
        $(\coralgfinal,\size,\time)$ is an $(\eps,\delta)$-coreset scheme for
        $(D_n, \vartheta_k(0), \phi_{\ta},\norm{\cdot}_1)$, where
        \begin{equation}\label{ss3}
        \size(n,w',\eps,\delta)
        \leq (\log M)^{g(d,k)}\cdot \left(\frac{\log (w')}{\eps}\right)^2\log\left(\frac{1}{\delta}\right),
        \end{equation}
        and
        \[
        \time(n,w',\eps,\delta)\in O(n^2)\cdot (\log M)^{g(d,k)}.
        \]

        Let $h\geq1$, and
        \begin{equation}\label{uu4}
        u(h)=\left(\log M^{g(d,k)}\cdot \frac{4h^5}{\eps^2}\cdot \log\left(\frac{4}{\delta}\right)\right)^{1/2}.
        \end{equation}
        Hence,
        \[
        \size(2n,w',\eps/h, \delta/4^h)
        \leq (\log M)^{g(d,k)}\cdot \left(\frac{h\log (w')}{\eps}\right)^2\log\left(\frac{4^h}{\delta}\right)
        \leq  \left(\frac{u(h)\log(w')}{4h}\right)^2,
        \]
        where the first inequality is by~\eqref{ss3}. Since $u$ is $(5/2)$-log-Lipschitz, and
        \[
        s(h) \geq 10u^3(h) \geq (4u(h)\ln 4u(h))^2,
        \]
        by~\eqref{ssh5} and~\eqref{uu4}, substituting $r=(5/2)$ and $c=2$ in Corollary~\ref{corhal} yields that $s$ is $(\eps,\delta,15)$-halving of $\size$. Substituting
        \[
        s(24)\in
        \left(\log^{g(d,k)} M \log\left(\frac{1}{\delta}\right)\right)^2\cdot \frac{1}{\eps^2}\log\frac{1}{\eps},
        \]
        and
        \[
        \time(b, n,\frac{\eps}{O(\log n)},\frac{\delta}{n^{O(1)}})
        \in O(b^2)\cdot (\log M)^{g(d,k)}
        \]
        in Theorem~\ref{thmstream} then proves Theorem~\ref{thm33} for the query space $(D_n, \vartheta_k(0), \phi_{\ta},\norm{\cdot}_1)$.

        Substituting $P:=D_n$ in Observation~\ref{obss}, proves the theorem also for $\left(D_n, \vartheta_k\left(\frac{e^{\ta}}{2\pi}\right), L,\norm{\cdot}_1\right)$.
        \end{proof}


\chapter{Experimental Results}\label{sec Experimental Results}
    We implemented our main coreset construction and its subroutines into a Python library called CoreGMM~\cite{code}.
    In this section we provide preliminary evaluations of our implementations for several public real-world databases.
    All our experimental results are reproducible and the scripts that generated them can be found in the library.
    While our experiments are preliminary for demonstration only, we publish CoreGMM as open code
    to help future research of the community that may use the code for other ML/DL problems or improve the results in future papers.
    In the next sections we describe our experiments and evaluations for
    effectiveness of using coresets of different sizes for training mixture models.

    \paragraph{From theory to practice. }
        As common in the coreset literature and theoretical computer science in general,
        the worst-case bounds, the VC-dimension and the $O(\cdot)$ notation
        are extremely pessimistic compared to practical experiments.
        E.g., because the analysis is not tight and since there is usually some structure in the data unlike the worst input.
        To this end, our implementation does not contain parameters such as
        $\eps$ or $\delta$, and there is no assumption on the input data
        (e.g. that it is contained in a grid).
        Instead, the input is the desired size of sample (coreset size),
        and the output is a coreset of this size.
        The coreset is computed based on the distribution that is defined by our algorithms.
        See for example the input $m$ to Algorithm~\ref{Alg_Coreset} which hides $\eps$ and $\delta$
        that appears only in the analysis of Theorem~\ref{supsampe}.
        We then run existing heuristics on our coreset and comapre the results to uniform sampling and existing state of the art for coresets of the same size.

    \section{Input Datasets}
        Our experiments were applied on the following three  public real-world datasets from~\cite{feldman2011scalable,lucic2017training}.

        \paragraph{CSN cell phone accelerometer data.}
        As explained in~\cite{feldman2011scalable},
        smart phones with accelerometers are being used by the Community Seismic Network (CSN) as inexpensive seismometers for earthquake detection.
        7GB of acceleration data were recorded from volunteers while carrying and operating their phone
        in normal conditions (walking, talking, on desk, etc.) \cite{faulkner2011next}.
        As done in \cite{feldman2011scalable}, from this data, 17-dimensional feature vectors were computed (containing frequency information, moments, etc.).
        The goal is to train GMMs based on normal data, which then can be used to perform anomaly detection to detect possible seismic activity.
        Motivated by the limited storage on smart phones.

        \paragraph{MNIST handwritten digits.}
        The MNIST dataset contains 60,000 grayscale images of handwritten digits.
        As in \cite{krause2010discriminative}, we normalize each component of the data to have zero mean and unit variance,
        and then reduce each 784-pixel (28x28) image using PCA.

        \paragraph{Higgs high-energy physics.}
        This databset contains 11,000,000 instances describing signal processes which produce Higgs bosons and background processes which do not \cite{baldi2014searching}.

    \section{The Experiment}
        \sloppy We trained a GMM $G_{org}$ using each dataset $T$ with a common python library
        (\href{https://pomegranate.readthedocs.io/en/latest/GeneralMixtureModel.html?highlight=gmm}{Pomegranate}),
        and computed the negative log-likelihood $\ell(G_{org},T)$.
        We used the entire data set several times to train a target GMM $G_{trg}$,
        and computed the average log-likelihood $\ell(G_{trg},T)$ of $T$ using $G_{trg}$.
        The idea behind it was that our coreset upper bound is to describe the dataset as $G_{trg}$ (which was trained with full dataset).
        Therefore, we define the optimal log-likelihood as $\ell_{opt}:= \left| \ell(G_{org},T) - \ell(G_{trg},T) \right|$.

        We compared our algorithm, the algorithm from \cite{lucic2017training} and uniform sample to construct coresets with different sizes in the range between 20 and 5000.
        We tested each coreset $C$ several time to train a GMM $G_C$,
        and computed the coreset average log-likelihood $\ell(G_C,T)$ of $T$ using $G_C$.
        We define the error log-likelihood of each coreset to be $\left| \ell_{opt} - \ell(G_C,T) \right|$.

        We evaluate coresets on, CSN, a dataset of $n=40,000$ feature vectors ($d=17$) using the parameter $k=6$,
        As for Higgs Dataset, we fitted GMMs with $d=150$ components and $k=15$.
        We reproduced \cite{lucic2017training} experiment with a scaled down training size ($n=100,000$).
        From MNIST dataset, we used only the top $d=100$ principal components as a feature vector.
        Using all $n=60000$ we produce coresets and uniformly sampled subsets of sizes between $20$ and $5000$, using $k=10$ and fit GMMs using EM.

    \section{Results}
        As shown in Fig.~\ref{fig:experimentals}, for all the three databases, and
        for every fixed sample size, our coresets introduce smaller approximation errors and obtain significant speedups with respect to solving the problem on the full data set.
        The coreset construction time was similar to \cite{lucic2017training} for every dataset and coreset size.

        For a coreset of size 20 (first point in~\ref{fig:experimentals} ),
        the approximation error using CSN dataset was 10 times smaller then \cite{lucic2017training} and 40 times smaller then uniform sampling,
        and using MNIST \cite{lucic2017training} coreset approximation error was twice then our approximated error.
        For some coreset sizes on Higgs dataset, our approximation error was 3 times smaller then \cite{lucic2017training}.

        \begin{figure*}
            \centering

            \begin{subfigure}[t]{0.45\textwidth}
                \includegraphics[width=\textwidth]{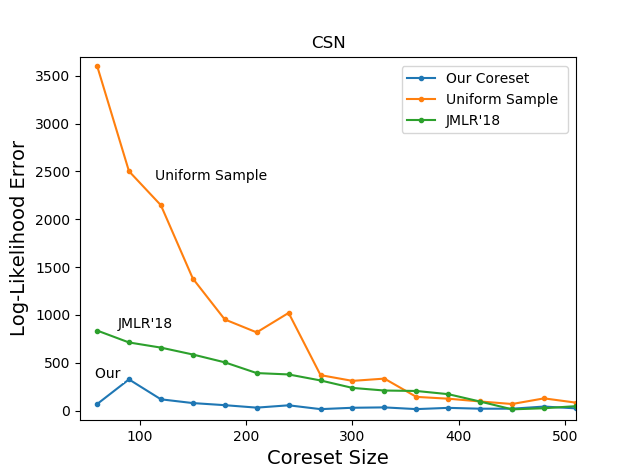}
                \label{fig:CSN_final}
            \end{subfigure}
            ~ \quad
            \begin{subfigure}[t]{0.45\textwidth}
                \includegraphics[width=\textwidth]{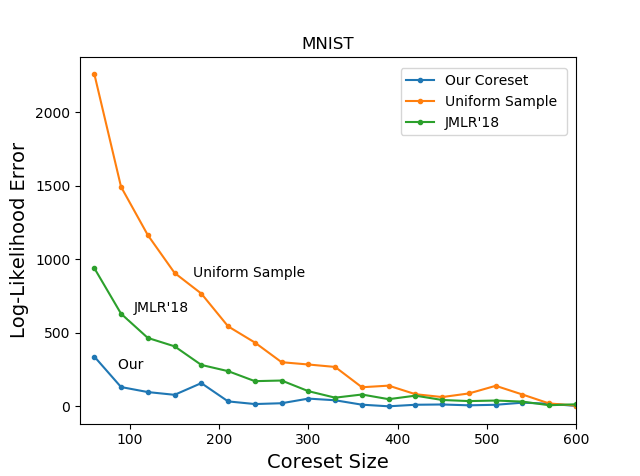}
                \label{fig:MNIST_final}
            \end{subfigure}

            ~

            ~
            \begin{subfigure}[t]{0.45\textwidth}
                \includegraphics[width=\textwidth]{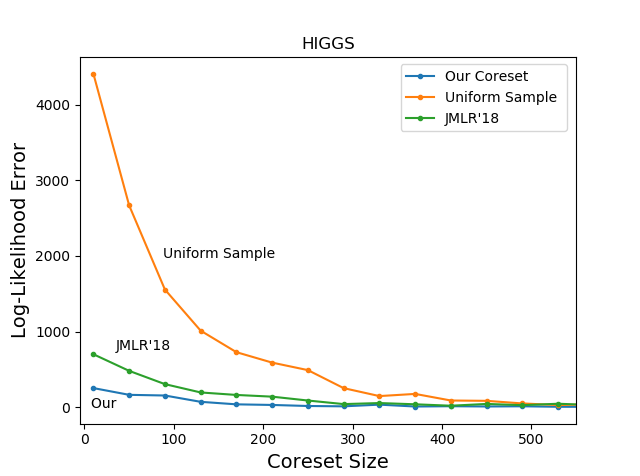}
                \label{fig:HIGGS_final}
            \end{subfigure}

            \caption[Experimental Results]
        	{
                Results are presented for three reconstructed experiments on three real datasets,
                two from \cite{feldman2011scalable} and a third from \cite{lucic2017training}.
                We trained a GMM $G_{org}$ using each dataset $T$ with a common python library
                (\href{https://pomegranate.readthedocs.io/en/latest/GeneralMixtureModel.html?highlight=gmm}{Pomegranate}),
                and computed the log-likelihood $\ell(G_{org},T)$.
                We used the entire data set several times to train a target GMM $G_{trg}$,
                and computed the average log-likelihood $\ell(G_{trg},T)$ of $T$ using $G_{trg}$.
                The idea behind it was that our coreset upper bound is to describe the dataset as $G_{trg}$ (which was trained with full dataset).
                Therefore, we define the optimal log-likelihood as $\ell_{opt}:= \left| \ell(G_{org},T) - \ell(G_{trg},T) \right|$.
                We used our algorithm, \cite{lucic2017training} algorithm and uniform sample to create coresets with different sizes between 20 and 5000.
                We used each coreset $C$ several time to train a GMM $G_C$,
                and computed the coreset average log-likelihood $\ell(G_C,T)$ of $T$ using $G_C$.
                We define the error log-likelihood of each coreset to be $\left| \ell_{opt} - \ell(G_C,T) \right|$.
                As presented here, for all three databases,
                for every fixed sample size, our coresets enjoy smaller approximation errors and obtain significant speedups with respect to solving the problem on the full data set.
        	}
            \label{fig:experimentals}
        \end{figure*}

\chapter{Conclusion and Open Problems}\label{sec Conclusion and Open Problems}
    Most of the papers that are related to coresets suggest a coreset that is tailored to a very specific loss function and family of models.
    In the recent years, there is an effort to replace the approach of "paper after paper" to a generic algorithm and frameworks that can be used compute coresets for a large family of problems.
    This is also the case in this paper.

    We provided algorithms that compute, with high probability (exponential in the coreset's size),
    an $\eps$-coresets for the family of mixtures of $k$ Gaussian  models.
    The cost (fitting error) function is over non-negative log-likelihood
    (with a lower bound on their smallest eigenvalue),
    $\phi$ function (for every mixture of $k$-Gaussians).
    Similarly, we proved the first coresets for maximum over these cost functions.
    The coreset can be maintained for streaming, distributed and dynamic input data using existing techniques.
    The key idea is a reduction for coresets that approximate every set of $k$ subspaces,
    known as the projective clustering problem,
    which is a well known problem with many related coresets in the community of computational geometry and machine learning.

    \paragraph{Limitations. }
    Our coreset has size exponential in both $d$ and $k$ for the general case of arbitrarily ratio between the eigenvalues, and we also assume that the points are scaled to be on a polynomial grid.
    Unfortunately, these assumptions were proved to be unavoidable for the case of projective clustering~\cite{edwards2005no,har2004no}.
    We believe that similar techniques can prove such lower bounds also to the $k$-GMM problem.

    \paragraph{Further improvements and open problems. }
    Still, there may be many lee-ways to obtain smaller coresets for $k$-GMMs.
    We may restrict the dimension of the subspaces in the projective clustering problem which corresponds to restrictions on the eigenvalues of the covariance matrix of each GMM.
    The result from~\cite{feldman2011scalable} can be considered as a special case where each GMM corresponds to a point,
    and the suggested coreset is essentially a coreset for $k$-means whose size is polynomial in $k$ and $d$ with no assumption on the input coordinates.
    The dependency of the coreset size on $d$ may be completely removed using weak coresets (approximates only optimal GMM) as explained in~\cite{feldman2015more},
    or by projecting the input on a lower dimensional space as in~\cite{feldman2013turning}.
    The exponential dependency on $k$ may be removed by assuming a finite but still large set of possible eigenvalues.

    Our experimental results show a huge but common gap between the pessimistic worst case analysis theoretical bounds and results on real-world input that has more structure.
    Adding assumptions on the input data or its distribution might allow provably smaller coresets than ours.
    For example, assume that the input is i.i.d. (as common in machine/PAC learning) and not arbitrary (as in computational geometry).

    The result from~\cite{feldman2011scalable,lucic2017training} can be seen as a special case of our framework,
    where we assume that the Gaussians are semi-spherical (all the eigenvalues of all their covariance matrices are in $[\eps,1/\eps]$).
    In this case we can use $\norm{\cdot}_\infty$ $\eps'$-coreset for $k$-center (points) instead of hyperplanes.

    While such coresets has size exponential in $d$ for small $\eps'$, for our reduction $\eps'=1$ ($2$-approximation) suffices to get $\eps$-coreset for $k$-GMMs; see Theorem~\ref{sensitivitylemmasquareddistances}.
    In this case, coresets for $k$-center of size $k+1$ (independent of $d$ and linear in $k$)
    can be obtained by simply running the algorithm in~\cite{gonzalez1985clustering} for $k+1$ iterations $O(ndk)$ time.
    Plugging the rest of our framework would yield $\eps$-coreset for $\ell_1$, whose size are linear in $k$ and $d$ and logarithmic in $n$.
    This results is the same as~\cite{feldman2011scalable,lucic2017training} for streaming data, but has a gap of $\log n$ factor due to the direct sensitivity bounds (without $\ell_{\infty}$ that were used there.
    We leave the closing of this gap to a future research.

    Some lower bound on the smallest eigenvalue of each covariance matrix is necessary due to scalability issues.
    We assumed that it is $e^{\ta}/2\pi$ when $\ta>0$ is an arbitrarily small constant, compared to $\ta=0$ in~\cite{feldman2011scalable,lucic2017training}. An open problem is whether we can have such coresets for $\ta<0$.

    The size of our coresets depends on bounds for the VC-dimension of projective clustering.
    Proving such lower bounds is still an active field (e.g.~\cite{bhattacharya2017k}) and the exponents in GMMs makes it even harder.

    Another open problem is to apply our technique to other distance functions, e.g. Laplacians mixtures models.
    Generic techniques to obtain such results for log-Lipschitz loss functions were suggested in~\cite{jubran2018minimizing}. Our coreset construction is generic in the sense that it never uses the GMM cost function explicitly.
    We thus suspect that exactly the same algorithm may be applied for a very large family of kernels and soft clustering.
    We leave the definition of this family as an open problem.

\bibliographystyle{plain}
\bibliography{bibfile}
\addcontentsline{toc}{chapter}{Bibliography}

\end{document}